\RequirePackage{fix-cm}
\documentclass{article}
\usepackage{paper_style,times}
\usepackage[T1]{fontenc}
\usepackage[utf8]{inputenc}
\usepackage{amsmath,amssymb,amsfonts,amsthm}
\usepackage{graphicx,booktabs,multirow,array}
\usepackage{longtable,tabularx,float,placeins}
\usepackage[table]{xcolor}
\usepackage{wrapfig,capt-of}
\usepackage{hyperref}
\usepackage{url}
\hypersetup{hidelinks,pdftitle={Teaching Neural Decoders with Language},pdfauthor={Xuan Wen, Yonghao Song, Yuhong Zhang, Guoliang Li, Tzyy-Ping Jung, Xiaorong Gao}}
\title{Teaching Neural Decoders with Language}
\iclrfinalcopy
\author{Xuan Wen$^1$\quad Yonghao Song$^{1,*}$\quad Yuhong Zhang$^2$\\
\bfseries Guoliang Li$^1$\quad Tzyy-Ping Jung$^2$\quad Xiaorong Gao$^{1,*}$\\
\normalfont $^1$Tsinghua University\quad $^2$University of California, San Diego\\
\normalfont $^*$Corresponding authors\\
\normalfont \texttt{eeyhsong@gmail.com}; \texttt{gxr-dea@tsinghua.edu.cn}}
\newcolumntype{P}[1]{>{\raggedright\arraybackslash}p{#1}}
\newcommand{\structhead}{\toprule \textbf{Module} & \textbf{Sub-module} & \textbf{Parameter} & \textbf{Value}\\\midrule}
\newcommand{\ptext}[1]{{\ttfamily\footnotesize #1}}
\definecolor{bestcell}{RGB}{225,234,246}
\definecolor{secondcell}{RGB}{241,241,241}
\newcommand{\tablehead}[1]{{\fontsize{7bp}{9bp}\selectfont #1}}

\makeatletter
\newcommand{\tablepanellabel}[2]{\begingroup\protected@edef\@currentlabel{\thetable(#2)}\label{#1}\endgroup}
\makeatother

\begin{document}
\maketitle
\lhead{}
\renewcommand{\headrulewidth}{0pt}
\begin{abstract}
Brain--computer interfaces decode neural activity, yet class-label supervision (CLS) leaves relationships among classes implicit. We propose Language as Supervision (LaS), which aligns neural features with natural-language class descriptions through learned projections and predicts classes using candidate text prototypes. We evaluate seven encoder architectures with independently trained models for visual, motor, and speech decoding from scalp and intracranial recordings. LaS improves speech-decoding accuracy by up to 17.15~pp over matched CLS baselines. Its representations support candidate-set reuse and label-supervised subject and session adaptation, while LaS enables conditional unseen-frequency decoding in steady-state visual evoked potential (SSVEP) responses. Correct class--description correspondence supports this generalization, and dual-sided attribution visualization shows neural and textual contributions to the same class-matching score. These findings suggest that text descriptions can serve as general-purpose references for learning and interpretation in neural decoding.

\end{abstract}
\begin{figure}[!b]
\centering
\includegraphics[width=\linewidth]{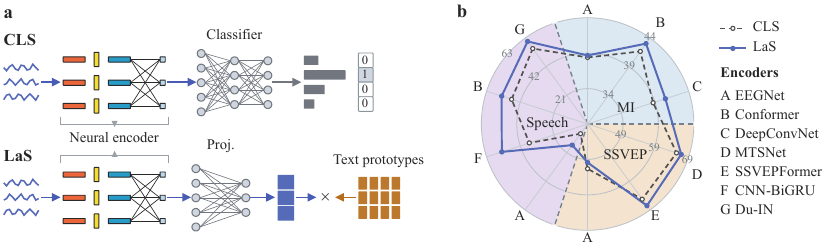}
\caption[Figure 1]{\input{captions/figure1.tex}}
\label{fig:main-1}
\end{figure}
\section{Introduction}
\label{sec:introduction}

Brain--computer interfaces (BCIs) use representations of scalp electroencephalography (EEG) or intracranial recordings for communication and control, such as finger-level robotic hand control \citep{ding2025eeg} and speech neuroprostheses \citep{willett2023high}.

Classical decoding combines engineered features with classifiers or template matching \citep{ramoser2000optimal,tangermann2012review,nakanishi2018enhancing}. Neural decoders learn signal representations jointly with task-specific classifiers \citep{lawhern2018eegnet,schirrmeister2017deep}. Under \emph{class-label supervision} (CLS), one-hot targets specify class identities without explicitly encoding class attributes or semantic relationships.

Natural-language class descriptions provide semantic information that pretrained text models can encode. EmotionCLIP uses templated descriptions for emotion recognition \citep{yan2026cross}, while EEG-PRIME combines class prototypes with task instructions and multilevel conditioning \citep{zhang2026eegprime}. We study class descriptions as a supervision and prediction interface through neural--text alignment, using paired CLS comparisons that retain established encoder architectures. We ask how this alignment affects decoding performance relative to CLS, whether it supports unseen classes under compatible neural--text structure, and which neural components and words support a specified match.

We introduce \emph{Language as Supervision} (LaS; Figure~\ref{fig:main-1}) with three contributions.
(1) LaS uses class descriptions as references for learning and candidate prediction, aligning neural features with text embeddings through learnable projections.
(2) Paired evaluations across seven encoders show decoding gains in motor imagery (MI), steady-state visual evoked potential (SSVEP), and speech decoding from scalp EEG and intracranial stereo-EEG (sEEG). LaS supports candidate configuration through class descriptions, subject/session adaptation, and conditional unseen-frequency prediction.
(3) Theory and controlled language experiments characterize alignment behavior. We propose a dual-sided attribution visualization scheme that shows neural signal components and prompt words contributing to the same class-matching score.

\section{Related Work}
\label{sec:related-work}

\subsection{Semantic Representations for Classification}
\label{sec:semantic-labels}

Semantic class representations link supervision to candidate prediction. Semantic output codes use class attributes as learning targets, including unseen-word fMRI decoding \citep{palatucci2009zero}. DeViSE links visual features to text-derived class-name representations for unseen-class prediction \citep{frome2013devise}; \citet{ba2015predicting} generate classifier weights from descriptions for classes without training images. CLIP combines large-scale contrastive image--text learning with prediction by candidate-description similarity \citep{radford2021learning}, demonstrating competitive zero-shot recognition, frozen-feature transfer through linear probes, and improved robustness over ImageNet supervision on the natural shifts studied. Its dataset transfer and instance-level pairing differ from our class holdouts and repeated descriptions. BLIP-2 connects a frozen image encoder and frozen large language model through a Q-Former \citep{li2023blip2}; CLAP extends contrastive alignment and candidate matching to audio for zero-shot classification \citep{elizalde2023clap}.

\subsection{Language-Aligned Neural Decoding}
\label{sec:language-neural-decoding}

Neural--text alignment methods differ in their sources of textual supervision, the units paired during training, and how text is used at prediction. BrainCLIP maps fMRI representations into a visual--language space using visual and text supervision \citep{ma2025brainclip}. Clinical EEG studies use paired reports: \citet{gijsen2025eeg} combine EEG crops, report segments, and multiple-instance learning for label-efficient phenotyping. Their shuffled-report control changes patient--report pairings, whereas ours reassigns fixed class descriptions. EEG-CLIP learns from EEG--report alignment, uses class text prototypes for zero-shot prediction, and visualizes gradients of EEG--text similarity with respect to frequency-domain inputs \citep{ndir2025eegclip}. Class descriptions also support direct EEG recognition: EmotionCLIP inserts emotion names into 16 templates and classifies by EEG--text matching \citep{yan2026cross}. The EEG-PRIME preprint combines class-name prototypes, task instructions, and multilevel conditioning across EEG tasks, and compares text encoders \citep{zhang2026eegprime}; some downstream fine-tuning settings use task-specific classification heads. LaS builds on prototype classification and neural--text alignment, examining class descriptions as supervision and candidate references across independently trained decoders.

\section{Methods}
\label{sec:methods}

\subsection{Problem Setting and Class-Label Supervision}
\label{sec:problem-setting}

For training classes $\mathcal S=\{1,\ldots,C\}$, $C\ge2$, trial $X_i$ labeled $y_i\in\mathcal S$ yields $N_f(X_i)=\operatorname{neural\_encoder}_\theta(X_i)\in\mathbb R^{1\times d_f}$. Class-label supervision (CLS) jointly trains the encoder and its original, possibly nonlinear head by cross-entropy: $Z_{\mathrm{CLS}}=g_\psi(N_f(X_i))$, $p_i^{\mathrm{CLS}}=\operatorname{softmax}(Z_{\mathrm{CLS}})$, $\ell_{\mathrm{CLS}}=-\log p_{i,y_i}^{\mathrm{CLS}}$. LaS pairs labeled trials with class descriptions.

\subsection{Language as Supervision}
\label{sec:las}
  \begin{figure}[!tb]
    \centering
    \includegraphics[width=\linewidth]{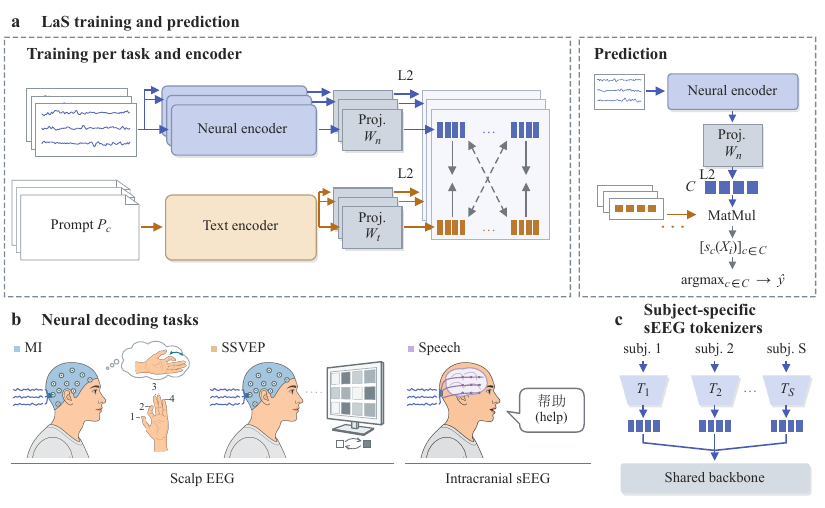}
    \caption[Figure 2]{\input{captions/figure2.tex}}
    \label{fig:main-2}
  \end{figure}

Fixed prompts $P_c$ (e.g., ``A neural SSVEP signal at 10.6 Hz.'') yield embeddings $T_{f,c}=\operatorname{text\_encoder}(P_c)\in\mathbb R^{1\times d_t}$ through a frozen, deterministic text encoder and its prescribed readout. Bias-free projections $W_n\in\mathbb R^{d_f\times d_{\mathrm{proj}}}$ and $W_t\in\mathbb R^{d_t\times d_{\mathrm{proj}}}$ give
\begin{equation}
\begin{aligned}
z_{\mathrm n,\mathrm{proj}}^{(i)}&=\nu(N_f(X_i)W_n),\qquad
z_{\mathrm t,\mathrm{proj}}^{(c)}=\nu(T_{f,c}W_t),\\
\nu(v)&=\frac{v}{\|v\|_2},\qquad
s_c(X_i)=z_{\mathrm n,\mathrm{proj}}^{(i)}\bigl(z_{\mathrm t,\mathrm{proj}}^{(c)}\bigr)^\top.
\end{aligned}
\label{eq:matching-score}
\end{equation}
$z_{\mathrm t,\mathrm{proj}}^{(c)}$ is the text \emph{prototype} and $s_c$ the unscaled cosine score. With the text encoder frozen, training jointly updates $W_t$ and the neural branch, so text prototypes evolve. LaS can thus be viewed as tying classifier-weight geometry to a semantic subspace anchored by text embeddings through the learned projection $W_t$. A learned inverse temperature $\tau^{-1}=e^t>0$ gives $a_{ic}=s_c(X_i)/\tau$. Stack normalized vectors for $B$ pairs $(X_i,P_{y_i})$ into $Z_{\mathrm n,\mathrm{proj}},Z_{\mathrm t,\mathrm{proj}}$, with $L=Z_{\mathrm n,\mathrm{proj}}Z_{\mathrm t,\mathrm{proj}}^\top/\tau\in\mathbb R^{B\times B}$ and $L_{ij}=a_{i,y_j}$. $B$ counts pairs in the loss matrix, irrespective of gradient accumulation. We train with bidirectional InfoNCE~\citep{oord2018representation,radford2021learning}:
\begin{equation}
\begin{aligned}
\mathcal L_{\mathrm{N2T}}&=-\frac1B\sum_{i=1}^B\log\frac{e^{L_{ii}}}{\sum_{j=1}^B e^{L_{ij}}},\qquad
\mathcal L_{\mathrm{T2N}}=-\frac1B\sum_{i=1}^B\log\frac{e^{L_{ii}}}{\sum_{j=1}^B e^{L_{ji}}},\\
\mathcal L_{\mathrm{SA}}&=\tfrac12(\mathcal L_{\mathrm{N2T}}+\mathcal L_{\mathrm{T2N}}).
\end{aligned}
\label{eq:bidirectional-infonce}
\end{equation}
N2T queries texts with neural signals; T2N queries trials with text. Both retain all $B$ entries and diagonal targets, introducing same-class off-diagonal competitors.

Each task--encoder model learns its own $\theta,W_n,W_t,t$ and alignment space, sharing its backbone only within that model (Figure~\ref{fig:main-2}c). Text embeddings can be precomputed; caching final prototypes after training avoids repeated text encoding at inference. For candidate set $\mathcal C$, predict $\widehat y_i=\operatorname*{arg\,max}_{c\in\mathcal C}s_c(X_i)$.

\subsection{Theoretical Properties}
\label{sec:analysis}

\paragraph{Directional asymmetry.}
For fixed prompts and deterministic text forwards, let $\mathcal G_c=\{i:y_i=c\}$, $n_c=|\mathcal G_c|$, $\mathcal S_b=\{c:n_c>0\}$ and $\bar a_c=n_c^{-1}\sum_{i\in\mathcal G_c}a_{ic}$. With $m$ identical texts per class, the diagonal target has probability at most $1/m$, giving N2T a $\log m$ loss floor. The N2T identity assumes all $C$ classes occur $m$ times; T2N allows unequal counts:
\begin{equation}
\begin{aligned}
\ell_i^{\mathrm{N2T}}&=\log m-a_{i,y_i}+\log\sum_{c=1}^C e^{a_{ic}},\qquad B=mC,\\
\bar\ell_c^{\mathrm{T2N}}&=-\bar a_c+\log\sum_{j=1}^B e^{a_{jc}},\qquad
\mathcal L_{\mathrm{T2N}}=\frac1B\sum_{c\in\mathcal S_b}n_c\bar\ell_c^{\mathrm{T2N}}.
\end{aligned}
\label{eq:directional-relations}
\end{equation}
Balanced N2T is class cross-entropy plus a constant; unequal counts reweight present classes. For $n_c>1$, T2N provides implicit regularization toward within-class aggregation of matching responses along the corresponding text-prototype direction. With the mean within-class score and all other-class scores against that prototype fixed, equal within-class matching scores minimize its class loss (Appendix~\ref{app:theory-score}). Other feature directions may still differ.

\paragraph{Repeated-prompt loss floors.}
Diagonal and Class-aware SigLIP ablations~\citep{zhai2023sigmoid} use shared learned bias $b_{\mathrm{sig}}\in\mathbb R$, $\xi_{ic}=a_{ic}+b_{\mathrm{sig}}$, $\ell_\pm(\xi)=\log(1+e^{\mp\xi})$ and binary labels $M\in\{0,1\}^{B\times B}$:
\begin{equation}
\begin{aligned}
\mathcal L_{\mathrm{Sig}}(M)&=\frac1B\sum_{i,j=1}^B\left[M_{ij}\ell_+(\xi_{i,y_j})+(1-M_{ij})\ell_-(\xi_{i,y_j})\right],\\
M_{ij}^{\mathrm{diag}}&=\mathbf1[i=j],\qquad M_{ij}^{\mathrm{class}}=\mathbf1[y_i=y_j].
\end{aligned}
\label{eq:sigmoid-objective}
\end{equation}
Both sum $B^2$ terms and divide by $B$. Deterministic class-text copies imply these bounds; for $B=mC$, the last row lists balanced values for InfoNCE, Diagonal and Class-aware SigLIP, respectively:
\begin{equation}
\begin{aligned}
\mathcal L_{\mathrm{SA}}&\ge\frac1B\sum_{c\in\mathcal S_b}n_c\log n_c,\qquad \mathcal L_{\mathrm{Sig}}^{\mathrm{class}}\ge0,\\
\mathcal L_{\mathrm{Sig}}^{\mathrm{diag}}&\ge\frac1B\sum_{c\in\mathcal S_b}n_c\left[n_c\log n_c-(n_c-1)\log(n_c-1)\right],\\
n_c=m:\quad&\log m,\qquad m\log m-(m-1)\log(m-1),\qquad0.
\end{aligned}
\label{eq:loss-floors}
\end{equation}
Logs are natural, with $0\log0=0$; attainability depends on separation and realizable logits (Appendix~\ref{app:theory-losses}).

\section{Experimental Evaluation}
\label{sec:evaluation}
\begin{table}[!tb]
\centering
\caption[Overall decoding performance and cross-encoder comparisons]{\textbf{Decoding performance across tasks and neural encoders.} CLS--LaS comparisons; FBCSP-LDA is a separate MI reference. Performance: subject-mean accuracy (\%) $\pm$ sample SD; mean MI $\kappa$ in parentheses. $\Delta=\mathrm{LaS}-\mathrm{CLS}$ in percentage points. Blue/gray: highest/second-highest means within each MI or Speech metric across CLS and LaS. ${}^{1}$SSVEP: RW training $\to$ IW evaluation. Stars: Holm-corrected paired $t$-tests within encoder (Section~\ref{sec:evaluation-setup}; Appendix~\ref{app:b-statistics}).
}
\label{tab:overall-cross-encoder}
\begingroup
\newcommand{\tableperf}[3]{#1{\fontsize{6.5bp}{10bp}\selectfont$\pm$#2}\if\relax\detokenize{#3}\relax\else\ (\ensuremath{\kappa}=#3)\fi}
\fontsize{8bp}{10bp}\selectfont
\setlength{\tabcolsep}{2pt}
\renewcommand{\arraystretch}{1.0}
\begin{tabular*}{\linewidth}{@{\extracolsep{\fill}}llccccc r@{}}
\toprule
\multirow{2}{*}{\tablehead{Task}} & \multirow{2}{*}{\tablehead{Method/encoder}} & \multicolumn{2}{c}{\tablehead{Params (M)}} & \multirow{2}{*}{\tablehead{\shortstack{Duration\\(s)}}} & \multicolumn{2}{c}{\tablehead{Performance}} & \multirow{2}{*}{\tablehead{\shortstack{$\Delta$ Acc.\\(pp)}}} \\
\cmidrule(lr){3-4}\cmidrule(lr){6-7}
& & \tablehead{CLS} & \tablehead{LaS} & & \tablehead{CLS} & \tablehead{LaS} & \\
\midrule
\multirow{4}{*}{\textbf{MI}} & FBCSP-LDA & --- & --- & \multirow{4}{*}{5.0} & \multicolumn{2}{c}{\tableperf{32.47}{5.94}{0.100}} & --- \\
& EEGNet & 0.002 & 0.30 & & \tableperf{38.39}{4.76}{0.179} & \tableperf{38.68}{7.08}{0.182} & +0.28 \\
& DeepConvNet & 0.29 & 0.79 & & \tableperf{38.72}{6.05}{0.183} & \tableperf{40.54}{8.65}{0.207} & +1.82\textsuperscript{*} \\
& Conformer & 0.29 & 0.48 & & \cellcolor{secondcell}\tableperf{41.67}{8.77}{0.222} & \cellcolor{bestcell}\tableperf{43.02}{9.29}{0.240} & +1.36\textsuperscript{**} \\
\midrule
\multirow{9}{*}{\shortstack[l]{\textbf{SSVEP}\\RW$\to$IW\textsuperscript{1}}} & \multirow{3}{*}{EEGNet} & 0.04 & 0.32 & 0.6 & \tableperf{26.76}{11.55}{} & \tableperf{26.98}{10.91}{} & +0.21 \\
& & 0.05 & 0.37 & 0.8 & \tableperf{41.07}{16.34}{} & \tableperf{40.42}{15.82}{} & -0.65 \\
& & 0.05 & 0.39 & 1.0 & \tableperf{51.62}{18.32}{} & \tableperf{49.86}{16.43}{} & -1.76 \\
& \multirow{3}{*}{MTSNet} & 9.92 & 7.37 & 0.6 & \tableperf{34.32}{15.23}{} & \tableperf{36.56}{15.27}{} & +2.24\textsuperscript{**} \\
& & 10.87 & 8.12 & 0.8 & \tableperf{53.27}{20.52}{} & \tableperf{55.17}{19.93}{} & +1.89 \\
& & 11.89 & 8.95 & 1.0 & \tableperf{65.39}{19.49}{} & \tableperf{66.79}{19.83}{} & +1.39 \\
& \multirow{3}{*}{SSVEPFormer} & 3.09 & 1.11 & 0.6 & \tableperf{33.46}{15.19}{} & \tableperf{35.94}{15.35}{} & +2.48\textsuperscript{*} \\
& & 3.09 & 1.11 & 0.8 & \tableperf{52.30}{19.89}{} & \tableperf{54.51}{19.17}{} & +2.21\textsuperscript{*} \\
& & 3.09 & 1.11 & 1.0 & \tableperf{65.23}{20.73}{} & \tableperf{67.50}{19.10}{} & +2.27\textsuperscript{*} \\
\midrule
\multirow{4}{*}{\textbf{Speech}} & EEGNet & 0.02 & 0.13 & \multirow{4}{*}{3.0} & \tableperf{7.13}{2.51}{} & \tableperf{15.46}{6.39}{} & +8.33\textsuperscript{***} \\
& CNN-BiGRU & 0.62 & 0.84 & & \tableperf{36.33}{9.26}{} & \tableperf{53.47}{13.43}{} & +17.15\textsuperscript{***} \\
& Conformer & 4.18 & 3.96 & & \tableperf{47.34}{11.73}{} & \tableperf{53.56}{15.04}{} & +6.21\textsuperscript{***} \\
& Du-IN & 5.05 & 4.94 & & \cellcolor{secondcell}\tableperf{55.27}{13.17}{} & \cellcolor{bestcell}\tableperf{60.49}{14.30}{} & +5.22\textsuperscript{***} \\
\bottomrule
\end{tabular*}
\endgroup
\par\medskip
\begingroup
\includegraphics[width=\linewidth]{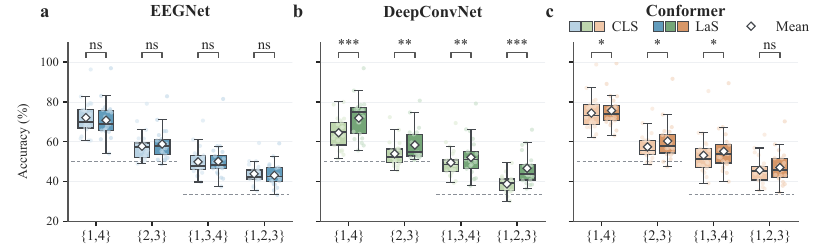}
\captionof{figure}[Candidate-subset reuse in motor imagery]{\textbf{Candidate-subset reuse in motor imagery.} CLS trains per subtask; LaS restricts candidates in each fold's four-class checkpoint without further optimization. Class IDs 1, 2, 3, 4 denote Thumb, Index, Middle, Pinky. Boxes: medians and interquartile ranges (IQRs); whiskers: $1.5\times\mathrm{IQR}$ rule; circles: subject accuracies ($n=21$); diamonds: means; dashed lines: chance (50\%/33.3\%). Stars: Holm-corrected paired CLS--LaS comparisons (Section~\ref{sec:evaluation-setup}).
}
\label{fig:main-3}
\endgroup
\end{table}

\subsection{Evaluation Setup}
\label{sec:evaluation-setup}

We train separate models on each dataset's pooled training subjects. \textbf{MI} contains 21 subjects imagining thumb, index, middle, or pinky movements~\citep{ding2025eeg}. Five within-subject physical-trial folds keep each trial's windows together; decisions aggregate 1-s inputs over 5-s trials. Following the released implementation of \citet{ding2025eeg}, checkpoints maximize validation-window accuracy. In our five-fold evaluation, the same held-out fold serves checkpoint selection and reporting of trial-level accuracy. \textbf{SSVEP} contains 35 subjects and 40 stimulus classes~\citep{wang2017benchmark}, evaluated in seven folds of 25/5/5 training/validation/test subjects; checkpoints are selected on validation subjects. \textbf{Speech} contains 12 subjects overtly reading 61 Chinese words with English prompt translations~\citep{zheng2024duin}, 3-s sEEG inputs, and a shared physical-trial split of approximately 80\%/10\%/10\% for training/validation/testing.

We evaluate EEGNet~\citep{lawhern2018eegnet}, DeepConvNet~\citep{schirrmeister2017deep}, EEG-Conformer~\citep{song2023eeg}, SSVEPFormer~\citep{chen2023transformer}, MTSNet~\citep{lan2025mtsnet}, CNN-BiGRU~\citep{moses2021neuroprosthesis}, and Du-IN~\citep{zheng2024duin} as neural encoders, with FBCSP-LDA~\citep{ang2008filter} as an MI reference.

CLS--LaS pairs share encoder architectures and splits, training to convergence with AdamW and their respective output modules and settings. Training uses PyTorch on a single NVIDIA RTX 4090 GPU. For Table~\ref{tab:overall-cross-encoder}, LaS uses frozen EmbeddingGemma (300M)~\citep{schechtervera2025embeddinggemma}, one fixed correct prompt per class (P0), and bidirectional InfoNCE. Parameter counts include all trainable modules and both projections, excluding the frozen text encoder; Duration includes the decision's full signal span and aggregation (Appendix~\ref{app:experimental-details}).

\noindent\textbf{Windowing and data augmentation.}
For SSVEP, initial-window (IW) sampling fixes the start at latency-corrected onset; random-window (RW) augmentation uniformly samples valid starts, expanding the training set to 300,000 windows. Inter-trial recombination (IR) combines segments from same-class training trials. Neural training zero-masks 30\% of input time (ratio 0.3)~\citep{ding2024maskencoding} for Table~\ref{tab:overall-cross-encoder} (both CLS and LaS), Table~\ref{tab:b-ssvep-itrs}, and the seen classes in Figure~\ref{fig:main-4}c. SSVEP augmentation details appear in Appendix~\ref{app:ssvep-window-augmentation}. All four Speech encoders use Du-IN's downstream augmentation~\citep{zheng2024duin} during CLS/LaS training: random left/right shifts of 0--0.3~s with zero padding preserve the 3-s input span. Validation and testing use unshifted inputs.

Table~\ref{tab:overall-cross-encoder} averages subjects equally after pooling MI's five validation-fold predictions or averaging Speech's five seeds per subject. Table~\ref{tab:overall-cross-encoder} and Figures~\ref{fig:main-3}, \ref{fig:main-5}, and~\ref{fig:main-8} use two-sided paired $t$-tests with Holm correction. Appendix~\ref{app:b-statistics} defines the comparison families and reports paired Cohen's $d_z$ for Table~\ref{tab:overall-cross-encoder}; $*$/$**$/$***$ denote $p<.05$/.01/.001; ns means nonsignificant.

\subsection{Decoding Performance}
\label{sec:decoding-performance}

\noindent\textbf{Motor imagery.}
LaS (EEG-Conformer) leads evaluated four-class accuracy at $43.02\%\pm9.29\%$, with mean $\kappa=0.240$ and a $1.36$~pp CLS-relative gain ($p<.01$; Table~\ref{tab:overall-cross-encoder}). DeepConvNet gains $1.82$~pp ($p<.05$); EEGNet differs nonsignificantly.

By changing only the candidate prompt set, LaS reuses its four-class decoders without further optimization and significantly outperforms subtask-trained CLS on all four seen-class subsets with DeepConvNet and on three with EEG-Conformer (Figure~\ref{fig:main-3}). Differences are nonsignificant for all four EEGNet subsets and EEG-Conformer's $\{\mathrm{Thumb},\mathrm{Index},\mathrm{Middle}\}$ subset.

\noindent\textbf{SSVEP.}
With RW training and IW validation and testing, LaS significantly improves MTSNet at 0.6~s and SSVEPFormer at all three durations after Holm correction; EEGNet differences are nonsignificant and negative at 0.8/1.0~s (Table~\ref{tab:overall-cross-encoder}). Separately, Table~\ref{tab:b-ssvep-itrs} uses IW training for neural baselines and IW+IR for LaS (MTSNet), all evaluated on IW: LaS leads in accuracy at 0.7--1.0~s and estimated ITR at all five durations; SSVEPFormer leads in accuracy at 0.6~s.

\noindent\textbf{Speech.}
All four encoders show significant paired LaS gains ($p<.001$), exceeding the modest MI/SSVEP improvements. With its neural encoder trained from scratch, LaS (Du-IN) leads at $60.49\%\pm14.30\%$ ($+5.22$~pp over CLS). This accuracy exceeds the $59.18\%$ reported for the multi-subject pretrained Du-IN (poms) variant~\citep{zheng2024duin}. CNN-BiGRU achieves the largest gain over CLS ($17.15$~pp). Gains occur with more or fewer trainable parameters (Table~\ref{tab:overall-cross-encoder}).

\subsection{Adaptation and Conditional Unseen-Class Decoding}
\label{sec:evaluation-adaptation}
  \begin{figure}[!tb]
    \centering
    \includegraphics[width=\linewidth]{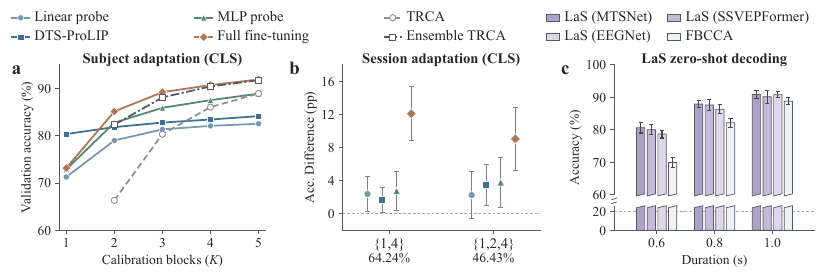}
    \caption[Figure 4]{\input{captions/figure4.tex}}
    \label{fig:main-4}
  \end{figure}

The preceding reuse results concern configuration of seen candidate sets; we next examine adaptation to changes in neural inputs and extension beyond training classes. The neural adaptation methods in Figure~\ref{fig:main-4}a,b start from LaS-trained models and use class-label supervision with cross-entropy during adaptation. Linear and MLP probes fit heads to frozen features; full fine-tuning updates the neural branch. Based on ProLIP~\citep{fahes2026prolip}, DTS-ProLIP restricts neural projection updates to the span of the current candidates' centered text prototypes, keeping the encoder and source text prototypes fixed (Appendix~\ref{app:adaptation}).

\paragraph{MI session adaptation.}
Figure~\ref{fig:main-4}b replays a subsequent session using the pooled four-class source model. Direct inference uses each subject's eight Base runs. Models adapted on these runs (80/20 physical-trial training/validation split) are evaluated on eight later Fine runs. Full fine-tuning gives the largest mean gains on Thumb/Pinky and Thumb/Index/Pinky, from direct-inference accuracies of $64.24\%$/$46.43\%$. The three-class linear probe's 95\% paired-difference $t$ interval crosses zero.

\paragraph{SSVEP subject adaptation.}
Figure~\ref{fig:main-4}a compares LaS (MTSNet) adaptations with TRCA and Ensemble TRCA~\citep{nakanishi2018enhancing} at 1~s and matched calibration budgets. Six leave-one-block-out rounds adapt on $K$ remaining blocks, each containing one trial per class; TRCA variants start at $K=2$. Reported validation accuracy also selects MLP/full-fine-tuning epochs. DTS-ProLIP leads the displayed means at $K=1$, full fine-tuning at $K=2$--$5$; Ensemble TRCA approaches the latter with increasing $K$.

\paragraph{Conditional unseen-class decoding.}
We evaluate LaS conditional zero-shot learning (ZSL) for unseen frequency classes using three independently trained encoder architectures (Figure~\ref{fig:main-4}c). The same 35 subjects are used for training and evaluation across eight class splits, each reserving all EEG from five unseen frequency classes for testing; 35 seen classes supply five training blocks and one validation block per subject, with the latter selecting checkpoints. After RW training, IW evaluation scores all 40 candidates in seen-class validation and only five unseen candidates in testing (20\% chance). All three LaS encoders outperform FBCCA numerically at 0.6/0.8/1.0~s, especially shorter durations. Stimulus frequency supplies structured neural--text correspondence in this setting.

  \begin{figure}[!b]
    \centering
    \includegraphics[width=\linewidth]{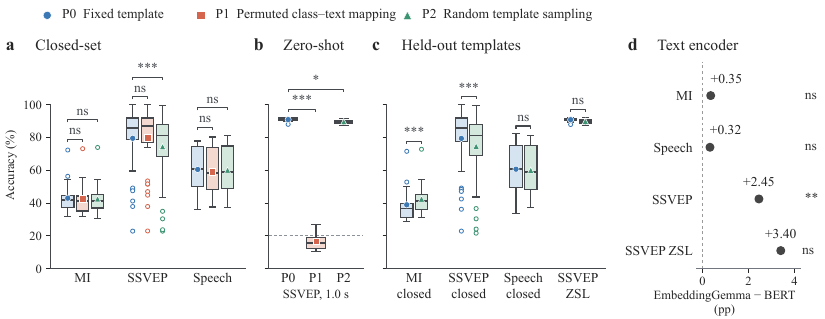}
    \caption[Figure 5]{\input{captions/figure5.tex}}
    \label{fig:main-5}
  \end{figure}
\section{Analysis and Discussion}
\label{sec:discussion}
\subsection{Class Correspondence and Language Configuration}
\label{sec:results-language}
Figure~\ref{fig:main-5}a--c compares training prompt configurations: fixed correct prompts (P0), a fixed derangement of seen-class descriptions (P1; $\pi(c)\ne c$ for every seen class), and random sampling of five correct training templates (P2). Panels (a,b) use one fixed prompt per class at inference; P2 uses P0 prompts. P1 maps closed-set predictions through the inverse permutation and uses true unseen descriptions for unseen-class prediction. P1 averages five derangement environments within subject or class split; Speech P2 averages three training seeds per subject.

Closed-set P0/P1 differences are nonsignificant, whereas P1 reduces unseen-only SSVEP accuracy to near chance ($p<.001$). P1 reassigns an unchanged raw text embedding set, testing class--description correspondence while preserving embedding relationships. With full-row-rank seen text embeddings and defined prototypes, a text projection can restore all seen prototypes while the neural branch remains fixed (Appendix~\ref{app:theory-correspondence}).

With fixed P0 prompts at inference, P2 worsens SSVEP closed-set ($p<.001$) and unseen-only decoding ($p<.05$), with nonsignificant MI/Speech differences. Panel (c) compares P0 and P2 using an ensemble of five held-out templates, none used in either training configuration. P2 improves MI and worsens closed-set SSVEP (both $p<.001$); Speech and unseen-only SSVEP differences are nonsignificant.

Figure~\ref{fig:main-5}d compares independently selected best configurations of BERT~\citep{devlin2019bert} and EmbeddingGemma. EmbeddingGemma improves closed-set SSVEP by $2.45$~pp ($p=.00485$); its $3.40$~pp unseen-only gain is nonsignificant ($p=.0838$), as are the MI/Speech differences. Text-encoder freezing versus LoRA adaptation~\citep{hu2022lora} has task- and encoder-dependent effects (Figure~\ref{fig:b-text-pairs}; Appendix~\ref{app:b-language}).

\subsection{Interpreting Neural--Text Matches}
\label{sec:results-attribution}
  \begin{figure}[!tb]
    \centering
    \includegraphics[width=\linewidth]{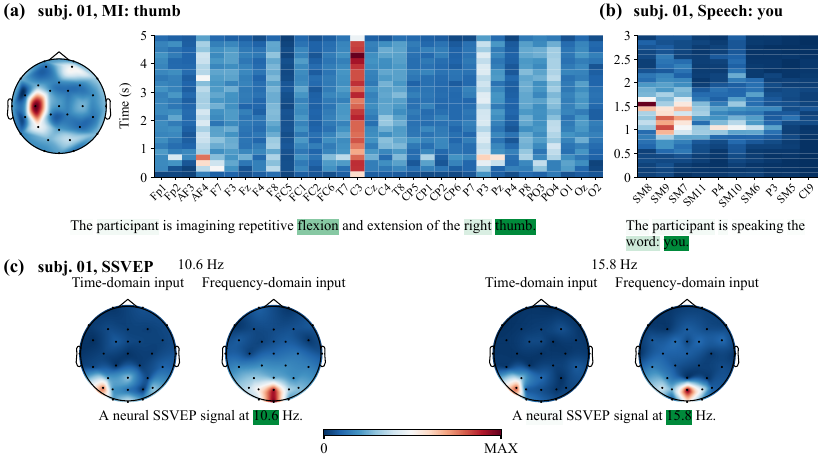}
    \caption[Figure 6]{\input{captions/figure6.tex}}
    \label{fig:main-6}
  \end{figure}
We propose a dual-sided attribution visualization scheme for LaS class matching (Figure~\ref{fig:main-6}). Each example averages attributions over all training trials of one class from the first subject in its dataset: 120 right-thumb trials for MI, 38 ``you'' trials for Speech, and six original recordings for each SSVEP class (10.6 or 15.8~Hz). Grouped Integrated Gradients~\citep{sundararajan2017axiomatic} quantify signed contributions of input groups (sets of input components) to changes in the same unscaled class-specific cosine score along specified baseline-to-input paths. Neural attribution fixes class text; conditional word attribution fixes the neural representation and other words (Appendices~\ref{app:b-attribution} and~\ref{app:attribution-method}). For display, ReLU-rectified neural attributions are divided by each complete input's maximum positive attribution before aggregation. This input is one MI window, one Speech trial, or both SSVEP branches jointly. Word attributions retain their signs during averaging; positive means are then divided by their maximum within each prompt for display.

Right-thumb MI shows a hotspot near C3 on the left central scalp over the sensorimotor region, with sustained C3 contributions and emphasis on ``thumb'' and ``flexion.'' An early AF4 contribution may reflect ocular or muscle artifacts. In Speech, earlier contributions on SM9 and SM7 precede a stronger peak on SM8, followed by weaker contributions toward the trial's end; ``you'' is the strongest word contribution, followed by ``word.'' SSVEP time-domain attributions emphasize the left occipital scalp, and frequency-domain attributions the occipital midline; word attributions emphasize ``10.6'' and ``15.8.''

  \begin{figure}[!t]
    \centering
    \includegraphics[width=\linewidth]{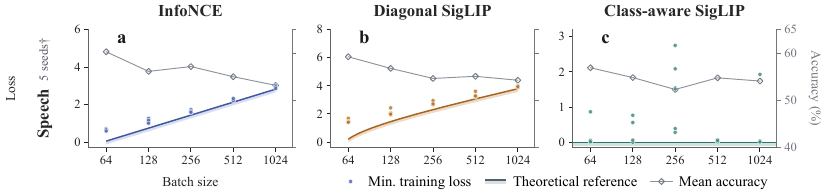}
    \caption[Figure 7]{\input{captions/figure7.tex}}
    \label{fig:main-7}
  \end{figure}
\subsection{Training Objectives and Directional Effects}
\label{sec:results-objectives}
  \begin{wrapfigure}{R}{59mm}
    \centering
    \includegraphics[width=59mm]{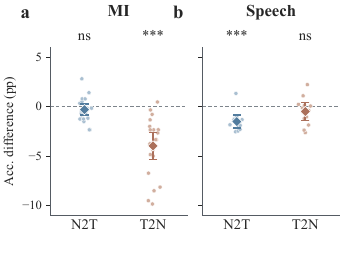}
    \caption[Directional InfoNCE ablation]{\input{captions/figure8.tex}}
    \label{fig:main-8}
  \end{wrapfigure}
For each Speech condition, accuracy is averaged over five seeds per subject and then across the 12 subjects (batch grids: Appendix~\ref{app:b-objectives}). All observed minimum training losses in Figure~\ref{fig:main-7} lie above the corresponding curves derived from Equation~\eqref{eq:loss-floors}. Across the three tasks (Figure~\ref{fig:b-loss-batch}), larger batches generally raise InfoNCE and Diagonal SigLIP minima without corresponding accuracy gains. Identical same-class descriptions create competition in InfoNCE and label conflicts in Diagonal SigLIP. Class-aware labels resolve these conflicts and increase positive-pair weights, but eliminating the positive loss floor does not consistently improve decoding accuracy (Section~\ref{sec:analysis}; Appendix~\ref{app:theory-losses}).

Figure~\ref{fig:main-8} compares MI (EEG-Conformer) and Speech (Du-IN), with five seeds per Speech condition. N2T promotes class discrimination, while T2N encourages within-class aggregation of matching responses along the corresponding text-prototype direction under the conditions in Section~\ref{sec:analysis}. Bidirectional InfoNCE outperforms T2N-only in MI and N2T-only in Speech (both $p<.001$); other contrasts are nonsignificant.

\subsection{Limitations and Future Work}
\label{sec:limitations}
Models are trained independently for each dataset and encoder; transfer between datasets and recording systems remains to be evaluated. We observed unseen-class decoding in SSVEP, but not in Speech. Initial nonstationarity in SSVEP responses limits transfer from RW training to IW evaluation in Table~\ref{tab:overall-cross-encoder} (Appendix~\ref{app:ssvep-further}). Online control and long-term decoding stability also require evaluation.

Future work could extend LaS from class-level prediction to real-time decoding of continuous behavior and closed-loop control. Language could help neural interfaces learn users' goals through interaction and generalize to new tasks. Language-guided feedback may also help users develop transferable control skills.

\par\ifnum\value{WF@wrappedlines}>1\relax
\vspace{\dimexpr\value{WF@wrappedlines}\baselineskip-\baselineskip\relax}
\fi
\WFclear
\section{Conclusion}
\label{sec:conclusion}
LaS uses natural-language class descriptions as references for supervision and prediction in neural decoding. Evaluations of independently trained models cover visual, motor, and speech decoding with scalp and intracranial recordings, showing its applicability to different encoder architectures. The learned representations support candidate-set reuse and label-supervised adaptation, and LaS enables conditional unseen-frequency decoding in SSVEP. Theoretical and language analyses characterize loss bounds and class correspondence; dual-sided attribution visualization shows contributions from neural signal components and prompt words to the same class-specific matching score. These findings suggest that language can provide a reusable interface through which neural decoding systems can be configured and extended as users' needs evolve.

\label{main:text-end}
\paragraph{Reproducibility Statement.}
The main text and appendices describe the data preprocessing, model architectures, training settings, and evaluation protocols. We will release the implementation code and detailed training configurations later at \url{https://github.com/Kagura-Yachiyo/Teaching-Neural-Decoders-with-Language}.

\subsection*{AI Use Statement}
We used DeepSeek and ChatGPT (OpenAI) to identify relevant literature, draft selected parts of the manuscript, and assist with writing and language editing. We reviewed the AI-assisted content and take responsibility for the final manuscript.

\bibliography{references}
\bibliographystyle{paper_references}

\clearpage
\appendix
\renewcommand{\theHtable}{\thesection.\arabic{table}}
\renewcommand{\theHfigure}{\thesection.\arabic{figure}}
\renewcommand{\theHequation}{\thesection.\arabic{equation}}
\setcounter{table}{0}
\setcounter{figure}{0}
\setcounter{equation}{0}
\begingroup
\setlength{\tabcolsep}{3pt}
\renewcommand{\arraystretch}{1.13}
\renewcommand{\thetable}{A.\arabic{table}}
\section{Experimental Details}
\label{app:experimental-details}

This appendix specifies the data processing, encoder interfaces, and evaluation measures used in the main performance comparison. Section~\ref{app:adaptation} describes the adaptation algorithms separately. MI denotes finger motor imagery, SSVEP denotes visual target decoding from steady-state visual evoked potentials, and Speech denotes overt word reading. MI and SSVEP use scalp EEG; Speech uses intracranial stereoelectroencephalography (sEEG). Within each task, CLS and LaS use the same training and evaluation splits. Each neural encoder is trained separately and learns its own alignment space.

\subsection{Shared Conventions}
\label{app:shared}

Table~\ref{tab:terminology} defines the task names, window conditions, and stages of text representation used throughout the paper.

{\small
\begin{longtable}{P{1.65in}P{3.68in}}
\caption{Terminology and naming conventions.}\label{tab:terminology}\\
\toprule Preferred term & Meaning and scope\\\midrule\endfirsthead
\multicolumn{2}{l}{\footnotesize Table \ref{tab:terminology} (continued)}\\
\toprule Preferred term & Meaning and scope\\\midrule\endhead
\bottomrule\endfoot
Language as Supervision (LaS) & Language-supervised decoding by aligning neural representations with natural-language class descriptions.\\
Class-label supervision (CLS) & Neural-classification comparator trained with discrete labels and a task-specific head; LaS also uses class labels to establish pairs.\\
One-hot targets; cross-entropy (CE) & Target encoding and loss function, respectively.\\
Motor imagery (MI) & Finger motor-imagery decoding task.\\
SSVEP & Steady-state visual evoked potential decoding of visual stimulus targets.\\
Speech & Speech-decoding task using overt word reading.\\
Scalp EEG; sEEG & Non-invasive scalp recording and invasive stereotactic intracranial recording, respectively; MI and SSVEP use the former, Speech the latter.\\
Random-window (RW) & Contiguous window of length $L$ with a random valid start within a training trial.\\
Initial-window (IW) & First $L$ seconds from the defined latency-corrected valid onset; used for IW training or IW evaluation.\\
Inter-trial recombination (IR) & Segment recombination among same-class training trials; implementation is defined in Appendix~\ref{app:ssvep-window-augmentation}.\\
IW $\rightarrow$ IW; IW + IR $\rightarrow$ IW; RW $\rightarrow$ IW & Training-window condition is left of the arrow and evaluation condition right of it; IR denotes additional recombination.\\
Prompt template & Sentence template containing a class variable.\\
Prompt, $P_c$ & Complete text passed to the text encoder.\\
Text embedding, $T_{f,c}$ & Text-encoder vector before the task-specific text projection.\\
Text prototype, $z_{\mathrm t,\mathrm{proj}}^{(c)}$ & Class reference after the current model's text projection and $\ell_2$ normalization.\\
Held-out templates; unseen classes & Templates absent from training and classes absent from neural-model training, respectively.\\
Similarity score; logit & $s_c(X_i)$ is cosine similarity; $a_{ic}=s_c(X_i)/\tau$ is the temperature-scaled logit.\\
Batch size, $B$ & Number of neural--text pairs in one pairing matrix.\\
Class count within a batch, $n_c$ & Number of pairs of class $c$ in the current batch; $B=\sum_{c=1}^{C}n_c$.\\
Balanced class multiplicity, $m$ & Number of pairs per class in a balanced batch; $B=mC$.\\
Frozen/LoRA & Frozen or low-rank-adapted text encoder; the neural and projection modules are specified separately.\\
\end{longtable}}

The datasets are the finger motor-imagery recordings of Ding et~al., the SSVEP Benchmark dataset of Wang et~al., and the Chinese word-reading sEEG dataset of Zheng et~al. Du-IN names an encoder. The Speech targets are Chinese words, whereas their prompts use the corresponding English translations. Prompt text shown in the figures follows the actual input wording.

\paragraph{Supervision and text representations.}
CLS trains a task-specific classification head with class labels. One-hot encoding specifies a target representation, and cross-entropy specifies a loss. LaS also uses class labels to pair neural trials with text, then predicts through matching between neural representations and class text prototypes.

Let $P_c$ be the complete prompt supplied to the text encoder, and let $T_{f,c}$ be the resulting text embedding before the model-specific projection. Text projection followed by $\ell_2$ normalization produces a text prototype. A prompt template becomes a prompt after its class variable is filled. The main performance comparison uses frozen EmbeddingGemma (300M), checkpoint \texttt{google/embeddinggemma-300m}~\citep{schechtervera2025embeddinggemma}, with one fixed, semantically matched P0 prompt per class during training and evaluation. The model input has one preceding \ptext{task: sentence similarity | query: } prefix. Attention-mask mean pooling, the model's native \texttt{2\_Dense} and \texttt{3\_Dense} layers, and $\ell_2$ normalization yield a 768-dimensional embedding. Frozen text features can be precomputed; the neural encoder and both projections remain trainable in LaS.

\paragraph{Prompt conditions.}
P0 uses template 0 for each task. P2 samples from templates 0--4 during training while preserving the class meaning; its regular validation and evaluation use template 0. Held-out templates S0--S4 are absent from template training. For their ensemble, the five sets of candidate-class logits are averaged classwise before prediction. The exact training and held-out templates appear in the respective task sections.

\paragraph{Reading the architecture tables.}
The following tables omit the batch dimension and use $C$ for the number of classes in the current task. For each encoder, the CLS and LaS paths begin at a common network boundary and list the operations performed thereafter. Neural and text projections are bias-free, and their outputs are $\ell_2$-normalized. Feature readouts differ between paths for some encoders.

\subsection{Motor Imagery (MI)}
\label{app:mi}

\paragraph{Data and main comparison.}
The finger motor-imagery data of \citet{ding2025eeg} contain 21 participants and four classes: Thumb, Index, Middle, and Pinky. The original 128-channel recordings were sampled at 1024 Hz. Offline trials are located by \texttt{Target} and the corresponding \texttt{TrialEnd} event, retaining at most 5.0 s after \texttt{Target}. We first apply common-average referencing across the 128 channels, select 32 channels, and resample to 250 Hz by the polyphase method. Each segment receives 250 zero samples at both ends before fourth-order Butterworth 4--40 Hz band-pass filtering with \texttt{lfilter}; the padding is then removed. The selected channels, in order, are Fp1, AF3, F7, F3, FC1, FC5, T7, C3, CP1, CP5, P7, P3, Pz, PO3, O1, Oz, O2, PO4, P4, P8, CP6, CP2, C4, T8, FC6, FC2, F4, F8, AF4, Fp2, Fz, and Cz.

The models receive 1.0-s windows at a 0.2-s stride; a complete 5.0-s trial yields 21 windows. Each window is standardized channelwise using its own statistics. Windows must lie entirely within a valid trial, and non-finite inputs are excluded. Physical trials are stratified by class into five folds within each participant. In each fold, the training trials from four folds are pooled across the 21 participants, and the remaining fold supplies held-out predictions. All windows from a trial inherit the same fold. Metrics are computed after combining held-out predictions within each participant, followed by an equal-weight participant average.

\begin{table}[htbp]
\centering\small
\caption{MI protocol for the main performance comparison.}
\label{tab:mi-protocol}
\begin{tabularx}{\linewidth}{P{0.91in} X}
\toprule Item & Setting\\\midrule
Trial split & Class-stratified five-fold split of physical trials within each of 21 participants; four folds train and one fold is held out. All windows of a trial stay in one fold.\\
Model input & 32 channels at 250 Hz; 1.0-s window, 0.2-s stride, and 5.0-s complete trial.\\
Candidate classes & Four finger motor-imagery classes.\\
Aggregation & Combine held-out trial predictions within each participant, then average participant metrics equally.\\\bottomrule
\end{tabularx}
\end{table}

\paragraph{Training and checkpoint selection.}
We train neural models with AdamW and cosine learning-rate decay, using weight decay and dropout for regularization. Optimization settings vary by encoder and supervision scheme. The checkpoint with the highest window-level accuracy on the held-out fold is selected for trial-level evaluation on that same fold.

\paragraph{Encoders and CLS/LaS interfaces.}
The MI encoders are EEGNet~\citep{lawhern2018eegnet}, a DeepConvNet1s variant of DeepConvNet~\citep{schirrmeister2017deep}, and EEG-Conformer~\citep{song2023eeg}. Tables~\ref{tab:mi-cls} and \ref{tab:mi-integration} specify the classification and projection paths used here.

{\footnotesize
\begin{longtable}{P{0.78in}P{1.02in}P{1.17in}P{2.13in}}
\caption{MI encoders: CLS readout and classification paths.}\label{tab:mi-cls}\\
\structhead\endfirsthead
\multicolumn{4}{l}{\footnotesize Table \ref{tab:mi-cls} (continued)}\\
\structhead\endhead
\bottomrule\endfoot
EEGNet & Common boundary & Position; shape & After separable-convolution ELU; $16\times25$\\
 & CLS readout & Average pool & Kernel/stride $(1,8)$/$(1,8)$\\
 & & Dropout & $p=0.5$\\
 & & Flatten & $16\times3\rightarrow48$\\
 & CLS head & Linear & $48\rightarrow C$\\\midrule
DeepConvNet\newline 1s & Common boundary & Position; shape & After fourth convolutional block; $200\times1\times6$\\
 & CLS head & Classification convolution & Channels $200\rightarrow C$; kernel $(1,6)$; stride 1; padding 0\\\midrule
EEG-Conformer & Common boundary & Position; shape & Transformer output; $11\times40$ tokens\\
 & CLS readout & Flatten & 440\\
 & CLS hidden 1 & Linear, ELU, dropout & $440\rightarrow256$; $p=0.5$\\
 & CLS hidden 2 & Linear, ELU, dropout & $256\rightarrow32$; $p=0.3$\\
 & CLS output & Linear & $32\rightarrow C$\\
\end{longtable}}

{\footnotesize
\begin{longtable}{P{0.78in}P{1.02in}P{1.17in}P{2.13in}}
\caption{MI encoders: LaS readout and projection paths.}\label{tab:mi-integration}\\
\structhead\endfirsthead
\multicolumn{4}{l}{\footnotesize Table \ref{tab:mi-integration} (continued)}\\
\structhead\endhead
\bottomrule\endfoot
EEGNet & Common boundary & Position; shape & After separable-convolution ELU; $16\times25$\\
 & LaS readout & Flatten & $16\times25\rightarrow400$\\
 & Neural projection & Linear & $400\rightarrow256$\\
 & Text projection & Linear & $768\rightarrow256$\\\midrule
DeepConvNet\newline 1s & Common boundary & Position; shape & After fourth convolutional block; $200\times1\times6$\\
 & LaS readout & Flatten & $200\times1\times6\rightarrow1200$\\
 & Neural projection & Linear & $1200\rightarrow256$\\
 & Text projection & Linear & $768\rightarrow256$\\\midrule
EEG-Conformer & Common boundary & Position; shape & Transformer output; $11\times40$ tokens\\
 & LaS readout & Flatten tokens & $11\times40\rightarrow440$\\
 & Neural projection & Linear & $440\rightarrow256$\\
 & Text projection & Linear & $768\rightarrow256$\\\midrule
Shared LaS setting & Both projections & Bias; output normalization & None; separate $\ell_2$ normalization\\
 & Numerical path & Precision & FP32 projection and normalization\\
\end{longtable}}

FBCSP-LDA~\citep{ang2008filter} is a separate conventional MI baseline fitted on the corresponding training folds. EEGNet's LaS feature is read before its second pooling layer, giving that pair distinct readout paths.

\paragraph{Class prompts.}
The variable \texttt{\{finger\}} takes \texttt{thumb}, \texttt{index finger}, \texttt{middle finger}, or \texttt{pinky}. Table~\ref{tab:mi-prompts} gives the English templates before the EmbeddingGemma prefix is added.

{\small
\begin{longtable}{P{0.55in}P{4.75in}}
\caption{MI P0/P2 training and held-out templates.}\label{tab:mi-prompts}\\
\toprule ID & Template text\\\midrule\endfirsthead
\multicolumn{2}{l}{\footnotesize Table \ref{tab:mi-prompts} (continued)}\\\toprule ID & Template text\\\midrule\endhead
\bottomrule\endfoot
0 (P0) & \ptext{The participant is imagining repetitive flexion and extension of the right \{finger\}.}\\
1 & \ptext{The person is performing motor imagery of moving the right \{finger\}.}\\
2 & \ptext{This EEG corresponds to imagined flexion and extension of the right \{finger\}.}\\
3 & \ptext{The participant is mentally rehearsing repetitive movement of the right \{finger\}.}\\
4 & \ptext{The intended action is flexion and extension of the right \{finger\}, imagined without physical movement.}\\\midrule
S0 & \ptext{an EEG of imagined right \{finger\} motion.}\\
S1 & \ptext{a motor-imagery signal for the right \{finger\}.}\\
S2 & \ptext{a neural pattern for right \{finger\} imagery.}\\
S3 & \ptext{brain activity from an imagined right \{finger\} action.}\\
S4 & \ptext{an imagined right \{finger\} gesture.}\\
\end{longtable}}

\begin{table}[H]
\centering\small
\caption{MI parameter counts and classification forward FLOPs.}
\label{tab:mi-complexity}
\begin{tabular*}{\linewidth}{@{\extracolsep{\fill}}lccccc@{}}
\toprule
\multirow{2}{*}{Encoder} & \multirow{2}{*}{Duration (s)} & \multicolumn{2}{c}{Params (M)} & \multicolumn{2}{c}{Forward GFLOPs}\\
\cmidrule(lr){3-4}\cmidrule(lr){5-6}
& & CLS & LaS & CLS & LaS\\\midrule
EEGNet & 5.0 & 0.002 & 0.30 & 0.22 & 0.23\\
DeepConvNet1s & 5.0 & 0.29 & 0.79 & 0.55 & 0.56\\
EEG-Conformer & 5.0 & 0.29 & 0.48 & 0.85 & 0.85\\\bottomrule
\end{tabular*}
\par\smallskip\raggedright The complete set of template-0 (T0) candidate prompts is encoded once, with a text-encoding cost of 53.1880 GFLOPs.
\end{table}

\subsection{Steady-State Visual Evoked Potentials (SSVEP)}
\label{app:ssvep}

\paragraph{Data and main comparison.}
We use the SSVEP Benchmark dataset of \citet{wang2017benchmark}. It contains 35 participants and 40 stimulus classes. Each participant completed six blocks, each with one trial per class. Neural models in the main performance comparison use nine channels (Pz, PO5, PO3, POz, PO4, PO6, O1, Oz, O2) at 250 Hz. For each trial and channel, we subtract the mean over the 0.5-s pre-stimulus interval. A sixth-order Butterworth 6--64 Hz band-pass filter is applied with \texttt{filtfilt} to the complete recorded segment before window extraction. This processing path applies no further resampling or referencing. Each extracted window is standardized channelwise using its own statistics.

Stimulus onset is $t=0$. An initial window (IW) of length $L$ occupies $[0.14,0.14+L)$ s. For random-window (RW) training, a valid sample-aligned start $a$ is drawn uniformly with $0.14\leq a\leq5.00-L$; all channels are extracted over the same interval $[a,a+L)$. The matched neural-encoder comparison uses RW training and IW evaluation at 0.6, 0.8, and 1.0 s. Approximately 300,000 RW windows are generated from each fold's training partition. The relevant time-domain inputs receive contiguous temporal masking at ratio 0.3 under the same condition for paired CLS and LaS encoders. Appendix~\ref{app:ssvep-window-augmentation} defines the augmentation operation.

Participants are divided into seven groups of five. In each fold, one group is evaluated, another is used for validation, and the remaining five groups (25 participants) are used for training. Each group serves as the evaluation group in turn. This subject split defines the Table~1 encoder comparison.

\begin{table}[htbp]
\centering\small
\caption{SSVEP protocol for the paired encoder comparison in Table~1.}
\label{tab:ssvep-protocol}
\begin{tabularx}{\linewidth}{P{0.96in} X}
\toprule Item & Setting\\\midrule
Participant split & Seven folds; each has 25 training, five validation, and five evaluation participants.\\
Training input & RW of 0.6, 0.8, or 1.0 s; temporal masking at ratio 0.3 for the applicable time-domain input.\\
Evaluation input & IW of the corresponding length, starting 0.14 s after stimulus onset.\\
Candidate classes & All 40 stimulus classes.\\\bottomrule
\end{tabularx}
\end{table}

\paragraph{Training and checkpoint selection.}
Within each cross-subject fold, we train neural models with AdamW and cosine learning-rate decay, using weight decay and dropout for regularization. Optimization settings and training duration vary by encoder, input-window length, and supervision scheme (CLS or LaS). The checkpoint with the highest classification accuracy on the validation subjects is selected for evaluation on the held-out test subjects.

\paragraph{Encoders and CLS/LaS interfaces.}
The SSVEP encoders are EEGNet~\citep{lawhern2018eegnet}, SSVEPFormer~\citep{chen2023transformer}, and MTSNet~\citep{lan2025mtsnet}. MTSNet's frequency branch zero-pads the observed short window to 1250 samples and concatenates the real and imaginary parts of its complex spectrum for $8\leq f<64$ Hz. In Tables~\ref{tab:ssvep-cls} and \ref{tab:ssvep-integration}, $T$ is the number of input samples. Durations of 0.6, 0.8, and 1.0 s correspond to $T=150,200,250$; slash-separated dimensions follow this order.

{\footnotesize
\begin{longtable}{P{0.78in}P{1.02in}P{1.17in}P{2.13in}}
\caption{SSVEP encoders: CLS readout and classification paths.}\label{tab:ssvep-cls}\\
\structhead\endfirsthead
\multicolumn{4}{l}{\footnotesize Table \ref{tab:ssvep-cls} (continued)}\\
\structhead\endhead
\bottomrule\endfoot
EEGNet & Common boundary & Position & After final dropout\\
 & CLS readout & Flattened width $d$ & 384/576/672\\
 & CLS head & Linear & $d\rightarrow C$\\\midrule
SSVEPFormer & Common boundary & Position & After frequency-domain encoder\\
 & CLS readout & Flatten & 10080\\
 & CLS hidden & Linear & $10080\rightarrow240$\\
 & & LayerNorm, GELU, dropout & Width 240; $p=0$\\
 & CLS output & Linear & $240\rightarrow C$\\\midrule
MTSNet & Common boundary & Position & After time-frequency fusion\\
 & CLS readout & Flatten, dropout & $d=18(T+560)$; $p=0.5$\\
 & CLS hidden & Linear & $d\rightarrow240$\\
 & & LayerNorm, GELU, dropout & Width 240; $p=0.5$\\
 & CLS output & Linear & $240\rightarrow C$\\
\end{longtable}}

{\footnotesize
\begin{longtable}{P{0.78in}P{1.02in}P{1.17in}P{2.13in}}
\caption{SSVEP encoders: LaS readout and projection paths.}\label{tab:ssvep-integration}\\
\structhead\endfirsthead
\multicolumn{4}{l}{\footnotesize Table \ref{tab:ssvep-integration} (continued)}\\
\structhead\endhead
\bottomrule\endfoot
EEGNet & Common boundary & Position & After final dropout\\
 & LaS readout & Flattened width $d$ & 384/576/672\\
 & Neural projection & Linear & $d\rightarrow128$\\
 & Text projection & Linear & $768\rightarrow128$\\\midrule
SSVEPFormer & Common boundary & Position & After frequency-domain encoder\\
 & LaS readout & Avg + max pooling & Elementwise sum; kernel/stride 10/10\\
 & & Flattened width & 1008\\
 & Neural projection & Linear & $1008\rightarrow256$\\
 & Text projection & Linear & $768\rightarrow256$\\\midrule
MTSNet & Common boundary & Position & After time-frequency fusion\\
 & LaS readout & Avg + max pooling & Elementwise sum; kernel/stride 10/10\\
 & & Flattened width $d$ & 1278/1368/1458\\
 & Neural projection & Linear & $d\rightarrow256$\\
 & Text projection & Linear & $768\rightarrow256$\\\midrule
Shared LaS setting & Both projections & Bias; output normalization & None; separate $\ell_2$ normalization\\
\end{longtable}}

EEGNet projects both modalities to 128 dimensions; SSVEPFormer and MTSNet use 256 dimensions. Each encoder learns its own projections.

\paragraph{Class prompts.}
The variable \texttt{\{frequency\}} is filled with the stimulus frequency in hertz. Table~\ref{tab:ssvep-prompts} reproduces the template text before the EmbeddingGemma prefix.

{\small
\begin{longtable}{P{0.55in}P{4.75in}}
\caption{SSVEP P0/P2 training and held-out templates.}\label{tab:ssvep-prompts}\\
\toprule ID & Template text\\\midrule\endfirsthead
\multicolumn{2}{l}{\footnotesize Table \ref{tab:ssvep-prompts} (continued)}\\\toprule ID & Template text\\\midrule\endhead
\bottomrule\endfoot
0 (P0) & \ptext{A neural SSVEP signal at \{frequency\} Hz.}\\
1 & \ptext{A \{frequency\}-Hz SSVEP.}\\
2 & \ptext{SSVEP signal, frequency \{frequency\} Hz.}\\
3 & \ptext{Steady-state visual evoked potential at \{frequency\} Hz.}\\
4 & \ptext{The neural response SSVEP corresponds to a \{frequency\}-Hz rhythm}\\\midrule
S0 & \ptext{An EEG response at \{frequency\} Hz.}\\
S1 & \ptext{An SSVEP at \{frequency\} Hz.}\\
S2 & \ptext{A response to \{frequency\} Hz flicker.}\\
S3 & \ptext{A response to \{frequency\} Hz stimulation.}\\
S4 & \ptext{A steady-state response at \{frequency\} Hz.}\\
\end{longtable}}

\begin{table}[H]
\centering\small
\caption{SSVEP parameter counts and classification forward FLOPs.}
\label{tab:ssvep-complexity}
\begin{tabular*}{\linewidth}{@{\extracolsep{\fill}}lccccc@{}}
\toprule
\multirow{2}{*}{Encoder} & \multirow{2}{*}{Duration (s)} & \multicolumn{2}{c}{Params (M)} & \multicolumn{2}{c}{Forward GFLOPs}\\
\cmidrule(lr){3-4}\cmidrule(lr){5-6}
& & CLS & LaS & CLS & LaS\\\midrule
EEGNet & 0.6 & 0.04 & 0.32 & 0.04 & 0.06\\
 & 0.8 & 0.05 & 0.37 & 0.05 & 0.07\\
 & 1.0 & 0.05 & 0.39 & 0.07 & 0.08\\
MTSNet & 0.6 & 9.92 & 7.37 & 0.33 & 0.34\\
 & 0.8 & 10.87 & 8.12 & 0.36 & 0.37\\
 & 1.0 & 11.89 & 8.95 & 0.40 & 0.41\\
SSVEPFormer & 0.6 & 3.09 & 1.11 & 0.05 & 0.06\\
 & 0.8 & 3.09 & 1.11 & 0.05 & 0.06\\
 & 1.0 & 3.09 & 1.11 & 0.05 & 0.06\\\bottomrule
\end{tabular*}
\par\smallskip\raggedright The complete set of template-0 (T0) candidate prompts is encoded once, with a text-encoding cost of 1087.5419 GFLOPs.
\end{table}
\FloatBarrier

\subsection{Overt Word Reading (Speech)}
\label{app:speech}

\paragraph{Data and main comparison.}
We use the Chinese word-reading sEEG data released by \citet{zheng2024duin}, comprising 12 participants and 61 word classes. The published 3.0-s physical trials are used without additional event realignment, with ten fixed channels per participant. Acquisition, upstream filtering, and bipolar referencing follow the released data. After loading these segments, we select the channels and sampling rate required by each encoder and apply normalization. Du-IN receives 3000 samples at 1000 Hz; EEGNet and EEG-Conformer receive 750 samples at 250 Hz; CNN-BiGRU receives 90 samples at 30 Hz. All inputs span the same 3.0-s trial.

Within each participant's run, physical trials are shuffled and approximately 80\% of each class enters training. Remaining trials are pooled and split approximately equally between validation and evaluation, yielding an overall split near 80\%/10\%/10\%. Training portions from the participants are pooled, and CLS and LaS share the split. At each input sampling rate, training, validation, and evaluation sets are standardized separately using their own unlabeled per-participant, per-channel statistics. During training, all four encoders use the random temporal shift from the Du-IN downstream protocol: each 3.0-s segment is shifted left or right by up to 0.3 s and zero-padded. Validation and evaluation use no such shift~\citep{zheng2024duin}. Each trial receives one prediction.

\begin{table}[htbp]
\centering\small
\caption{Three-second Speech trial inputs for the main performance comparison.}
\label{tab:speech-protocol}
\begin{tabularx}{\linewidth}{P{1.26in} P{1.15in} X}
\toprule Encoder & Sampling rate & Samples per trial and evaluation unit\\\midrule
Du-IN & 1000 Hz & 3000; one prediction per trial\\
EEGNet & 250 Hz & 750; one prediction per trial\\
EEG-Conformer & 250 Hz & 750; one prediction per trial\\
CNN-BiGRU & 30 Hz & 90; one prediction per trial\\\bottomrule
\end{tabularx}
\end{table}

\paragraph{Training and checkpoint selection.}
Speech training follows the downstream setup of Du-IN~\citep{zheng2024duin}. We use AdamW with weight decay and learning-rate warmup followed by cosine decay. Encoder-specific settings, including batch size and training duration, are set separately for CLS and LaS. For both schemes, the checkpoint with the highest validation accuracy is selected for evaluation on the test split.

\paragraph{Encoders and CLS/LaS interfaces.}
All four neural backbones are trained from scratch. Tables~\ref{tab:speech-cls} and \ref{tab:speech-integration} specify their readout and output paths.

{\footnotesize
\begin{longtable}{P{0.78in}P{1.02in}P{1.17in}P{2.13in}}
\caption{Speech encoders: CLS readout and classification paths.}\label{tab:speech-cls}\\
\structhead\endfirsthead
\multicolumn{4}{l}{\footnotesize Table \ref{tab:speech-cls} (continued)}\\
\structhead\endhead
\bottomrule\endfoot
EEGNet & Common boundary & Position & Before final pooling\\
 & Backbone end & Average pool & Kernel/stride $(1,8)$/$(1,8)$\\
 & & Dropout & $p=0.5$\\
 & CLS readout & Flatten & 144\\
 & CLS head & Linear; output & $144\rightarrow C$; raw logits\\\midrule
CNN-BiGRU & Common boundary & Position & Bidirectional GRU output\\
 & CLS readout & Last sequence position & 256\\
 & CLS hidden & Linear & $256\rightarrow128$\\
 & CLS output & Linear, sigmoid & $128\rightarrow C$\\\midrule
EEG-Conformer & Common boundary & Position; shape & Transformer output; $50\times128$ tokens\\
 & CLS readout & Flatten & 6400\\
 & CLS hidden & Linear & $6400\rightarrow128$\\
 & CLS output & Linear, sigmoid & $128\rightarrow C$\\\midrule
Du-IN & Common boundary & Position; shape & Transformer output; $30\times160$ tokens\\
 & CLS readout & Flatten & 4800\\
 & CLS hidden & Linear, ReLU & $4800\rightarrow128$\\
 & CLS output & Linear, sigmoid & $128\rightarrow C$\\
\end{longtable}}

{\footnotesize
\begin{longtable}{P{0.78in}P{1.02in}P{1.17in}P{2.13in}}
\caption{Speech encoders: LaS readout and projection paths.}\label{tab:speech-integration}\\
\structhead\endfirsthead
\multicolumn{4}{l}{\footnotesize Table \ref{tab:speech-integration} (continued)}\\
\structhead\endhead
\bottomrule\endfoot
EEGNet & Common boundary & Position & Before final pooling\\
 & Backbone end & Average pool & Kernel/stride $(1,8)$/$(1,8)$\\
 & & Dropout & $p=0.1$\\
 & LaS readout & Flatten & 144\\
 & Neural projection & Linear & $144\rightarrow128$\\
 & Text projection & Linear & $768\rightarrow128$\\\midrule
CNN-BiGRU & Common boundary & Position & Bidirectional GRU output\\
 & LaS readout & Concatenated final hidden states & 256\\
 & Neural projection & Linear & $256\rightarrow256$\\
 & Text projection & Linear & $768\rightarrow256$\\\midrule
EEG-Conformer & Common boundary & Position; shape & Transformer output; $50\times128$ tokens\\
 & LaS readout 1 & Tokenwise bias-free linear & $128\rightarrow32$\\
 & LaS readout 2 & Dropout, flatten & $p=0.1$; 1600\\
 & Neural projection & Linear & $1600\rightarrow256$\\
 & Text projection & Linear & $768\rightarrow256$\\\midrule
Du-IN & Common boundary & Position; shape & Transformer output; $30\times160$ tokens\\
 & LaS readout 1 & Tokenwise bias-free linear & $160\rightarrow40$\\
 & LaS readout 2 & Dropout, flatten & $p=0.1$; 1200\\
 & Neural projection & Linear & $1200\rightarrow256$\\
 & Text projection & Linear & $768\rightarrow256$\\\midrule
Shared LaS setting & Both projections & Bias; output normalization & None; separate $\ell_2$ normalization\\
\end{longtable}}

In the Speech EEG-Conformer and Du-IN Transformer backbones, the attention and first feed-forward dropout ratios are 0.2 for CLS and 0.1 for LaS; the second feed-forward dropout is zero in both paths. The token outputs in the two tables identify the corresponding layer, while their values depend on the branch configuration.

Speech EEGNet applies multiclass cross-entropy to raw logits and class indices. The Du-IN, CNN-BiGRU, and EEG-Conformer CLS implementations call \texttt{F.cross\_entropy} on sigmoid outputs with one-hot targets, following the official Du-IN implementation~\citep{zheng2024duin}.\footnote{Official Du-IN code, commit \texttt{c2e289c}: \url{https://github.com/liulab-repository/Du-IN/tree/c2e289c353005d5ac184860909e7b5be81da31ba}. The classification heads are defined in \texttt{models/duin}, \texttt{models/cnn\_gru}, and \texttt{models/eeg\_cfmr}.}

\paragraph{Subject-specific input frontends.}
Figure~2c depicts routing by participant identity into a subject-specific input frontend, followed by the shared part of that encoder instance. Table~\ref{tab:speech-frontends} specifies the layers before the shared boundary for each architecture. The frontend outputs are restored to the original batch order. CLS and LaS are trained separately for each encoder. Shapes omit the batch dimension; convolutional channel counts are written input-to-output.

{\footnotesize
\begin{longtable}{P{0.78in}P{1.02in}P{1.17in}P{2.13in}}
\caption{Subject-specific input frontends and shared-backbone boundaries for Speech decoding.}\label{tab:speech-frontends}\\
\structhead\endfirsthead
\multicolumn{4}{l}{\footnotesize Table \ref{tab:speech-frontends} (continued)}\\
\structhead\endhead
\bottomrule\endfoot
Du-IN & Input & Shape & $10\times3000$ (channels $\times$ samples)\\
 & Segmentation & Output & 30 segments of $10\times100$\\
 & Temporal Conv1d 1 & Channels; kernel; stride; padding & $10\rightarrow128$; 19; 10; 9\\
 & Normalization 1 & BatchNorm1d & 128 channels; $\epsilon=10^{-5}$; momentum 0.1\\
 & Temporal Conv1d 2 & Channels; kernel; stride; padding & $128\rightarrow128$; 3; 1; 1\\
 & Normalization 2 & BatchNorm1d & 128 channels; $\epsilon=10^{-5}$; momentum 0.1\\
 & Temporal Conv1d 3 & Channels; kernel; stride; padding & $128\rightarrow16$; 3; 1; 1\\
 & Normalization 3 & BatchNorm1d & 16 channels; $\epsilon=10^{-5}$; momentum 0.1\\
 & Frontend readout & Reshape; output & Flatten $16\times10$ per segment; 30 tokens of width 160\\
 & Shared boundary & First shared module & Temporal embedding, then Transformer\\\midrule
EEGNet & Input & Shape & $10\times750$; Conv2d input $1\times10\times750$\\
 & Temporal Conv2d & Channels; kernel; stride; padding & $1\rightarrow8$; $(1,80)$; 1; same\\
 & Normalization 1 & BatchNorm2d & 8 channels; $\epsilon=10^{-3}$; momentum 0.01\\
 & Spatial Conv2d & Channels; kernel; stride; groups & $8\rightarrow16$; $(10,1)$; 1; 8\\
 & Normalization 2 & BatchNorm2d; output & 16 channels; $\epsilon=10^{-3}$; momentum 0.01; $16\times1\times750$\\
 & Shared boundary & First shared module & ELU, then pooling, separable convolution, and readout\\\midrule
CNN-BiGRU & Input & Shape & $10\times90$ (channel-first)\\
 & Temporal Conv1d & Channels; kernel; stride; padding & $10\rightarrow128$; 5; 3; 2\\
 & Normalization & BatchNorm1d; output & 128 channels; $\epsilon=10^{-5}$; momentum 0.1; transpose to 30 tokens of width 128\\
 & Shared boundary & First shared module & Bidirectional GRU, then branch-specific readout\\\midrule
EEG-Conformer & Input & Shape & $10\times750$; Conv2d input $1\times10\times750$\\
 & Temporal Conv2d & Channels; kernel; stride; padding & $1\rightarrow128$; $(1,25)$; $(1,1)$; $(0,12)$\\
 & Spatial Conv2d & Channels; kernel; stride; padding & $128\rightarrow128$; $(10,1)$; 1; 0\\
 & Normalization & BatchNorm2d; output & 128 channels; $\epsilon=10^{-5}$; momentum 0.1; $128\times1\times750$\\
 & Shared boundary & First shared module & ELU, then pooling, token projection, and Transformer\\
\end{longtable}}

The Du-IN, CNN-BiGRU, and EEG-Conformer frontend convolutions use $\mathrm{groups}=1$ and bias. EEGNet uses a bias-free temporal convolution and a grouped spatial convolution. Activations and pooling follow these frontends in the shared backbone. Attribution groups are defined in Appendix~\ref{app:attribution-method}.

\paragraph{Class prompts.}
The source targets are 61 Chinese words. Their English translations, listed in Table~4 of \citet{zheng2024duin}, fill \texttt{\{label\}}. Table~\ref{tab:speech-prompts} reproduces the templates before the EmbeddingGemma prefix.

{\small
\begin{longtable}{P{0.55in}P{4.75in}}
\caption{Speech P0/P2 training and held-out templates.}\label{tab:speech-prompts}\\
\toprule ID & Template text\\\midrule\endfirsthead
\multicolumn{2}{l}{\footnotesize Table \ref{tab:speech-prompts} (continued)}\\\toprule ID & Template text\\\midrule\endhead
\bottomrule\endfoot
0 (P0) & \ptext{The participant is speaking the word: \{label\}.}\\
1 & \ptext{This sEEG recording corresponds to the spoken word: \{label\}.}\\
2 & \ptext{The intended speech content is: \{label\}.}\\
3 & \ptext{The neural activity was recorded while the participant said: \{label\}.}\\
4 & \ptext{The target speech label for this intracranial neural signal is: \{label\}.}\\\midrule
S0 & \ptext{An sEEG response for spoken \{label\}.}\\
S1 & \ptext{A neural signal for saying \{label\}.}\\
S2 & \ptext{Brain activity from uttering \{label\}.}\\
S3 & \ptext{An intracranial response to the word \{label\}.}\\
S4 & \ptext{A speech-related neural pattern for \{label\}.}\\
\end{longtable}}

\begin{table}[H]
\centering\small
\caption{Speech parameter counts and classification forward FLOPs.}
\label{tab:speech-complexity}
\begin{tabular*}{\linewidth}{@{\extracolsep{\fill}}lccccc@{}}
\toprule
\multirow{2}{*}{Encoder} & \multirow{2}{*}{Duration (s)} & \multicolumn{2}{c}{Params (M)} & \multicolumn{2}{c}{Forward GFLOPs}\\
\cmidrule(lr){3-4}\cmidrule(lr){5-6}
& & CLS & LaS & CLS & LaS\\\midrule
EEGNet & 3.0 & 0.02 & 0.13 & 0.01 & 0.02\\
CNN-BiGRU & 3.0 & 0.62 & 0.84 & 0.03 & 0.05\\
EEG-Conformer & 3.0 & 4.18 & 3.96 & 0.45 & 0.47\\
Du-IN & 3.0 & 5.05 & 4.94 & 0.27 & 0.29\\\bottomrule
\end{tabular*}
\par\smallskip\raggedright The complete set of template-0 (T0) candidate prompts is encoded once, with a text-encoding cost of 811.1166 GFLOPs.
\end{table}

\subsection{Metrics and Aggregation}
\label{app:metrics}

For $N$ evaluation trials within a participant, accuracy is
\begin{equation}
\operatorname{Accuracy}=\frac{1}{N}\sum_{i=1}^{N}\mathbf 1[\hat y_i=y_i].
\end{equation}
For MI, we also report Cohen's $\kappa=(p_o-p_e)/(1-p_e)$, where $p_o$ is that participant's trial accuracy and $p_e=\sum_c p_{\mathrm{true}}(c)p_{\mathrm{pred}}(c)$ is determined by the marginal proportions of true and predicted classes. SSVEP estimated information transfer rate (ITR) uses 40 candidate classes, accuracy $p$, and decision time $T=L+0.5$ s:
\begin{equation}
\operatorname{ITR}=\frac{60}{T}\left[\log_2 C+p\log_2 p+(1-p)\log_2\frac{1-p}{C-1}\right],\qquad C=40.
\end{equation}
Here $L$ is the window duration, and $0\log_2 0$ takes its continuous limit of zero. ITR is computed within each participant before equal-weight averaging across participants.

For MI trial prediction, LaS averages window logits; CLS and FBCSP-LDA use majority voting over window classes. A voting tie is resolved by the tied classes' mean scores, then by a fixed class index if necessary. Accuracy and $\kappa$ are computed from each participant's combined five-fold held-out trials and averaged equally across the 21 participants. SSVEP accuracy is computed for each of the 35 evaluation participants and then averaged equally. The Speech main comparison uses five independent training runs on the same physical-trial split: run results are first averaged within each participant, then averaged equally across 12 participants.

\subsection{Adaptation Algorithms for MI and SSVEP}
\label{app:adaptation}

The neural methods compared in Figures~4a and 4b initialize from LaS-trained source decoders and update different modules using class-label supervision during adaptation. Let $E_0$ be the source neural encoder, $h=E_0(X)$ its representation, and $W_0$ the source neural projection. For the current candidate set $\mathcal C$, $t_c$ denotes a fixed unit-length text prototype, and $\nu$ denotes $\ell_2$ normalization. Methods using candidate-class scores optimize multiclass cross-entropy over $\mathcal C$. Tables~\ref{tab:mi-integration} and \ref{tab:ssvep-integration} locate the encoder readouts; projected features are read after that model's neural projection.

\paragraph{Probe heads.}
A linear probe freezes the source feature extractor, reads an encoder or projected feature $r(X)$, and trains a linear classifier $g(r)=Hr+b$. An MLP probe trains a nonlinear head $g_\theta(r)$ on the same frozen features. Encoder-feature and projected-feature variants differ in their readout position. Both heads use cross-entropy on calibration labels, and their output logits are used at inference.

\paragraph{Adaptation using text references.}
LP++ combines learned class weights with text representations over frozen vision-language features~\citep{huang2024lpplusplus}. Here the source projected neural feature $v=\nu(hW_0)$ and prototypes $t_c$ are fixed, while class weights $w_c$ and coefficients $\alpha_c$ are trained with scores
\begin{equation}
z_c(X)=v^\top(w_c+\alpha_c t_c).
\end{equation}
ProLIP motivates adapting a projection while penalizing deviation from its source value~\citep{fahes2026prolip}. In this neural-decoding implementation, the neural encoder and text prototypes are fixed, the neural projection $W$ starts from $W_0$, and $\gamma\langle\nu(hW),t_c\rangle$ provides class logits. Its calibration objective adds $\lambda\lVert W-W_0\rVert_F^2$ to cross-entropy.

\paragraph{DTS-ProLIP.}
We constrain the trainable neural-projection offset to the discriminative subspace of the current candidate prototypes. Let $W_0\in\mathbb R^{d_f\times d}$. After centering $\{t_c\}_{c\in\mathcal C}$, let $U\in\mathbb R^{d\times r}$ contain an orthonormal basis for their span, with $U^\top U=I_r$. Only $A\in\mathbb R^{d_f\times r}$ is trained, starting at zero:
\begin{equation}
W(A)=W_0+AU^\top,\qquad A_0=0.
\end{equation}
The neural encoder, text prototypes, and source projection remain fixed. The calibration objective is
\begin{equation}
\mathcal L(A)=\frac{1}{M}\sum_{i=1}^{M}\operatorname{CE}\!\left(
\left\{\gamma\left\langle\nu(h_iW(A)),t_c\right\rangle\right\}_{c\in\mathcal C},y_i\right)
+\lambda\lVert W(A)-W_0\rVert_F^2.
\end{equation}
Because $U$ has orthonormal columns, $\lVert W(A)-W_0\rVert_F^2=\lVert A\rVert_F^2$.

\paragraph{Updates beyond a probe.}
MLP fine-tuning updates an encoder or projection branch along with its MLP readout head. Full fine-tuning starts from the source model and updates the corresponding neural modules in its encoder- or projection-centered variant. The names \emph{Encoder} and \emph{Projection} identify the feature readout and update path. Classification uses the trained head or the corresponding candidate-prototype scores.

\clearpage
\endgroup
\setcounter{table}{0}
\setcounter{figure}{0}
\setcounter{equation}{0}
\renewcommand{\thetable}{B.\arabic{table}}
\renewcommand{\thefigure}{B.\arabic{figure}}
\section{Additional Experimental Results}
\label{app:additional-results}

This appendix reports subject-level results for Table~1 and Figure~3 and aggregate results for the other analyses, using the evaluation units defined for each experiment. Dataset splits and model-specific protocols are specified in Appendix~\ref{app:experimental-details}, attribution methods and loss references in Appendix~\ref{app:theory}, and SSVEP window analyses in Appendix~D.

\subsection{Main Results, Candidate-Set Reuse, and Adaptation}
\label{app:b-main-adaptation}

\paragraph{SSVEP accuracy and estimated ITR under initial-window evaluation.}
Table~\ref{tab:b-ssvep-itrs} reports accuracy and estimated ITR at five IW durations. Neural CLS baselines use IW training, LaS (MTSNet) uses IW training with inter-trial recombination, and every method is evaluated on the initial Physical window. FBCCA directly classifies the window~\citep{chen2015filter}. This experiment has a different training protocol from the RW-to-IW paired comparison in Table~1. Stars in Table~\ref{tab:b-ssvep-itrs} denote unadjusted, two-sided paired $t$-tests, performed separately for accuracy and estimated ITR.
\begin{table}[!htbp]
\centering
\caption{SSVEP subject-mean accuracy (\%) and estimated ITR (bits/min) across five initial-window durations. Stars indicate unadjusted paired-test significance: $^{*}p<.05$, $^{**}p<.01$, and $^{***}p<.001$.}
\label{tab:b-ssvep-itrs}
\includegraphics[width=\linewidth]{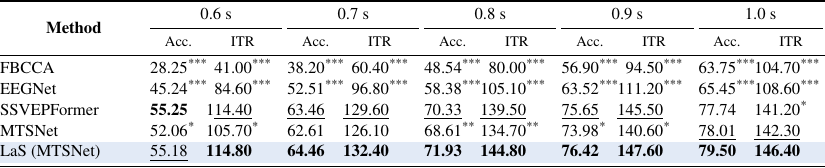}
\end{table}

\paragraph{Subject-level data for Table 1.}
The MI table follows the same five out-of-fold trial splits for all methods, including the separately trained FBCSP-LDA reference~\citep{ang2008filter}. In the following tables, C denotes supervised CLS and L denotes LaS. EEGNet, DeepConvNet1s, and EEG-Conformer follow their respective neural-encoder implementations~\citep{lawhern2018eegnet,schirrmeister2017deep,song2023eeg}. Values in paired C/L columns are subject-level metrics for the same participant.

{\scriptsize\setlength{\tabcolsep}{2.2pt}
\begin{longtable}{@{}lrrrrrrr@{}}
\caption{MI four-class trial accuracy by subject (\%). C and L denote supervised CLS and LaS.}\label{tab:b-mi-acc}\\
\toprule
Subject & FBCSP & EEGNet C & EEGNet L & DeepConv C & DeepConv L & Conformer C & Conformer L \\
\midrule
\endfirsthead
\toprule
Subject & FBCSP & EEGNet C & EEGNet L & DeepConv C & DeepConv L & Conformer C & Conformer L \\
\midrule
\endhead
S01 & 32.00 & 41.17 & 38.67 & 39.17 & 42.17 & 42.33 & 42.00 \\
S02 & 36.83 & 42.33 & 37.00 & 39.83 & 44.33 & 44.17 & 46.83 \\
S03 & 27.17 & 35.17 & 32.67 & 43.83 & 44.67 & 44.00 & 44.50 \\
S04 & 28.91 & 39.84 & 36.25 & 38.44 & 38.59 & 36.88 & 35.94 \\
S05 & 28.91 & 38.91 & 40.63 & 37.97 & 40.47 & 40.63 & 43.91 \\
S06 & 30.67 & 36.00 & 37.00 & 35.33 & 34.33 & 37.00 & 38.33 \\
S07 & 37.17 & 42.67 & 45.50 & 47.00 & 53.00 & 50.33 & 56.33 \\
S08 & 29.50 & 39.67 & 41.00 & 37.00 & 36.83 & 39.83 & 42.67 \\
S09 & 39.67 & 44.33 & 48.50 & 49.17 & 56.00 & 49.83 & 54.17 \\
S10 & 30.33 & 35.17 & 34.50 & 38.67 & 37.67 & 36.17 & 34.50 \\
S11 & 31.33 & 36.33 & 40.17 & 36.17 & 34.67 & 39.33 & 41.67 \\
S12 & 31.00 & 33.17 & 36.50 & 34.33 & 31.67 & 38.50 & 38.67 \\
S13 & 31.29 & 38.39 & 37.10 & 38.06 & 39.52 & 40.97 & 41.61 \\
S14 & 27.34 & 38.28 & 33.13 & 31.72 & 31.56 & 32.34 & 32.03 \\
S15 & 34.69 & 38.44 & 37.34 & 39.06 & 42.03 & 44.06 & 42.34 \\
S16 & 28.91 & 31.88 & 36.56 & 30.31 & 28.75 & 35.16 & 37.19 \\
S17 & 36.25 & 44.22 & 44.22 & 40.94 & 47.97 & 47.34 & 49.69 \\
S18 & 28.59 & 38.91 & 36.09 & 35.16 & 36.88 & 39.84 & 41.25 \\
S19 & 53.13 & 49.53 & 61.09 & 55.31 & 62.81 & 72.19 & 72.34 \\
S20 & 31.72 & 32.19 & 32.03 & 34.22 & 36.09 & 32.97 & 34.84 \\
S21 & 26.41 & 29.69 & 26.25 & 31.41 & 31.41 & 31.09 & 32.66 \\
\bottomrule
\end{longtable}}

{\scriptsize\setlength{\tabcolsep}{2.2pt}
\begin{longtable}{@{}lrrrrrrr@{}}
\caption{MI four-class Cohen's $\kappa$ by subject.}\label{tab:b-mi-kappa}\\
\toprule
Subject & FBCSP & EEGNet C & EEGNet L & DeepConv C & DeepConv L & Conformer C & Conformer L \\
\midrule
\endfirsthead
\toprule
Subject & FBCSP & EEGNet C & EEGNet L & DeepConv C & DeepConv L & Conformer C & Conformer L \\
\midrule
\endhead
S01 & 0.09 & 0.22 & 0.18 & 0.19 & 0.23 & 0.23 & 0.23 \\
S02 & 0.16 & 0.23 & 0.16 & 0.20 & 0.26 & 0.26 & 0.29 \\
S03 & 0.03 & 0.14 & 0.10 & 0.25 & 0.26 & 0.25 & 0.26 \\
S04 & 0.05 & 0.20 & 0.15 & 0.18 & 0.18 & 0.16 & 0.15 \\
S05 & 0.05 & 0.19 & 0.21 & 0.17 & 0.21 & 0.21 & 0.25 \\
S06 & 0.08 & 0.15 & 0.16 & 0.14 & 0.12 & 0.16 & 0.18 \\
S07 & 0.16 & 0.24 & 0.27 & 0.29 & 0.37 & 0.34 & 0.42 \\
S08 & 0.06 & 0.20 & 0.21 & 0.16 & 0.16 & 0.20 & 0.24 \\
S09 & 0.20 & 0.26 & 0.31 & 0.32 & 0.41 & 0.33 & 0.39 \\
S10 & 0.07 & 0.14 & 0.13 & 0.18 & 0.17 & 0.15 & 0.13 \\
S11 & 0.08 & 0.15 & 0.20 & 0.15 & 0.13 & 0.19 & 0.22 \\
S12 & 0.08 & 0.11 & 0.15 & 0.12 & 0.09 & 0.18 & 0.18 \\
S13 & 0.08 & 0.18 & 0.16 & 0.17 & 0.19 & 0.21 & 0.22 \\
S14 & 0.03 & 0.18 & 0.11 & 0.09 & 0.09 & 0.10 & 0.09 \\
S15 & 0.13 & 0.18 & 0.16 & 0.19 & 0.23 & 0.25 & 0.23 \\
S16 & 0.05 & 0.09 & 0.15 & 0.07 & 0.05 & 0.14 & 0.16 \\
S17 & 0.15 & 0.26 & 0.26 & 0.21 & 0.31 & 0.30 & 0.33 \\
S18 & 0.05 & 0.19 & 0.15 & 0.14 & 0.16 & 0.20 & 0.22 \\
S19 & 0.38 & 0.33 & 0.48 & 0.40 & 0.50 & 0.63 & 0.63 \\
S20 & 0.09 & 0.10 & 0.09 & 0.12 & 0.15 & 0.11 & 0.13 \\
S21 & 0.02 & 0.06 & 0.02 & 0.09 & 0.09 & 0.08 & 0.10 \\
\bottomrule
\end{longtable}}

The SSVEP tables give 35 subject-level accuracies for each encoder and duration in the RW-training to IW-evaluation comparison. Neural-model identities are EEGNet, MTSNet, and SSVEPFormer~\citep{lawhern2018eegnet,lan2025mtsnet,chen2023transformer}. C and L share the same Physical evaluation trials within a row.

{\scriptsize\setlength{\tabcolsep}{3.2pt}
\begin{longtable}{@{}lrrrrrr@{}}
\caption{SSVEP subject accuracy (\%) for EEGNet after RW training and IW evaluation.}\label{tab:b-ss-eegnet}\\
\toprule
Subject & 0.6 C & 0.6 L & 0.8 C & 0.8 L & 1.0 C & 1.0 L \\
\midrule
\endfirsthead
\toprule
Subject & 0.6 C & 0.6 L & 0.8 C & 0.8 L & 1.0 C & 1.0 L \\
\midrule
\endhead
0 & 30.42 & 28.33 & 45.42 & 41.25 & 53.33 & 52.08 \\
1 & 33.75 & 28.75 & 40.83 & 39.58 & 50.00 & 54.17 \\
2 & 48.75 & 47.08 & 58.33 & 57.92 & 75.00 & 72.08 \\
3 & 36.67 & 39.58 & 64.58 & 55.83 & 70.00 & 57.50 \\
4 & 20.42 & 19.17 & 34.58 & 30.00 & 43.75 & 38.75 \\
5 & 27.50 & 32.08 & 51.67 & 47.50 & 66.25 & 59.17 \\
6 & 19.58 & 23.75 & 38.33 & 37.08 & 49.17 & 48.75 \\
7 & 15.00 & 19.58 & 24.17 & 24.17 & 36.25 & 41.67 \\
8 & 24.58 & 25.00 & 42.50 & 42.92 & 53.75 & 54.17 \\
9 & 26.67 & 32.50 & 31.67 & 39.17 & 43.75 & 53.33 \\
10 & 6.25 & 6.67 & 4.17 & 5.83 & 6.25 & 10.42 \\
11 & 22.50 & 25.42 & 37.92 & 33.75 & 53.75 & 52.08 \\
12 & 21.25 & 15.42 & 24.58 & 19.17 & 27.92 & 29.17 \\
13 & 14.58 & 17.50 & 20.42 & 20.83 & 28.33 & 20.42 \\
14 & 22.92 & 21.25 & 38.33 & 39.58 & 44.58 & 45.83 \\
15 & 16.67 & 12.92 & 25.00 & 27.08 & 37.92 & 35.83 \\
16 & 32.92 & 37.92 & 49.58 & 50.00 & 59.58 & 55.83 \\
17 & 42.50 & 36.67 & 49.58 & 44.58 & 63.33 & 61.25 \\
18 & 9.58 & 11.25 & 17.92 & 15.00 & 22.08 & 24.17 \\
19 & 19.17 & 20.83 & 39.17 & 33.33 & 52.08 & 47.50 \\
20 & 18.33 & 20.00 & 36.67 & 35.83 & 55.00 & 47.50 \\
21 & 43.33 & 42.08 & 69.58 & 64.58 & 80.83 & 78.33 \\
22 & 25.42 & 27.50 & 53.75 & 50.42 & 62.50 & 58.75 \\
23 & 38.75 & 32.92 & 57.92 & 61.25 & 76.67 & 75.00 \\
24 & 28.33 & 25.00 & 40.83 & 41.67 & 45.83 & 50.83 \\
25 & 30.42 & 37.92 & 49.58 & 50.42 & 55.42 & 43.75 \\
26 & 45.00 & 41.25 & 70.42 & 66.25 & 81.25 & 79.58 \\
27 & 38.75 & 35.00 & 55.83 & 61.67 & 65.83 & 62.92 \\
28 & 13.33 & 12.50 & 16.67 & 17.50 & 25.42 & 30.42 \\
29 & 16.25 & 22.08 & 38.75 & 38.75 & 48.33 & 40.42 \\
30 & 37.92 & 37.08 & 48.75 & 55.83 & 70.00 & 61.25 \\
31 & 47.92 & 47.92 & 69.17 & 67.50 & 78.75 & 75.42 \\
32 & 16.25 & 17.50 & 18.75 & 21.25 & 27.50 & 29.58 \\
33 & 34.58 & 34.17 & 46.67 & 49.17 & 61.25 & 55.83 \\
34 & 10.42 & 9.58 & 25.42 & 27.92 & 35.00 & 41.25 \\
\bottomrule
\end{longtable}}

{\scriptsize\setlength{\tabcolsep}{3.2pt}
\begin{longtable}{@{}lrrrrrr@{}}
\caption{SSVEP subject accuracy (\%) for MTSNet after RW training and IW evaluation.}\label{tab:b-ss-mts}\\
\toprule
Subject & 0.6 C & 0.6 L & 0.8 C & 0.8 L & 1.0 C & 1.0 L \\
\midrule
\endfirsthead
\toprule
Subject & 0.6 C & 0.6 L & 0.8 C & 0.8 L & 1.0 C & 1.0 L \\
\midrule
\endhead
0 & 35.83 & 35.83 & 54.58 & 48.33 & 62.08 & 61.67 \\
1 & 56.67 & 55.42 & 72.92 & 73.75 & 82.92 & 85.83 \\
2 & 65.00 & 69.58 & 87.08 & 87.92 & 94.17 & 93.33 \\
3 & 62.08 & 63.33 & 82.92 & 75.42 & 87.50 & 88.75 \\
4 & 35.83 & 39.58 & 70.42 & 63.75 & 77.50 & 77.50 \\
5 & 31.25 & 33.75 & 56.25 & 63.75 & 72.08 & 68.75 \\
6 & 32.50 & 37.50 & 55.00 & 67.08 & 71.25 & 75.83 \\
7 & 22.50 & 19.17 & 34.17 & 36.67 & 47.92 & 43.75 \\
8 & 31.25 & 33.75 & 42.50 & 45.42 & 62.92 & 63.75 \\
9 & 27.50 & 32.92 & 50.83 & 51.67 & 67.08 & 67.08 \\
10 & 7.92 & 11.67 & 8.75 & 13.33 & 14.17 & 16.25 \\
11 & 26.25 & 28.33 & 55.83 & 52.08 & 69.17 & 63.75 \\
12 & 19.58 & 20.42 & 34.17 & 27.92 & 55.83 & 47.92 \\
13 & 22.92 & 34.58 & 29.17 & 47.50 & 38.75 & 59.17 \\
14 & 27.50 & 31.25 & 41.67 & 45.83 & 55.42 & 55.83 \\
15 & 14.17 & 18.33 & 31.67 & 40.83 & 51.67 & 49.17 \\
16 & 40.83 & 37.92 & 53.75 & 54.58 & 63.75 & 61.67 \\
17 & 42.92 & 48.75 & 59.58 & 60.83 & 68.33 & 72.50 \\
18 & 7.08 & 9.17 & 12.08 & 14.58 & 24.17 & 25.42 \\
19 & 32.08 & 37.50 & 53.33 & 51.25 & 65.00 & 68.33 \\
20 & 17.92 & 18.75 & 48.33 & 48.33 & 62.92 & 64.17 \\
21 & 50.42 & 54.17 & 77.92 & 82.92 & 90.00 & 90.42 \\
22 & 43.75 & 39.17 & 70.83 & 71.67 & 79.17 & 85.83 \\
23 & 39.17 & 43.75 & 65.83 & 69.58 & 73.75 & 83.75 \\
24 & 43.33 & 39.17 & 52.08 & 61.67 & 70.00 & 73.33 \\
25 & 47.08 & 50.42 & 67.50 & 69.17 & 80.42 & 77.50 \\
26 & 46.25 & 51.67 & 74.17 & 74.58 & 85.83 & 87.92 \\
27 & 50.83 & 47.92 & 77.92 & 78.33 & 85.00 & 88.33 \\
28 & 16.25 & 15.83 & 24.17 & 21.25 & 37.50 & 35.00 \\
29 & 29.58 & 35.42 & 51.25 & 60.83 & 70.42 & 67.50 \\
30 & 50.83 & 49.58 & 71.25 & 72.50 & 82.50 & 81.25 \\
31 & 54.58 & 60.42 & 83.33 & 82.08 & 91.25 & 95.83 \\
32 & 17.08 & 17.92 & 22.08 & 22.50 & 29.58 & 33.33 \\
33 & 37.92 & 40.83 & 58.33 & 61.67 & 70.83 & 77.50 \\
34 & 14.58 & 15.83 & 32.92 & 31.25 & 47.92 & 49.58 \\
\bottomrule
\end{longtable}}

{\scriptsize\setlength{\tabcolsep}{3.2pt}
\begin{longtable}{@{}lrrrrrr@{}}
\caption{SSVEP subject accuracy (\%) for SSVEPFormer after RW training and IW evaluation.}\label{tab:b-ss-former}\\
\toprule
Subject & 0.6 C & 0.6 L & 0.8 C & 0.8 L & 1.0 C & 1.0 L \\
\midrule
\endfirsthead
\toprule
Subject & 0.6 C & 0.6 L & 0.8 C & 0.8 L & 1.0 C & 1.0 L \\
\midrule
\endhead
0 & 35.42 & 42.50 & 44.58 & 53.75 & 62.50 & 66.67 \\
1 & 51.25 & 54.17 & 74.17 & 75.00 & 80.83 & 83.33 \\
2 & 62.92 & 69.17 & 80.83 & 84.17 & 91.67 & 94.17 \\
3 & 62.50 & 52.50 & 75.00 & 78.75 & 88.75 & 87.50 \\
4 & 37.50 & 35.00 & 60.42 & 50.83 & 75.83 & 72.92 \\
5 & 31.25 & 34.58 & 57.50 & 60.83 & 67.08 & 75.42 \\
6 & 30.83 & 39.17 & 61.67 & 49.58 & 69.17 & 72.50 \\
7 & 17.08 & 17.08 & 35.00 & 30.83 & 43.33 & 47.92 \\
8 & 32.08 & 34.17 & 46.25 & 45.42 & 67.08 & 65.42 \\
9 & 32.08 & 32.50 & 46.67 & 54.58 & 63.75 & 72.92 \\
10 & 9.17 & 8.75 & 10.00 & 9.58 & 15.42 & 17.08 \\
11 & 24.17 & 27.50 & 42.92 & 52.08 & 63.33 & 74.17 \\
12 & 15.00 & 27.92 & 27.92 & 42.50 & 34.17 & 54.58 \\
13 & 25.83 & 29.17 & 42.08 & 39.58 & 52.08 & 46.67 \\
14 & 21.25 & 26.25 & 44.17 & 45.83 & 54.58 & 56.25 \\
15 & 17.08 & 21.67 & 36.67 & 31.67 & 48.75 & 55.00 \\
16 & 37.92 & 33.75 & 54.17 & 56.25 & 65.83 & 65.42 \\
17 & 47.50 & 48.33 & 59.58 & 64.17 & 71.67 & 73.33 \\
18 & 7.92 & 8.33 & 12.08 & 17.08 & 20.42 & 25.42 \\
19 & 31.25 & 30.83 & 47.08 & 54.17 & 67.08 & 66.25 \\
20 & 17.92 & 20.83 & 44.58 & 48.33 & 65.42 & 70.42 \\
21 & 47.92 & 55.42 & 77.50 & 81.67 & 87.50 & 91.67 \\
22 & 35.42 & 44.58 & 65.00 & 73.33 & 82.92 & 83.75 \\
23 & 42.50 & 37.08 & 69.58 & 66.25 & 83.33 & 81.67 \\
24 & 40.00 & 42.92 & 57.50 & 58.75 & 77.50 & 71.25 \\
25 & 45.83 & 46.25 & 71.25 & 67.50 & 78.75 & 73.75 \\
26 & 50.00 & 52.08 & 72.92 & 76.25 & 87.08 & 90.00 \\
27 & 45.83 & 50.42 & 77.08 & 80.83 & 88.75 & 88.33 \\
28 & 11.67 & 15.83 & 17.50 & 29.58 & 27.50 & 39.58 \\
29 & 30.83 & 30.00 & 60.00 & 54.58 & 67.50 & 66.25 \\
30 & 41.67 & 54.58 & 67.50 & 69.17 & 78.33 & 81.25 \\
31 & 61.67 & 63.33 & 82.50 & 83.75 & 94.58 & 92.92 \\
32 & 15.00 & 16.25 & 21.67 & 27.50 & 30.83 & 32.08 \\
33 & 38.33 & 40.42 & 57.08 & 57.92 & 76.67 & 75.00 \\
34 & 16.67 & 14.58 & 30.00 & 35.83 & 52.92 & 51.67 \\
\bottomrule
\end{longtable}}

The Speech table gives the 12 subject-level accuracies for the four neural encoders used in Table~1. Du-IN and the other supervised baselines follow the Speech benchmark implementation~\citep{zheng2024duin}; the LaS values use the corresponding encoder and text-conditioned classifier.

{\scriptsize\setlength{\tabcolsep}{2pt}
\begin{longtable}{@{}lrrrrrrrr@{}}
\caption{Speech subject accuracy (\%) across four neural encoders.}\label{tab:b-speech-subj}\\
\toprule
Subject & EEGNet C & EEGNet L & BiGRU C & BiGRU L & Conformer C & Conformer L & Du-IN C & Du-IN L \\
\midrule
\endfirsthead
\toprule
Subject & EEGNet C & EEGNet L & BiGRU C & BiGRU L & Conformer C & Conformer L & Du-IN C & Du-IN L \\
\midrule
\endhead
S001 & 9.69 & 24.06 & 51.56 & 70.94 & 60.31 & 73.75 & 73.75 & 77.81 \\
S002 & 8.49 & 18.90 & 32.05 & 41.92 & 44.11 & 46.03 & 46.58 & 47.95 \\
S003 & 6.69 & 4.46 & 25.63 & 38.72 & 32.03 & 35.93 & 38.44 & 45.68 \\
S004 & 8.20 & 19.02 & 42.62 & 66.56 & 59.02 & 69.84 & 68.85 & 76.72 \\
S005 & 11.48 & 22.62 & 50.82 & 69.18 & 61.31 & 73.11 & 73.11 & 77.38 \\
S006 & 3.72 & 8.31 & 29.80 & 46.70 & 39.26 & 40.40 & 47.56 & 50.72 \\
S007 & 4.96 & 11.57 & 30.03 & 42.15 & 37.47 & 45.18 & 46.01 & 50.69 \\
S008 & 6.28 & 12.02 & 36.61 & 49.18 & 46.72 & 55.19 & 52.46 & 59.84 \\
S009 & 6.96 & 14.48 & 39.00 & 53.48 & 49.03 & 55.15 & 51.25 & 62.12 \\
S010 & 2.76 & 9.67 & 23.48 & 33.15 & 27.90 & 28.45 & 38.40 & 35.91 \\
S011 & 6.99 & 23.40 & 43.16 & 69.60 & 62.31 & 67.17 & 71.12 & 73.86 \\
S012 & 9.29 & 16.94 & 31.15 & 60.11 & 48.63 & 52.46 & 55.74 & 67.21 \\
\bottomrule
\end{longtable}}

\paragraph{MI candidate-set reuse in Figure 3.}
Figure~3 compares a supervised classifier trained for each candidate set with a LaS model whose four-class checkpoint is reused with the corresponding candidates. Classes 1--4 denote thumb, index, middle, and pinky, respectively. Each table retains the 21 paired subject accuracies; C and L again denote supervised CLS and LaS.

{\scriptsize\setlength{\tabcolsep}{2pt}
\begin{longtable}{@{}lrrrrrrrr@{}}
\caption{MI candidate-set reuse: EEGNet subject accuracy (\%).}\label{tab:b-subset-eegnet}\\
\toprule
Subject & 1--4 C & 1--4 L & 2--3 C & 2--3 L & 1--3--4 C & 1--3--4 L & 1--2--3 C & 1--2--3 L \\
\midrule
\endfirsthead
\toprule
Subject & 1--4 C & 1--4 L & 2--3 C & 2--3 L & 1--3--4 C & 1--3--4 L & 1--2--3 C & 1--2--3 L \\
\midrule
\endhead
S01 & 76.00 & 76.00 & 61.67 & 63.33 & 54.22 & 52.89 & 43.11 & 40.22 \\
S02 & 74.00 & 76.33 & 66.00 & 65.33 & 44.44 & 46.67 & 45.33 & 42.67 \\
S03 & 70.00 & 65.00 & 53.00 & 50.67 & 50.22 & 42.00 & 40.89 & 36.89 \\
S04 & 73.44 & 69.38 & 52.19 & 53.44 & 47.92 & 47.08 & 45.83 & 42.71 \\
S05 & 79.06 & 74.69 & 55.63 & 51.88 & 53.13 & 51.46 & 41.25 & 41.67 \\
S06 & 69.33 & 67.33 & 59.33 & 56.67 & 47.11 & 47.56 & 42.67 & 44.00 \\
S07 & 82.67 & 83.00 & 70.33 & 71.00 & 61.11 & 55.56 & 50.22 & 48.67 \\
S08 & 79.33 & 75.67 & 54.67 & 59.00 & 51.78 & 54.89 & 47.33 & 47.11 \\
S09 & 76.00 & 77.33 & 58.33 & 59.33 & 53.33 & 57.78 & 45.78 & 52.67 \\
S10 & 70.00 & 66.00 & 55.33 & 59.67 & 47.33 & 44.67 & 42.22 & 40.44 \\
S11 & 62.67 & 65.67 & 49.00 & 51.00 & 45.33 & 49.78 & 41.33 & 42.44 \\
S12 & 67.67 & 68.67 & 58.67 & 56.67 & 46.67 & 48.44 & 41.33 & 43.56 \\
S13 & 75.48 & 72.58 & 57.42 & 61.29 & 49.46 & 49.89 & 42.15 & 42.58 \\
S14 & 60.63 & 61.88 & 57.19 & 57.81 & 47.92 & 41.46 & 41.46 & 40.21 \\
S15 & 61.56 & 61.25 & 59.38 & 60.63 & 48.13 & 46.67 & 48.54 & 47.92 \\
S16 & 66.88 & 66.25 & 51.88 & 55.31 & 42.71 & 46.88 & 37.71 & 38.75 \\
S17 & 78.44 & 75.31 & 61.56 & 65.31 & 57.92 & 57.71 & 50.21 & 46.88 \\
S18 & 66.25 & 68.75 & 49.69 & 54.69 & 46.04 & 49.17 & 42.08 & 38.54 \\
S19 & 96.25 & 96.88 & 79.06 & 82.81 & 70.63 & 81.46 & 59.38 & 59.17 \\
S20 & 68.13 & 65.31 & 51.56 & 48.75 & 42.92 & 43.96 & 37.08 & 35.63 \\
S21 & 60.63 & 54.06 & 50.63 & 51.25 & 39.79 & 37.50 & 35.42 & 33.33 \\
\bottomrule
\end{longtable}}

{\scriptsize\setlength{\tabcolsep}{2pt}
\begin{longtable}{@{}lrrrrrrrr@{}}
\caption{MI candidate-set reuse: DeepConvNet1s subject accuracy (\%).}\label{tab:b-subset-deep}\\
\toprule
Subject & 1--4 C & 1--4 L & 2--3 C & 2--3 L & 1--3--4 C & 1--3--4 L & 1--2--3 C & 1--2--3 L \\
\midrule
\endfirsthead
\toprule
Subject & 1--4 C & 1--4 L & 2--3 C & 2--3 L & 1--3--4 C & 1--3--4 L & 1--2--3 C & 1--2--3 L \\
\midrule
\endhead
S01 & 63.33 & 76.33 & 55.67 & 59.67 & 50.67 & 55.11 & 39.56 & 47.78 \\
S02 & 67.33 & 76.00 & 66.00 & 67.00 & 48.22 & 51.11 & 40.89 & 50.67 \\
S03 & 67.00 & 75.00 & 45.67 & 54.33 & 55.33 & 58.67 & 39.11 & 47.56 \\
S04 & 69.69 & 75.63 & 48.13 & 52.50 & 48.54 & 51.46 & 35.00 & 40.83 \\
S05 & 71.56 & 78.75 & 52.50 & 52.81 & 50.83 & 53.96 & 39.58 & 42.71 \\
S06 & 55.33 & 64.00 & 52.67 & 54.67 & 44.44 & 44.89 & 37.33 & 43.56 \\
S07 & 78.67 & 85.67 & 72.33 & 74.67 & 59.78 & 66.22 & 42.67 & 58.44 \\
S08 & 63.00 & 72.67 & 56.33 & 52.00 & 49.56 & 49.78 & 41.33 & 48.22 \\
S09 & 65.67 & 77.33 & 51.67 & 63.67 & 57.56 & 65.33 & 45.11 & 59.56 \\
S10 & 71.67 & 77.00 & 54.00 & 58.00 & 48.44 & 50.22 & 36.67 & 43.78 \\
S11 & 58.33 & 61.33 & 52.33 & 52.00 & 45.11 & 46.44 & 41.33 & 40.67 \\
S12 & 55.67 & 61.67 & 56.67 & 54.67 & 45.78 & 41.33 & 40.00 & 43.11 \\
S13 & 68.71 & 75.16 & 49.35 & 58.39 & 49.89 & 51.40 & 35.70 & 48.39 \\
S14 & 53.44 & 58.13 & 49.38 & 53.75 & 40.21 & 41.67 & 35.42 & 38.54 \\
S15 & 61.88 & 65.00 & 59.69 & 64.69 & 50.42 & 51.04 & 41.88 & 55.83 \\
S16 & 51.56 & 55.63 & 47.81 & 50.94 & 43.13 & 37.92 & 38.75 & 39.17 \\
S17 & 70.63 & 78.44 & 56.25 & 64.69 & 53.96 & 61.46 & 38.75 & 47.92 \\
S18 & 65.00 & 69.69 & 50.31 & 52.81 & 46.25 & 48.54 & 32.50 & 43.13 \\
S19 & 80.00 & 96.88 & 59.38 & 80.63 & 69.38 & 79.17 & 49.38 & 66.04 \\
S20 & 60.63 & 69.69 & 49.69 & 50.94 & 42.71 & 47.92 & 34.38 & 36.46 \\
S21 & 55.00 & 59.38 & 46.88 & 52.81 & 39.38 & 41.25 & 30.21 & 38.33 \\
\bottomrule
\end{longtable}}

{\scriptsize\setlength{\tabcolsep}{2pt}
\begin{longtable}{@{}lrrrrrrrr@{}}
\caption{MI candidate-set reuse: EEG-Conformer subject accuracy (\%).}\label{tab:b-subset-conf}\\
\toprule
Subject & 1--4 C & 1--4 L & 2--3 C & 2--3 L & 1--3--4 C & 1--3--4 L & 1--2--3 C & 1--2--3 L \\
\midrule
\endfirsthead
\toprule
Subject & 1--4 C & 1--4 L & 2--3 C & 2--3 L & 1--3--4 C & 1--3--4 L & 1--2--3 C & 1--2--3 L \\
\midrule
\endhead
S01 & 77.00 & 76.67 & 60.67 & 65.33 & 56.44 & 52.00 & 46.00 & 46.89 \\
S02 & 73.33 & 76.67 & 64.33 & 67.33 & 51.56 & 53.78 & 45.33 & 53.11 \\
S03 & 77.67 & 77.67 & 56.67 & 59.67 & 58.67 & 60.00 & 45.33 & 46.00 \\
S04 & 69.38 & 72.19 & 54.69 & 53.44 & 46.88 & 48.75 & 39.17 & 39.79 \\
S05 & 78.75 & 82.81 & 54.38 & 56.56 & 55.83 & 58.54 & 44.79 & 43.13 \\
S06 & 70.67 & 71.33 & 60.67 & 61.33 & 44.44 & 49.33 & 43.56 & 43.11 \\
S07 & 87.00 & 89.00 & 69.33 & 73.67 & 64.44 & 66.67 & 54.67 & 58.67 \\
S08 & 81.00 & 78.00 & 55.00 & 57.00 & 49.33 & 54.22 & 49.11 & 49.33 \\
S09 & 78.67 & 83.00 & 60.67 & 61.33 & 64.22 & 67.33 & 53.33 & 55.78 \\
S10 & 67.33 & 68.33 & 53.67 & 56.67 & 45.56 & 45.33 & 40.67 & 41.56 \\
S11 & 67.33 & 71.67 & 55.33 & 50.00 & 48.44 & 55.78 & 44.00 & 43.33 \\
S12 & 67.67 & 70.67 & 50.00 & 57.67 & 49.78 & 49.78 & 47.56 & 47.56 \\
S13 & 77.42 & 77.10 & 56.77 & 62.90 & 55.05 & 56.13 & 46.45 & 48.39 \\
S14 & 63.75 & 65.63 & 51.88 & 54.38 & 39.17 & 40.00 & 36.46 & 38.75 \\
S15 & 72.19 & 70.00 & 62.19 & 63.44 & 55.42 & 52.29 & 56.88 & 51.67 \\
S16 & 69.69 & 71.56 & 49.69 & 53.75 & 47.92 & 50.21 & 40.83 & 41.88 \\
S17 & 80.31 & 80.63 & 55.63 & 65.00 & 59.38 & 63.75 & 47.29 & 51.67 \\
S18 & 73.13 & 73.75 & 55.94 & 56.88 & 51.25 & 53.96 & 44.17 & 45.42 \\
S19 & 98.75 & 99.38 & 76.56 & 91.56 & 85.42 & 89.38 & 63.54 & 70.21 \\
S20 & 68.75 & 69.38 & 53.13 & 47.50 & 45.83 & 47.08 & 35.63 & 34.58 \\
S21 & 62.19 & 63.13 & 48.75 & 51.25 & 42.29 & 44.38 & 37.29 & 37.92 \\
\bottomrule
\end{longtable}}

\paragraph{MI cross-session adaptation in Figure 4b.}
Table~\ref{tab:b-mi-adapt} reports subject-mean trial accuracy and between-subject SD for nine methods in the three-class and binary tasks, including LP++ and ProLIP~\citep{huang2024lpplusplus,fahes2026prolip}. Direct inference is evaluated on the Base block and adapted models on the Fine block. In Figures~4a and~4b, ``MLP probe'' uses encoder features, and ``Full fine-tuning'' denotes the Projection variant described in Appendix~\ref{app:adaptation}.

{\scriptsize\setlength{\tabcolsep}{4pt}
\begin{longtable}{@{}lrr@{}}
\caption{MI cross-session adaptation: subject-mean trial accuracy $\pm$ between-subject SD (\%).}\label{tab:b-mi-adapt}\\
\toprule
Method & 3-class & 2-class \\
\midrule
\endfirsthead
\toprule
Method & 3-class & 2-class \\
\midrule
\endhead
Direct inference & 46.43 \(\pm\) 9.14 & 64.24 \(\pm\) 9.56 \\
Linear Probe & 48.69 \(\pm\) 10.08 & 66.63 \(\pm\) 9.80 \\
LP++ & 48.67 \(\pm\) 9.81 & 65.76 \(\pm\) 11.08 \\
ProLIP & 49.13 \(\pm\) 9.67 & 65.55 \(\pm\) 11.48 \\
DTS-ProLIP & 49.94 \(\pm\) 9.36 & 65.91 \(\pm\) 11.93 \\
MLP Probe Encoder & 50.23 \(\pm\) 10.39 & 67.01 \(\pm\) 11.06 \\
MLP Probe Projection & 49.44 \(\pm\) 10.72 & 66.33 \(\pm\) 10.56 \\
Full Encoder & 55.34 \(\pm\) 12.87 & 75.14 \(\pm\) 10.83 \\
Full Projection & 55.48 \(\pm\) 13.08 & 76.35 \(\pm\) 10.84 \\
\bottomrule
\end{longtable}}

\paragraph{SSVEP subject adaptation in Figure 4a.}
Table~\ref{tab:b-ss-adapt} reports leave-one-block validation accuracy for eight neural adaptation methods and three TRCA configurations. K is the number of calibration blocks. TRCA methods require at least two blocks under this protocol, so their K=1 cells are unavailable~\citep{nakanishi2018enhancing}. The Figure~4a curve labelled Ensemble TRCA uses the single-band ensemble configuration; Table~\ref{tab:b-ss-adapt} also reports the other Ensemble TRCA configuration.

{\scriptsize\setlength{\tabcolsep}{2.5pt}
\begin{longtable}{@{}llrrrrr@{}}
\caption{SSVEP K-block adaptation: leave-one-block validation accuracy (\%).}\label{tab:b-ss-adapt}\\
\toprule
Family & Method & K=1 & K=2 & K=3 & K=4 & K=5 \\
\midrule
\endfirsthead
\toprule
Family & Method & K=1 & K=2 & K=3 & K=4 & K=5 \\
\midrule
\endhead
LaS & DTS-ProLIP & 80.38 & 81.85 & 82.81 & 83.45 & 84.17 \\
LaS & ProLIP & 79.46 & 80.14 & 80.70 & 80.94 & 81.25 \\
LaS & Linear Probe & 71.29 & 79.01 & 81.37 & 82.11 & 82.56 \\
LaS & LP++ & 75.44 & 77.58 & 79.01 & 79.71 & 80.26 \\
LaS & MLP Probe Encoder & 72.89 & 82.77 & 85.93 & 87.55 & 88.93 \\
LaS & Full Finetune Encoder & 62.92 & 77.26 & 82.88 & 85.57 & 88.07 \\
LaS & MLP Probe Projection & 71.05 & 79.92 & 83.07 & 83.89 & 84.71 \\
LaS & Full Finetune Projection & 73.18 & 85.15 & 89.24 & 90.74 & 91.93 \\
TRCA & Standard TRCA & -- & 66.35 & 80.35 & 86.07 & 88.95 \\
TRCA & Ensemble TRCA & -- & 81.92 & 88.50 & 90.90 & 92.27 \\
TRCA & Single-band Ensemble TRCA & -- & 82.33 & 88.13 & 90.45 & 91.79 \\
\bottomrule
\end{longtable}}

\subsection{Conditional Unseen-Class Decoding}
\label{app:b-zsl}

Figure~4c averages five-class conditional decoding across eight held-out class splits. Tables~\ref{tab:b-zsl-06}--\ref{tab:b-zsl-10} retain the accuracy of each of the 40 stimulus classes when it belongs to its unseen set. Class index $i$ is zero based and sorted by stimulus frequency, $f_i=8+0.2i$ Hz. Each class contributes 210 Physical trials to its evaluation. Neural results use the fixed initial window at offset 0.0 s and the five unseen candidates; FBCCA uses its corresponding signal-processing reference~\citep{wang2017benchmark,chen2015filter}.

{\scriptsize\setlength{\tabcolsep}{3pt}
\begin{longtable}{@{}rrrrrr@{}}
\caption{Five-class unseen-only accuracy (\%) by stimulus frequency with 0.6 s inputs.}\label{tab:b-zsl-06}\\
\toprule
Class & Hz & MTSNet & SSVEPFormer & EEGNet & FBCCA \\
\midrule
\endfirsthead
\toprule
Class & Hz & MTSNet & SSVEPFormer & EEGNet & FBCCA \\
\midrule
\endhead
0 & 8.0 & 72.38 & 82.86 & 78.57 & 85.71 \\
1 & 8.2 & 81.90 & 83.33 & 83.81 & 86.19 \\
2 & 8.4 & 84.76 & 79.05 & 79.05 & 86.19 \\
3 & 8.6 & 85.24 & 85.24 & 75.24 & 85.24 \\
4 & 8.8 & 82.86 & 78.57 & 81.90 & 82.38 \\
5 & 9.0 & 72.86 & 72.38 & 70.48 & 85.24 \\
6 & 9.2 & 84.76 & 86.67 & 85.71 & 86.19 \\
7 & 9.4 & 89.52 & 90.00 & 89.52 & 84.76 \\
8 & 9.6 & 84.76 & 76.19 & 85.71 & 73.81 \\
9 & 9.8 & 83.81 & 81.43 & 78.57 & 69.05 \\
10 & 10.0 & 77.62 & 86.19 & 79.52 & 68.10 \\
11 & 10.2 & 79.52 & 77.62 & 82.86 & 75.24 \\
12 & 10.4 & 90.00 & 90.00 & 82.86 & 83.33 \\
13 & 10.6 & 86.67 & 83.33 & 75.24 & 88.10 \\
14 & 10.8 & 82.38 & 78.57 & 79.05 & 87.14 \\
15 & 11.0 & 79.52 & 76.67 & 77.14 & 84.76 \\
16 & 11.2 & 86.19 & 85.71 & 78.57 & 81.43 \\
17 & 11.4 & 79.05 & 80.00 & 83.33 & 83.33 \\
18 & 11.6 & 87.14 & 80.00 & 83.33 & 85.24 \\
19 & 11.8 & 70.48 & 75.71 & 68.57 & 80.00 \\
20 & 12.0 & 69.52 & 67.14 & 72.86 & 71.90 \\
21 & 12.2 & 85.71 & 86.67 & 86.19 & 65.71 \\
22 & 12.4 & 83.81 & 80.48 & 85.71 & 60.48 \\
23 & 12.6 & 81.90 & 79.05 & 78.10 & 57.62 \\
24 & 12.8 & 83.33 & 79.05 & 71.90 & 64.29 \\
25 & 13.0 & 70.00 & 67.14 & 65.24 & 50.00 \\
26 & 13.2 & 80.48 & 83.33 & 79.05 & 59.05 \\
27 & 13.4 & 86.67 & 81.90 & 88.10 & 64.76 \\
28 & 13.6 & 74.29 & 75.24 & 73.81 & 57.62 \\
29 & 13.8 & 80.00 & 75.71 & 71.90 & 52.86 \\
30 & 14.0 & 65.71 & 66.19 & 65.24 & 48.10 \\
31 & 14.2 & 80.48 & 77.62 & 77.62 & 64.76 \\
32 & 14.4 & 82.86 & 78.10 & 78.57 & 50.48 \\
33 & 14.6 & 77.14 & 80.00 & 76.67 & 46.67 \\
34 & 14.8 & 88.10 & 87.14 & 84.29 & 61.90 \\
35 & 15.0 & 76.19 & 79.05 & 77.62 & 49.52 \\
36 & 15.2 & 80.95 & 85.71 & 77.14 & 54.29 \\
37 & 15.4 & 81.43 & 87.14 & 83.81 & 54.29 \\
38 & 15.6 & 80.00 & 77.14 & 81.43 & 60.48 \\
39 & 15.8 & 77.14 & 79.05 & 74.76 & 62.38 \\
\bottomrule
\end{longtable}}

{\scriptsize\setlength{\tabcolsep}{3pt}
\begin{longtable}{@{}rrrrrr@{}}
\caption{Five-class unseen-only accuracy (\%) by stimulus frequency with 0.8 s inputs.}\label{tab:b-zsl-08}\\
\toprule
Class & Hz & MTSNet & SSVEPFormer & EEGNet & FBCCA \\
\midrule
\endfirsthead
\toprule
Class & Hz & MTSNet & SSVEPFormer & EEGNet & FBCCA \\
\midrule
\endhead
0 & 8.0 & 89.52 & 92.86 & 84.29 & 93.81 \\
1 & 8.2 & 90.48 & 90.95 & 87.62 & 94.76 \\
2 & 8.4 & 92.86 & 90.95 & 88.10 & 90.48 \\
3 & 8.6 & 93.81 & 88.57 & 86.19 & 91.43 \\
4 & 8.8 & 90.95 & 82.86 & 89.05 & 90.95 \\
5 & 9.0 & 84.76 & 86.19 & 85.24 & 90.95 \\
6 & 9.2 & 93.81 & 92.38 & 92.38 & 88.10 \\
7 & 9.4 & 92.86 & 95.71 & 94.76 & 93.33 \\
8 & 9.6 & 93.33 & 89.05 & 92.38 & 82.86 \\
9 & 9.8 & 90.48 & 87.14 & 87.14 & 82.86 \\
10 & 10.0 & 80.00 & 87.62 & 89.05 & 80.95 \\
11 & 10.2 & 91.90 & 90.95 & 83.81 & 83.81 \\
12 & 10.4 & 94.76 & 97.14 & 90.00 & 90.48 \\
13 & 10.6 & 89.05 & 90.48 & 87.62 & 93.33 \\
14 & 10.8 & 87.62 & 88.57 & 86.19 & 91.90 \\
15 & 11.0 & 85.71 & 84.29 & 83.81 & 90.48 \\
16 & 11.2 & 87.62 & 90.95 & 90.48 & 91.43 \\
17 & 11.4 & 88.10 & 88.10 & 90.00 & 88.10 \\
18 & 11.6 & 90.48 & 90.48 & 86.19 & 93.33 \\
19 & 11.8 & 80.00 & 75.24 & 78.57 & 86.19 \\
20 & 12.0 & 78.57 & 76.67 & 79.52 & 86.19 \\
21 & 12.2 & 89.52 & 92.86 & 88.57 & 79.52 \\
22 & 12.4 & 91.43 & 90.00 & 90.95 & 78.10 \\
23 & 12.6 & 91.43 & 90.95 & 86.19 & 77.62 \\
24 & 12.8 & 88.10 & 87.62 & 82.86 & 77.14 \\
25 & 13.0 & 81.90 & 80.00 & 70.00 & 69.52 \\
26 & 13.2 & 81.90 & 86.67 & 84.76 & 74.76 \\
27 & 13.4 & 91.90 & 92.38 & 93.81 & 76.67 \\
28 & 13.6 & 87.14 & 84.76 & 84.29 & 72.86 \\
29 & 13.8 & 91.43 & 84.76 & 80.48 & 78.10 \\
30 & 14.0 & 73.33 & 70.00 & 75.24 & 67.14 \\
31 & 14.2 & 88.10 & 85.71 & 89.05 & 83.33 \\
32 & 14.4 & 89.05 & 89.05 & 88.10 & 67.14 \\
33 & 14.6 & 92.38 & 89.52 & 85.71 & 65.24 \\
34 & 14.8 & 91.43 & 90.95 & 90.95 & 80.95 \\
35 & 15.0 & 77.14 & 86.19 & 82.38 & 69.05 \\
36 & 15.2 & 84.29 & 85.24 & 85.71 & 73.33 \\
37 & 15.4 & 88.10 & 90.48 & 91.43 & 68.10 \\
38 & 15.6 & 88.10 & 86.67 & 87.62 & 77.14 \\
39 & 15.8 & 85.71 & 86.19 & 83.81 & 75.71 \\
\bottomrule
\end{longtable}}

{\scriptsize\setlength{\tabcolsep}{3pt}
\begin{longtable}{@{}rrrrrr@{}}
\caption{Five-class unseen-only accuracy (\%) by stimulus frequency with 1.0 s inputs.}\label{tab:b-zsl-10}\\
\toprule
Class & Hz & MTSNet & SSVEPFormer & EEGNet & FBCCA \\
\midrule
\endfirsthead
\toprule
Class & Hz & MTSNet & SSVEPFormer & EEGNet & FBCCA \\
\midrule
\endhead
0 & 8.0 & 90.95 & 92.86 & 92.38 & 96.67 \\
1 & 8.2 & 96.67 & 96.67 & 93.81 & 98.57 \\
2 & 8.4 & 91.43 & 91.43 & 91.43 & 95.24 \\
3 & 8.6 & 95.24 & 93.81 & 92.86 & 94.29 \\
4 & 8.8 & 92.38 & 91.90 & 92.38 & 94.76 \\
5 & 9.0 & 83.33 & 92.38 & 89.52 & 92.86 \\
6 & 9.2 & 96.19 & 96.67 & 97.14 & 91.90 \\
7 & 9.4 & 96.67 & 97.14 & 97.14 & 94.76 \\
8 & 9.6 & 94.76 & 94.76 & 95.71 & 91.43 \\
9 & 9.8 & 91.43 & 90.00 & 90.48 & 87.14 \\
10 & 10.0 & 88.10 & 94.29 & 93.81 & 89.52 \\
11 & 10.2 & 92.86 & 91.43 & 93.33 & 92.38 \\
12 & 10.4 & 98.10 & 96.67 & 96.19 & 97.62 \\
13 & 10.6 & 94.29 & 92.86 & 92.38 & 95.24 \\
14 & 10.8 & 92.86 & 91.90 & 89.05 & 95.71 \\
15 & 11.0 & 82.86 & 86.19 & 85.71 & 93.81 \\
16 & 11.2 & 93.33 & 96.19 & 90.00 & 93.81 \\
17 & 11.4 & 92.86 & 92.86 & 93.33 & 90.95 \\
18 & 11.6 & 97.14 & 90.00 & 88.10 & 96.67 \\
19 & 11.8 & 88.57 & 72.86 & 87.62 & 93.33 \\
20 & 12.0 & 73.81 & 69.05 & 78.57 & 93.33 \\
21 & 12.2 & 91.43 & 90.95 & 95.71 & 84.29 \\
22 & 12.4 & 95.71 & 94.76 & 93.81 & 84.76 \\
23 & 12.6 & 96.67 & 92.86 & 89.05 & 84.76 \\
24 & 12.8 & 88.10 & 89.52 & 90.00 & 85.24 \\
25 & 13.0 & 80.48 & 81.43 & 79.05 & 80.95 \\
26 & 13.2 & 86.67 & 89.52 & 89.05 & 83.81 \\
27 & 13.4 & 94.29 & 93.81 & 95.71 & 85.71 \\
28 & 13.6 & 89.52 & 91.43 & 89.05 & 84.29 \\
29 & 13.8 & 96.19 & 92.86 & 86.19 & 87.14 \\
30 & 14.0 & 70.95 & 70.95 & 81.43 & 79.52 \\
31 & 14.2 & 91.90 & 89.52 & 91.90 & 88.10 \\
32 & 14.4 & 94.29 & 90.48 & 90.95 & 77.14 \\
33 & 14.6 & 92.38 & 90.95 & 91.90 & 79.05 \\
34 & 14.8 & 91.43 & 92.86 & 93.81 & 89.05 \\
35 & 15.0 & 88.10 & 89.05 & 85.71 & 79.52 \\
36 & 15.2 & 86.19 & 85.71 & 90.00 & 79.52 \\
37 & 15.4 & 94.29 & 93.81 & 94.29 & 80.48 \\
38 & 15.6 & 92.86 & 89.05 & 93.81 & 85.24 \\
39 & 15.8 & 90.48 & 87.14 & 91.43 & 84.76 \\
\bottomrule
\end{longtable}}

\subsection{Language-Side Analyses}
\label{app:b-language}

\paragraph{Text encoders and adaptation state.}
Figure~\ref{fig:b-text-pairs} compares paired Frozen/LoRA runs for BERT-base-uncased and EmbeddingGemma-300M~\citep{devlin2019bert,schechtervera2025embeddinggemma,hu2022lora}. A line joins matching MI folds, Speech seeds, SSVEP subject folds, or unseen-class folds. Figure~5d compares independently selected best configurations. In the paired runs, BERT Frozen SSVEP uses random training templates, while the other text-encoder conditions use fixed T0.
\begin{figure}[!htbp]
\centering
\includegraphics[width=\linewidth]{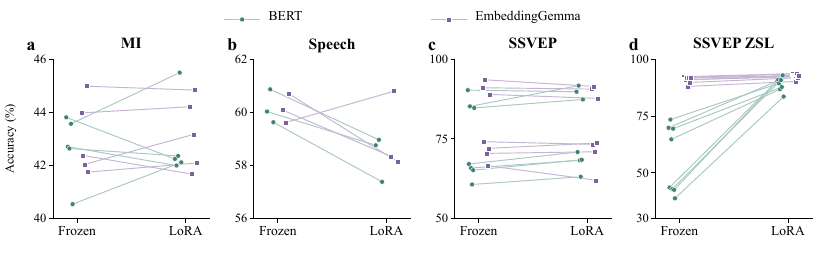}
\caption{Paired Frozen and LoRA comparisons for BERT-base-uncased and EmbeddingGemma-300M. Lines connect the matched run units specified in the text.}
\label{fig:b-text-pairs}
\end{figure}

\paragraph{Held-out short templates.}
Table~\ref{tab:b-templates} gives the five evaluation templates and their five-template ensemble under P0 and P2. P0 uses fixed template 0 during training; P2 randomly samples from five semantic training templates. For the ensemble, candidate-class logits are averaged across the five held-out templates before prediction. MI and Speech SDs are across subjects, the SSVEP closed-set SD is across seven subject folds, and the conditional five-class ZSL SD is across eight class folds. Speech P0 is seed 3407; for Speech P2, three seeds are first averaged within subject.

{\scriptsize\setlength{\tabcolsep}{5pt}
\begin{longtable}{@{}llrr@{}}
\caption{Accuracy (\%, mean $\pm$ SD) under held-out short templates.}\label{tab:b-templates}\\
\toprule
Task & Template & P0 & P2 \\
\midrule
\endfirsthead
\toprule
Task & Template & P0 & P2 \\
\midrule
\endhead
MI closed-set & S0 & 40.20 \(\pm\) 9.66 & 42.65 \(\pm\) 9.59 \\
 & S1 & 38.81 \(\pm\) 9.91 & 42.43 \(\pm\) 9.48 \\
 & S2 & 39.23 \(\pm\) 9.74 & 42.54 \(\pm\) 9.36 \\
 & S3 & 40.19 \(\pm\) 9.71 & 42.52 \(\pm\) 9.45 \\
 & S4 & 35.88 \(\pm\) 10.06 & 42.32 \(\pm\) 9.47 \\
 & Five-template ensemble & 38.94 \(\pm\) 9.96 & 42.51 \(\pm\) 9.40 \\
SSVEP closed-set & S0 & 79.27 \(\pm\) 11.23 & 74.54 \(\pm\) 12.48 \\
 & S1 & 79.32 \(\pm\) 11.54 & 74.55 \(\pm\) 12.39 \\
 & S2 & 79.49 \(\pm\) 11.54 & 74.43 \(\pm\) 12.64 \\
 & S3 & 79.39 \(\pm\) 11.23 & 74.60 \(\pm\) 12.15 \\
 & S4 & 79.33 \(\pm\) 11.36 & 74.65 \(\pm\) 12.52 \\
 & Five-template ensemble & 79.43 \(\pm\) 11.42 & 74.63 \(\pm\) 12.29 \\
SSVEP five-class ZSL & S0 & 90.88 \(\pm\) 1.64 & 89.64 \(\pm\) 1.55 \\
 & S1 & 90.80 \(\pm\) 1.42 & 89.60 \(\pm\) 1.61 \\
 & S2 & 90.80 \(\pm\) 1.54 & 89.88 \(\pm\) 1.57 \\
 & S3 & 91.06 \(\pm\) 1.31 & 90.02 \(\pm\) 1.34 \\
 & S4 & 90.80 \(\pm\) 1.60 & 89.67 \(\pm\) 1.63 \\
 & Five-template ensemble & 90.82 \(\pm\) 1.47 & 89.92 \(\pm\) 1.67 \\
Speech closed-set & S0 & 59.96 \(\pm\) 15.77 & 60.08 \(\pm\) 15.21 \\
 & S1 & 60.29 \(\pm\) 15.84 & 60.16 \(\pm\) 15.09 \\
 & S2 & 59.62 \(\pm\) 15.77 & 59.48 \(\pm\) 15.10 \\
 & S3 & 60.44 \(\pm\) 15.91 & 59.82 \(\pm\) 15.13 \\
 & S4 & 60.19 \(\pm\) 15.85 & 59.89 \(\pm\) 15.03 \\
 & Five-template ensemble & 60.66 \(\pm\) 16.12 & 60.06 \(\pm\) 15.19 \\
\bottomrule
\end{longtable}}

\subsection{Additional Signal and Text Attribution Visualizations}
\label{app:b-attribution}

The figures in this section use grouped Integrated Gradients for the specified class's unscaled cosine-matching score~\citep{sundararajan2017axiomatic}. For each neural input, positive group attributions are divided by that input's largest positive group attribution; an input with no positive group remains zero. The normalized attributions are then averaged across the selected training trials without renormalizing the mean. Word attributions retain their signs during trial averaging, and words with positive mean attribution are shaded within each prompt. Each example uses its own relative colour scale.

For MI, each 5 s trial yields 21 overlapping 1 s windows. Each window's attributions are normalized over its 32 channels and five 200 ms time bins, then placed at its trial time; overlapping bins are averaged by their coverage count before 120 trials per class are combined. Speech uses ten channels and 30 bins of 100 ms over a 3 s trial. SSVEP uses 32 time-domain and 32 frequency-domain channel values for each 1 s input; both branches share the normalization denominator for that input. The selected SSVEP records are six original training trials per frequency, without IR synthesis or training masking.

\begin{figure}[p]
\centering
\includegraphics[width=\linewidth,height=0.68\textheight,keepaspectratio]{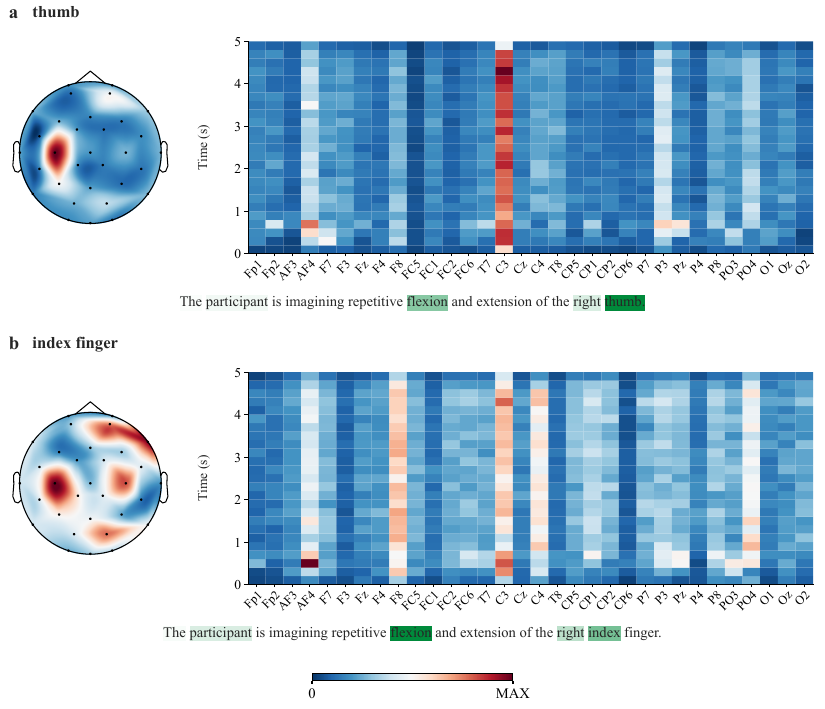}
\caption{MI training attribution for thumb and index, subject S01; 120 trials per class.}
\label{fig:b-att-mi-1}
\end{figure}

\begin{figure}[p]
\centering
\includegraphics[width=\linewidth,height=0.68\textheight,keepaspectratio]{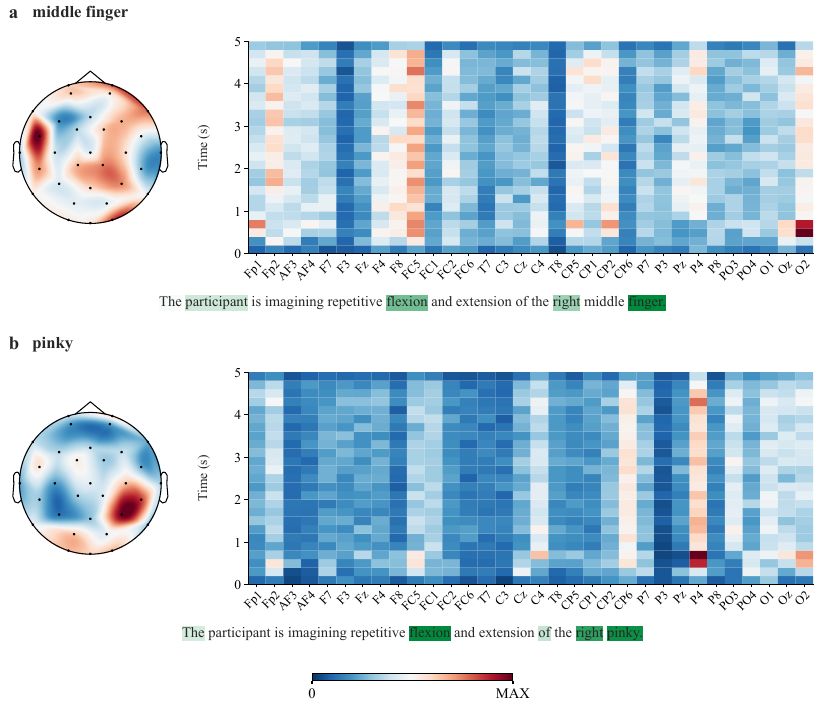}
\caption{MI training attribution for middle and pinky, subject S01; 120 trials per class.}
\label{fig:b-att-mi-2}
\end{figure}

\begin{figure}[p]
\centering
\includegraphics[width=\linewidth,height=0.68\textheight,keepaspectratio]{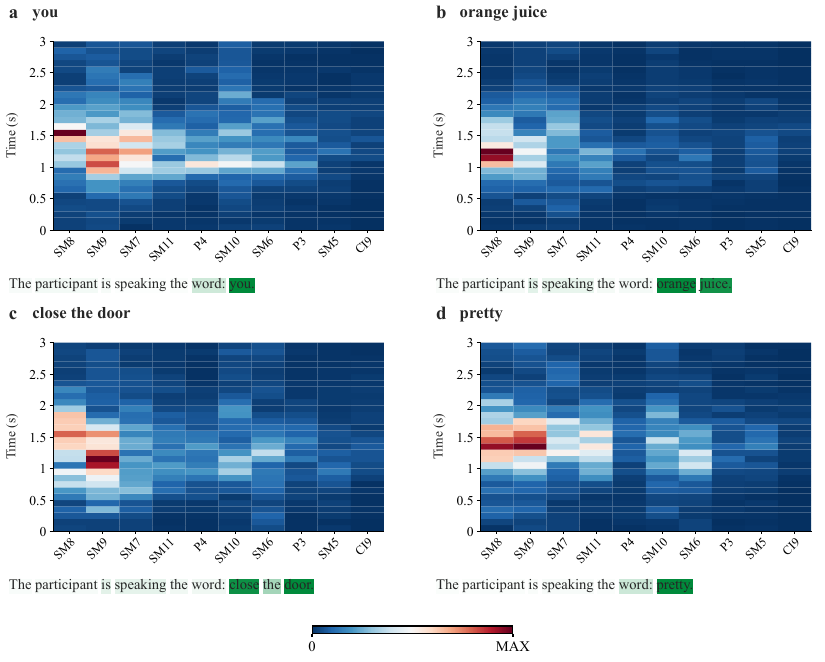}
\caption{Speech training attribution for four word classes, subject S001; class supports are 38, 40, 40, and 37.}
\label{fig:b-att-sp}
\end{figure}

\begin{figure}[p]
\centering
\includegraphics[width=\linewidth,height=0.68\textheight,keepaspectratio]{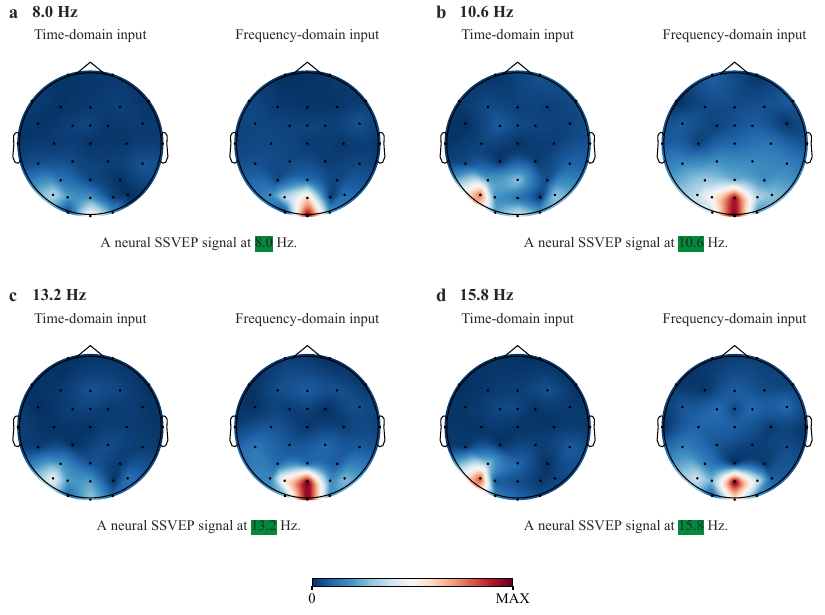}
\caption{SSVEP training attribution at 8.0, 10.6, 13.2, and 15.8 Hz, subject S01; six original trials per class.}
\label{fig:b-att-ss}
\end{figure}

\clearpage

\subsection{Training Objectives, Directional Ablations, and Temporal Aggregation}
\label{app:b-objectives}

\paragraph{Training objective and batch-size grid.}
Figure~\ref{fig:b-loss-batch} compares bidirectional InfoNCE, diagonal SigLIP, and class-aware SigLIP for MI, SSVEP, and Speech~\citep{radford2021learning,zhai2023sigmoid}. Each training-loss point is the minimum epoch-average loss for one fold or seed; accuracy uses the task-specific aggregation unit. The reference curves use the balanced-class formulas in Appendix~\ref{app:theory-losses}.
\begin{figure}[!htbp]
\centering
\includegraphics[width=\linewidth]{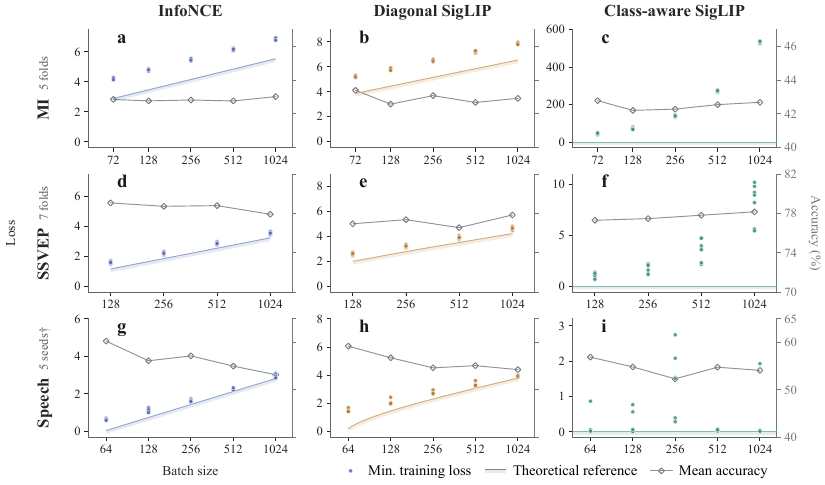}
\caption{Full training-objective by batch-size comparison across MI, SSVEP, and Speech. The left and right axes show training loss and classification accuracy, respectively; reference curves are theoretical loss baselines.}
\label{fig:b-loss-batch}
\end{figure}

\paragraph{InfoNCE direction.}
Figure~8 compares neural-to-text (N2T-only), text-to-neural (T2N-only), and bidirectional InfoNCE for MI EEG-Conformer and Speech Du-IN, including their paired differences and 95\% intervals.

\paragraph{Cumulative fusion over an MI trial.}
Figure~\ref{fig:b-mi-fusion} follows the four-class LaS EEG-Conformer model as successive 1 s windows become available at a 0.2 s stride. The three rules use the same trial predictions; five held-out folds are pooled within each subject before the 21 subjects are averaged equally. With one window at 1.0 s, logit averaging, probability averaging, and majority voting coincide at 33.78\% accuracy. The subject-mean curves generally rise as more overlapping windows are accumulated. At 5.0 s, logit averaging reaches 43.02\%, probability averaging 42.68\%, and majority voting 42.38\%. The logit-averaging endpoint is the LaS EEG-Conformer four-class value in Table~1.
\clearpage
\begin{figure}[!htbp]
\centering
\includegraphics[width=110mm]{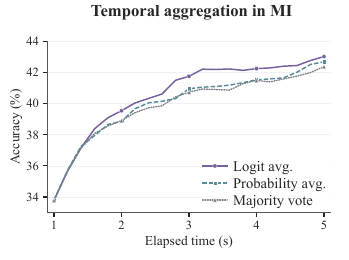}
\caption{Cumulative MI trial accuracy with logit averaging, probability averaging, and majority voting. Each time point uses all windows available up to that time.}
\label{fig:b-mi-fusion}
\end{figure}

\clearpage
\setcounter{subsection}{5}
\setcounter{table}{16}
\subsection{Paired Comparisons and Effect Sizes}
\label{app:b-statistics}

All comparisons use two-sided paired $t$-tests. In Table~1, each task--encoder combination forms a Holm comparison family: the three SSVEP window durations are corrected together, while each MI or Speech encoder has one comparison. In Figure~3, each encoder's family contains five comparisons: the four displayed subsets and the four-class result. Figure~5a uses a two-comparison family within each task (P1 versus P0 and P2 versus P0); Figure~5b uses a separate two-comparison family for unseen-class decoding. Figures~5c and~5d each use a four-comparison family across their task conditions. Figure~8 uses a two-comparison family within each task, comparing N2T-only and T2N-only with bidirectional InfoNCE.

Table~\ref{tab:b-table1-statistics} reports the Table~1 comparisons. Differences are LaS minus CLS, paired by subject. Paired Cohen's $d_z$ is the mean paired difference divided by its sample standard deviation. MI pairs 21 subjects after pooling each subject's five held-out folds; SSVEP pairs 35 subjects; Speech pairs 12 subjects after averaging the five seeds within each subject.

\begin{table}[!htbp]
\centering\small
\caption{Paired statistical comparisons for Table~1. $n$ is the number of subjects, $\Delta$ is the mean accuracy difference in percentage points, and $p_{\mathrm{Holm}}$ is corrected within the comparison families defined above.}
\label{tab:b-table1-statistics}
\begin{tabular*}{\linewidth}{@{\extracolsep{\fill}}llrrrrr@{}}
\toprule
Task & Encoder (window) & $n$ & $\Delta$ (pp) & $t(df)$ & $p_{\mathrm{Holm}}$ & $d_z$\\
\midrule
MI & EEGNet & 21 & $+0.28$ & $0.33\ (20)$ & $0.747607843$ & $+0.07$ \\
MI & DeepConvNet & 21 & $+1.82$ & $2.74\ (20)$ & $0.0125586274$ & $+0.60$ \\
MI & EEG-Conformer & 21 & $+1.36$ & $3.23\ (20)$ & $0.00423557447$ & $+0.70$ \\
\midrule
SSVEP & EEGNet (0.6 s) & 35 & $+0.21$ & $0.35\ (34)$ & $0.727554094$ & $+0.06$ \\
SSVEP & EEGNet (0.8 s) & 35 & $-0.65$ & $-1.04\ (34)$ & $0.613050642$ & $-0.18$ \\
SSVEP & EEGNet (1.0 s) & 35 & $-1.76$ & $-2.02\ (34)$ & $0.154051070$ & $-0.34$ \\
SSVEP & MTSNet (0.6 s) & 35 & $+2.24$ & $3.83\ (34)$ & $0.001597153525$ & $+0.65$ \\
SSVEP & MTSNet (0.8 s) & 35 & $+1.89$ & $2.05\ (34)$ & $0.0953715778$ & $+0.35$ \\
SSVEP & MTSNet (1.0 s) & 35 & $+1.39$ & $1.67\ (34)$ & $0.103955385$ & $+0.28$ \\
SSVEP & SSVEPFormer (0.6 s) & 35 & $+2.48$ & $3.14\ (34)$ & $0.0103493646$ & $+0.53$ \\
SSVEP & SSVEPFormer (0.8 s) & 35 & $+2.21$ & $2.30\ (34)$ & $0.0330535594$ & $+0.39$ \\
SSVEP & SSVEPFormer (1.0 s) & 35 & $+2.27$ & $2.52\ (34)$ & $0.0330535594$ & $+0.43$ \\
\midrule
Speech & EEGNet & 12 & $+8.33$ & $5.95\ (11)$ & $9.59414678\times10^{-5}$ & $+1.72$ \\
Speech & CNN-BiGRU & 12 & $+17.15$ & $9.21\ (11)$ & $1.66659570\times10^{-6}$ & $+2.66$ \\
Speech & EEG-Conformer & 12 & $+6.21$ & $5.04\ (11)$ & $0.000380624783$ & $+1.45$ \\
Speech & Du-IN & 12 & $+5.22$ & $4.54\ (11)$ & $0.000837161674$ & $+1.31$ \\
\bottomrule
\end{tabular*}
\end{table}

\clearpage

\clearpage
\setcounter{table}{0}
\setcounter{figure}{0}
\setcounter{equation}{5}
\clearpage
\newtheorem{proposition}{Proposition}[section]
\section{Theory and Attribution Methods}
\label{app:theory}

\subsection{Neural--Text Alignment: Properties and Correspondence}
\label{app:theory-notation}

\paragraph{Supervision paths and notation.}
Let $\mathcal S=\{1,\ldots,C\}$ be the training classes, $C\ge2$, and let $(X_i,y_i)$ denote a neural trial and its class label. We use the row-vector convention of Section~\ref{sec:methods}: $N_f(X_i)\in\mathbb R^{1\times d_f}$ is the neural encoder output. A task-specific CLS head $g_\psi$ produces $C$ class logits and minimizes
\begin{equation}
Z_{\mathrm{CLS}}(X_i)=g_\psi(N_f(X_i)),\qquad
p_{ic}^{\mathrm{CLS}}=\frac{e^{Z_{\mathrm{CLS}}(X_i)_c}}{\sum_{r=1}^{C}e^{Z_{\mathrm{CLS}}(X_i)_r}},\qquad
\ell_{\mathrm{CLS}}(X_i,y_i)=-\log p_{i,y_i}^{\mathrm{CLS}}.
\label{eq:app-cls-path}
\end{equation}
The head can be linear or nonlinear. Both CLS and LaS use class labels; LaS uses them to pair neural trials with class prompts.

For the complete prompt $P_c$, the text encoder gives $T_{f,c}\in\mathbb R^{1\times d_t}$. Bias-free, right-multiplying projections $W_n\in\mathbb R^{d_f\times d_{\mathrm{proj}}}$ and $W_t\in\mathbb R^{d_t\times d_{\mathrm{proj}}}$ define
\begin{equation}
\begin{aligned}
z_{\mathrm n,\mathrm{proj}}^{(i)}&=\nu\bigl(N_f(X_i)W_n\bigr),&
z_{\mathrm t,\mathrm{proj}}^{(c)}&=\nu(T_{f,c}W_t),&
\nu(v)&=v/\|v\|_2,\\
s_c(X_i)&=z_{\mathrm n,\mathrm{proj}}^{(i)}
\bigl(z_{\mathrm t,\mathrm{proj}}^{(c)}\bigr)^\top,&
a_{ic}&=s_c(X_i)/\tau,&\tau^{-1}&=e^t>0.
\end{aligned}
\label{eq:app-las-path}
\end{equation}
The projected vectors must be finite and nonzero. The normalized text vector is the class prototype, while $s_c$ and $a_{ic}$ denote the unscaled cosine score and temperature-scaled logit, respectively. For a paired batch of $B$ entries, $L_{ij}=a_{i,y_j}$ and
\begin{equation}
\begin{aligned}
\mathcal L_{\mathrm{N2T}}
&=-\frac1B\sum_{i=1}^{B}\log\frac{e^{a_{i,y_i}}}{\sum_{j=1}^{B}e^{a_{i,y_j}}},\\
\mathcal L_{\mathrm{T2N}}
&=-\frac1B\sum_{i=1}^{B}\log\frac{e^{a_{i,y_i}}}{\sum_{j=1}^{B}e^{a_{j,y_i}}},&
\mathcal L_{\mathrm{SA}}&=\tfrac12(\mathcal L_{\mathrm{N2T}}+\mathcal L_{\mathrm{T2N}}).
\end{aligned}
\label{eq:app-paired-infonce}
\end{equation}
This is the bidirectional InfoNCE objective~\citep{radford2021learning}. Neural and text encoder parameters are separate from those of CLS; each paired architecture is trained independently.

The repeated-prompt arguments assume one fixed prompt per class and deterministic text forwards, so all copies of a class text have the same embedding and prototype. Same-class neural trials may have different features. In a paired batch, $\mathcal G_c=\{i:y_i=c\}$, $n_c=|\mathcal G_c|$, $\mathcal S_b=\{c:n_c>0\}$ and $B=\sum_{c\in\mathcal S_b}n_c$. Summing over $\mathcal S_b$ is equivalent to indexing the represented classes consecutively; it avoids treating absent classes as paired entries. When all $C$ training classes occur equally often, $n_c=m$ and $B=mC$.

All logarithms are natural. InfoNCE averages its two directions; SigLIP sums all $B^2$ pairwise terms and divides by $B$. Comparisons of pair repetition or label rules hold the compared representations, temperature and bias fixed. Additional assumptions, including fixed mean scores and full row rank, are stated for the relevant results below.

\paragraph{Effective class weights and linear decisions.}
Define the effective reading weight $q_c\in\mathbb R^{d_f\times1}$ by
\begin{equation}
q_c=W_n\bigl(z_{\mathrm t,\mathrm{proj}}^{(c)}\bigr)^\top,\qquad
\frac{s_c(X_i)}{\tau}=\frac{N_f(X_i)q_c}{\tau\|N_f(X_i)W_n\|_2}.
\label{eq:app-effective-weights}
\end{equation}
\begin{proposition}[Prototype-generated ordering]
\label{prop:app-prototype-ordering}
For fixed parameters and a finite candidate set $\mathcal C$, on the domain $N_f(X_i)W_n\ne0$,
\begin{equation}
\begin{aligned}
\operatorname*{arg\,max}_{c\in\mathcal C}s_c(X_i)
&=\operatorname*{arg\,max}_{c\in\mathcal C}N_f(X_i)q_c,\\
s_c(X_i)=s_{c'}(X_i)
&\Longleftrightarrow z_{\mathrm n,\mathrm{proj}}^{(i)}
\bigl(z_{\mathrm t,\mathrm{proj}}^{(c)}-z_{\mathrm t,\mathrm{proj}}^{(c')}\bigr)^\top=0\\
&\Longleftrightarrow N_f(X_i)(q_c-q_{c'})=0.
\end{aligned}
\label{eq:app-linear-ranking}
\end{equation}
\end{proposition}
\emph{Proof.} Substituting the normalized neural projection into the matching score yields Equation~\eqref{eq:app-effective-weights}. Its denominator is a positive scalar shared by every candidate for this input, so it preserves all pairwise orderings. The tie relations follow by subtraction.

If $q_c\ne q_{c'}$, their equal-score set is a hyperplane through the origin in neural-feature space, restricted to the valid projection domain. If the weights coincide, the two classes tie throughout that domain. A pairwise tie becomes a prediction boundary only where neither score is below another candidate's. The encoder from the raw signal to $N_f$ may remain nonlinear. The common denominator cancels for rankings, but softmax probabilities and training losses use the complete logits.

\paragraph{Prototype distances and score variation.}
For fixed neural features, suppose both compared text projections have norm at least $\rho>0$. With $\|W_t\|_2$ denoting the spectral norm,
\begin{equation}
|s_c(X_i)-s_{c'}(X_i)|
\le\bigl\|z_{\mathrm t,\mathrm{proj}}^{(c)}-z_{\mathrm t,\mathrm{proj}}^{(c')}\bigr\|_2
\le\frac{2\|W_t\|_2}{\rho}\|T_{f,c}-T_{f,c'}\|_2.
\label{eq:app-prototype-distance}
\end{equation}
\emph{Proof.} The first inequality follows from the unit norm of the neural vector and Cauchy--Schwarz. Let $v=T_{f,c}W_t$ and $w=T_{f,c'}W_t$. The triangle and reverse-triangle inequalities give
\begin{equation}
\|\nu(v)-\nu(w)\|_2
\le\frac{\|v-w\|_2+|\|w\|_2-\|v\|_2|}{\|v\|_2}
\le\frac{2\|v-w\|_2}{\rho}.
\end{equation}
Finally, $\|(T_{f,c}-T_{f,c'})W_t\|_2\le\|T_{f,c}-T_{f,c'}\|_2\|W_t\|_2$. Coincident prototypes therefore give identical scores. Positive rescaling of $W_t$ rescales both $\rho$ and $\|W_t\|_2$, preserving the prototypes and their ratio in the bound.

\paragraph{N2T decomposition by class.}
Grouping identical class-text copies in the complete $B$-entry denominator gives, for each neural query,
\begin{equation}
\begin{aligned}
\ell_i^{\mathrm{N2T}}
&=-a_{i,y_i}+\log\sum_{c\in\mathcal S_b}n_c e^{a_{ic}}\\
&=\log n_{y_i}+\log\left[1+
\sum_{\substack{c\in\mathcal S_b\\c\ne y_i}}
\frac{n_c}{n_{y_i}}e^{a_{ic}-a_{i,y_i}}\right].
\end{aligned}
\label{eq:app-n2t-counts}
\end{equation}
Thus, multiplicities weight competing classes through $n_c/n_{y_i}$. For a batch containing all $C$ classes $m$ times, the denominator is $m\sum_{c=1}^C e^{a_{ic}}$, so
\begin{equation}
\ell_i^{\mathrm{N2T}}
=\underbrace{-a_{i,y_i}+\log\sum_{c=1}^C e^{a_{ic}}}_{\text{class cross-entropy on matching logits}}
+\log m.
\label{eq:app-n2t-balanced}
\end{equation}
For this fixed query, repetition adds a parameter-independent constant and leaves class-score gradients unchanged. Enlarging the batch with new neural trials also introduces new query terms. Class cross-entropy and this decomposition allow same-class trials to have different correct-class scores.

\paragraph{T2N matching balance along the corresponding text-prototype direction.}
\label{app:theory-score}
Every copy of class text $c$ has denominator $\sum_{j=1}^B e^{a_{jc}}$, while successive diagonal targets select different $i\in\mathcal G_c$. Averaging their negative log probabilities yields
\begin{equation}
\bar a_c=\frac1{n_c}\sum_{i\in\mathcal G_c}a_{ic},\qquad
\bar\ell_c^{\mathrm{T2N}}=-\bar a_c+\log\sum_{j=1}^{B}e^{a_{jc}},\qquad
\mathcal L_{\mathrm{T2N}}=\frac1B\sum_{c\in\mathcal S_b}n_c\bar\ell_c^{\mathrm{T2N}}.
\label{eq:app-t2n-class-mean}
\end{equation}
\begin{proposition}[Within-class matching balance]
\label{prop:app-t2n-balance}
Fix the mean $\bar a_c$ and every other-class score $a_{jc}$, $j\notin\mathcal G_c$. The class loss $\bar\ell_c^{\mathrm{T2N}}$ is minimized when all within-class scores equal $\bar a_c$.
\end{proposition}
The repeated text queries have the same denominator. Under these conditions, their average satisfies
\begin{equation}
\begin{aligned}
\bar\ell_c^{\mathrm{T2N}}
&=\log\left[\sum_{i\in\mathcal G_c}e^{a_{ic}-\bar a_c}
+\sum_{j\notin\mathcal G_c}e^{a_{jc}-\bar a_c}\right]\\
&\ge\log\left[n_c+\sum_{j\notin\mathcal G_c}e^{a_{jc}-\bar a_c}\right].
\end{aligned}
\label{eq:app-t2n-jensen}
\end{equation}
\emph{Proof.} The within-class deviations sum to zero. Strict convexity of the exponential gives $\sum_{i\in\mathcal G_c}e^{a_{ic}-\bar a_c}\ge n_c$, with equality exactly when every deviation is zero. The other-class sum remains fixed under the stated conditions, proving that balanced same-class scores minimize this class loss.

For $n_c>1$, this balance provides implicit regularization toward within-class aggregation of matching responses along the corresponding text-prototype direction, under the fixed-score conditions above. Other feature directions may retain trial differences; at $n_c=1$, there are no within-class score differences to balance.

In a two-trial same-class example, one identical text copy targets the first trial and treats the second as an off-diagonal competitor; the other reverses these roles. The shared denominator couples their scores. This class-query coupling and the same-class false negatives arise from the same repeated pairing structure. A fixed prompt permutation preserves a stable common reference within each class, although it changes the assigned descriptions.

\paragraph{Compensating a fixed seen-prompt permutation.}
\label{app:theory-correspondence}
Stack the $C$ seen text embeddings in class order. Let $\Pi\in\mathbb R^{C\times C}$ be a permutation matrix, fixed throughout training, so the assigned embedding matrix is $\Pi T_f^{\mathcal S}$:
\begin{equation}
T_f^{\mathcal S}=\begin{bmatrix}T_{f,1}\\\vdots\\T_{f,C}\end{bmatrix}\in\mathbb R^{C\times d_t}.
\label{eq:app-seen-text-matrix}
\end{equation}
\begin{proposition}[Seen-prompt permutation compensation]
\label{prop:app-permutation-compensation}
Assume $\operatorname{rank}(T_f^{\mathcal S})=C$ and all original projected prototypes are finite and nonzero. Full row rank requires $d_t\ge C$. For any original $W_t$, keeping the neural branch fixed, a projection of the same shape restores all seen prototypes and pairing scores:
\begin{equation}
W_t^\Pi=(\Pi T_f^{\mathcal S})^\dagger T_f^{\mathcal S}W_t,\qquad
\Pi T_f^{\mathcal S}W_t^\Pi=T_f^{\mathcal S}W_t,
\label{eq:app-permutation-compensation}
\end{equation}
\end{proposition}
Here $\dagger$ denotes the Moore--Penrose pseudoinverse.

\emph{Proof.} Multiplication by a permutation preserves row rank, hence $(\Pi T_f^{\mathcal S})(\Pi T_f^{\mathcal S})^\dagger=I_C$. Substituting the proposed projection proves the second identity. Consequently, each assigned seen text has the same projected vector and normalized prototype as before. Keeping the neural branch and temperature unchanged preserves every seen pairing score and the data loss.

\paragraph{Seen correspondence and unseen extension.}
Unseen classes $\mathcal U$ are disjoint from $\mathcal S$ and use their true descriptions at prediction. Suppose an unseen embedding lies in the span of the seen embeddings,
\begin{equation}
T_{f,u}=\beta_u^\top T_f^{\mathcal S},\qquad \beta_u\in\mathbb R^C.
\label{eq:app-unseen-combination}
\end{equation}
Under this span assumption, its projected vector is the same linear combination of the projected seen embeddings. The compensation in Equation~\eqref{eq:app-permutation-compensation} restores the \emph{assigned, permuted} seen correspondence, while the neural relationships remain fixed. Extension to the true unseen description also depends on the combination coefficients $\beta_u$. Thus, fitting the seen classes and extending to unseen classes impose distinct correspondence requirements.

Successful transfer also requires compatibility between the neural distinctions available in a task and the relations in its class descriptions. In SSVEP, stimulus frequency and the EEG frequency response provide a physical connection to the numerical frequency in a prompt. The model receives frequency text, without an extra hand-designed attribute coordinate. This physical connection motivates the conditional evaluation.

\paragraph{Candidate competition.}
Unseen-only prediction uses $\mathcal C=\mathcal U$, whereas adding seen candidates uses $\mathcal C=\mathcal S\cup\mathcal U$. With parameters fixed, both candidate sets preserve each neural representation and class score. If the highest seen score exceeds all unseen scores, the joint prediction selects a seen class even when the true class ranks first within $\mathcal U$. Joint prediction therefore also compares the leading unseen score with the seen scores.

\paragraph{Text-representation scope.}
The preceding score bound concerns text embeddings after the chosen text encoder and readout, followed by the learned projection. A contextual BERT [CLS] vector~\citep{devlin2019bert} and an embedding model trained for sentence comparison~\citep{schechtervera2025embeddinggemma} have different training objectives and readouts. Changing the encoder, readout, or prompt changes $T_{f,c}$ and may change the resulting prototypes; Equation~\eqref{eq:app-prototype-distance} bounds score variation for fixed neural features under its stated norm condition. The permutation construction establishes an available parameter solution for the seen classes.

\subsection{Repeated-Prompt Training Objectives}
\label{app:theory-losses}

\paragraph{Pair counts.}
The arguments below use one fixed prompt per class and a deterministic text forward, so same-class text copies have identical embeddings. Diagonal labels mark $i=j$ positive; class-aware labels mark every same-class pair positive. Among off-diagonal entries, the diagonal rule treats $y_i=y_j$ as negative:
\begin{equation}
M_{ij}^{\mathrm{diag}}=\mathbf1[i=j],\qquad
M_{ij}^{\mathrm{class}}=\mathbf1[y_i=y_j],\qquad
N_{\mathrm{FN}}=\sum_{c\in\mathcal S_b}n_c(n_c-1).
\label{eq:app-false-negative-count}
\end{equation}
Each neural row contains one diagonal positive, $n_{y_i}-1$ same-class false negatives and $B-n_{y_i}$ other-class negatives. For $B=mC$, the latter two counts become $m-1$ and $m(C-1)$. Class-aware labels turn all same-class pairs positive; the $m=1$ baseline has no same-class off-diagonal pairs.

\paragraph{InfoNCE class aggregation.}
Grouping the identical text copies in the full paired denominator gives
\begin{equation}
\begin{aligned}
\ell_i^{\mathrm{N2T}}
&=-\log\frac{e^{a_{i,y_i}}}{\sum_{j=1}^{B}e^{a_{i,y_j}}}
=-a_{i,y_i}+\log\sum_{c\in\mathcal S_b}n_c e^{a_{ic}}\\
&=\log n_{y_i}+\log\!\left[1+\sum_{\substack{c\in\mathcal S_b\\c\ne y_i}}
\frac{n_c}{n_{y_i}}e^{a_{ic}-a_{i,y_i}}\right].
\end{aligned}
\label{eq:app-paired-class-aggregation}
\end{equation}
For $n_c=m$ across all $C$ classes, the full denominator and loss reduce to
\begin{equation}
\begin{aligned}
\sum_{j=1}^{B}e^{a_{i,y_j}}&=m\sum_{c=1}^{C}e^{a_{ic}},\\
\ell_i^{\mathrm{N2T}}&=\underbrace{-a_{i,y_i}+\log\sum_{c=1}^{C}e^{a_{ic}}}_{\text{class cross-entropy on matching logits}}+\log m.
\end{aligned}
\label{eq:app-paired-balanced}
\end{equation}
For a fixed neural query, repetition adds a parameter-independent constant and leaves class-score gradients unchanged. New neural samples contribute new queries. In T2N, copies of one class text share a denominator; Equation~\eqref{eq:app-t2n-class-mean} and the conditional Jensen bound in Equation~\eqref{eq:app-t2n-jensen} give their class-averaged response constraint. Bidirectional InfoNCE averages the two full directional losses.

\paragraph{SigLIP label conflicts and gradients.}
With shared bias $b_{\mathrm{sig}}$, define $\xi_{ic}=a_{ic}+b_{\mathrm{sig}}$, $\ell_+(\xi)=\log(1+e^{-\xi})$, and $\ell_-(\xi)=\log(1+e^{\xi})$. Both SigLIP variants~\citep{zhai2023sigmoid} sum over all $B^2$ pairs and divide by $B$:
\begin{equation}
\mathcal L_{\mathrm{Sig}}(M)=\frac1B\sum_{i=1}^{B}\sum_{j=1}^{B}
\left[M_{ij}\ell_+(\xi_{i,y_j})+(1-M_{ij})\ell_-(\xi_{i,y_j})\right].
\label{eq:app-sigmoid-full}
\end{equation}
Grouping the repeated texts gives row losses
\begin{equation}
\begin{aligned}
\ell_i^{\mathrm{diag}}
&=\ell_+(\xi_{i,y_i})+(n_{y_i}-1)\ell_-(\xi_{i,y_i})
+\sum_{\substack{c\in\mathcal S_b\\c\ne y_i}}n_c\ell_-(\xi_{ic}),\\
\ell_i^{\mathrm{class}}
&=n_{y_i}\ell_+(\xi_{i,y_i})
+\sum_{\substack{c\in\mathcal S_b\\c\ne y_i}}n_c\ell_-(\xi_{ic}).
\end{aligned}
\label{eq:app-sigmoid-rows}
\end{equation}
The full losses average these $B$ row losses. Since $1+e^\xi=e^\xi(1+e^{-\xi})$, changing a same-class off-diagonal label changes its loss by $\ell_-(\xi)-\ell_+(\xi)=\xi$. Therefore, at the same parameter state,
\begin{equation}
\mathcal L_{\mathrm{Sig}}^{\mathrm{diag}}-\mathcal L_{\mathrm{Sig}}^{\mathrm{class}}
=\frac1B\sum_{i=1}^B(n_{y_i}-1)\xi_{i,y_i}.
\label{eq:app-sigmoid-difference}
\end{equation}
Let $\sigma_{\mathrm{sig}}(\xi)=(1+e^{-\xi})^{-1}$. Using $\ell_+'=-(1-\sigma_{\mathrm{sig}})$ and $\ell_-'=\sigma_{\mathrm{sig}}$ gives
\begin{equation}
\frac{\partial\ell_i^{\mathrm{diag}}}{\partial\xi_{i,y_i}}
=n_{y_i}\sigma_{\mathrm{sig}}(\xi_{i,y_i})-1,\qquad
\frac{\partial\ell_i^{\mathrm{class}}}{\partial\xi_{i,y_i}}
=-n_{y_i}\bigl[1-\sigma_{\mathrm{sig}}(\xi_{i,y_i})\bigr].
\label{eq:app-sigmoid-gradients}
\end{equation}
The derivatives of the complete loss include a further factor $1/B$. Diagonal supervision asks one positive term to raise the correct-class score while $n_{y_i}-1$ identical-score negative terms lower it. Class-aware supervision removes this conflict and increases the number of positive terms to $n_{y_i}$. The loss difference in Equation~\eqref{eq:app-sigmoid-difference} can be positive, zero, or negative.

The normalization retains a repetition-dependent scale. If the same representation of each class is copied $m$ times, class-aware pair terms multiply by $m^2$ but the denominator only by $m$, so the full loss is $m$ times its one-copy baseline. These pair terms add independently, whereas T2N couples trials through a shared softmax denominator. Removing label conflicts also changes positive-pair weights.

\paragraph{Bidirectional InfoNCE bound.}
The second logarithm in Equation~\eqref{eq:app-n2t-counts} is nonnegative, hence $\ell_i^{\mathrm{N2T}}\ge\log n_{y_i}$. Equation~\eqref{eq:app-t2n-jensen} likewise gives $\bar\ell_c^{\mathrm{T2N}}\ge\log n_c$. Aggregating and averaging directions yields
\begin{equation}
\begin{aligned}
\mathcal L_{\mathrm{N2T}}&\ge\frac1B\sum_{i=1}^B\log n_{y_i}
=\frac1B\sum_{c\in\mathcal S_b}n_c\log n_c,\\
\mathcal L_{\mathrm{T2N}}&\ge\frac1B\sum_{c\in\mathcal S_b}n_c\log n_c,\\
\mathcal L_{\mathrm{SA}}&\ge\frac1B\sum_{c\in\mathcal S_b}n_c\log n_c.
\end{aligned}
\label{eq:app-infonce-floor}
\end{equation}
With $n_c=m$ for all $C$ classes, this becomes $\log m$. Approaching the bound requires other-class exponential terms to become negligible relative to same-class terms; T2N additionally requires balanced same-class scores. Identical texts cannot distinguish the $m$ pairing positions even when classes are separated.

\paragraph{Diagonal SigLIP bound.}
For a class count $n>1$, the same-class contribution to a row is $f_n(\xi)=\ell_+(\xi)+(n-1)\ell_-(\xi)$. Treating the logit temporarily as a free scalar,
\begin{equation}
\begin{aligned}
f_n'(\xi)&=n\sigma_{\mathrm{sig}}(\xi)-1=0
\quad\Longrightarrow\quad \xi^\star=-\log(n-1),\\
f_n''(\xi)&=n\sigma_{\mathrm{sig}}(\xi)\bigl[1-\sigma_{\mathrm{sig}}(\xi)\bigr]>0,\\
f_n(\xi^\star)&=\log n+(n-1)\log\frac{n}{n-1}
=n\log n-(n-1)\log(n-1).
\end{aligned}
\label{eq:app-sigmoid-scalar-minimum}
\end{equation}
The stationary point is therefore the unique minimum. Other-class terms are nonnegative; substituting $n=n_{y_i}$ and averaging the row bounds gives
\begin{equation}
\mathcal L_{\mathrm{Sig}}^{\mathrm{diag}}
\ge\frac1B\sum_{c\in\mathcal S_b}n_c\left[n_c\log n_c-(n_c-1)\log(n_c-1)\right].
\label{eq:app-diagonal-floor}
\end{equation}
For $n_c=m$, this reduces to $m\log m-(m-1)\log(m-1)$. At $n=1$ the same-class contribution has only a positive term and infimum zero, consistent with $0\log0=0$. Approaching the full bound additionally requires all other-class negative losses to vanish and the same-class row minima to be jointly realizable by the shared neural--text model.

\paragraph{Class-aware SigLIP bound.}
All pair losses are nonnegative, so $\mathcal L_{\mathrm{Sig}}^{\mathrm{class}}\ge0$. If positive and negative pairs can be simultaneously separated with unbounded logit magnitudes, positive logits can tend to $+\infty$ and negative logits to $-\infty$, making the loss approach zero. Under this separation condition, class-aware labels remove the positive floor caused by conflicting labels on identical copies.

These bounds use the specified objectives and actual paired-batch counts. Their attainability depends on the separation and shared-logit conditions above. Figure~\ref{fig:main-7} continuously extends the balanced formulas to $m=B/C$; fractional $m$ serves as a plotting reference, while literal balanced batches have integer multiplicities. The empirical points are each run's minimum epoch-average full training loss.

\subsection{Dual-Sided Attribution for Class Matching}
\label{app:attribution-method}

\paragraph{Score and neural-input path.}
For a trained LaS model and a specified class prompt $t$, both attribution sides concern the unscaled cosine matching score
\begin{equation}
s(x,t)=\left\langle
\frac{P_E(f_E(x))}{\|P_E(f_E(x))\|_2},
\frac{P_T(f_T(t))}{\|P_T(f_T(t))\|_2}
\right\rangle .
\label{eq:app-attribution-score}
\end{equation}
Here $x$ denotes the complete preprocessed model input, including all input branches, and $P_E,P_T$ include the respective projections. This is $s_c(X)$ from Section~\ref{sec:methods}, before division by $\tau$. For neural attribution, the class text is fixed. All neural-input components move together from the zero model-input baseline to $x$ along $x(\alpha)=\alpha x$, $0\le\alpha\le1$. The zero baseline denotes zero in the preprocessed input space; it has no direct physiological interpretation.

Partition the input into disjoint groups $g$. For a channel $c$ and a temporal group $p$, grouped Integrated Gradients (IG)~\citep{sundararajan2017axiomatic} gives
\begin{equation}
A_{c,p}=\sum_{u\in p}x_{c,u}\int_0^1
\left.\frac{\partial s(\tilde x,t)}{\partial\tilde x_{c,u}}\right|_{\tilde x=\alpha x}\,d\alpha.
\label{eq:app-eeg-ig}
\end{equation}
The implementation integrates derivatives with respect to a gate shared by the components of each group, using Gauss--Legendre quadrature. If $x(\boldsymbol\gamma)$ multiplies group $g$ by $\gamma_g$, then $\partial s(x(\boldsymbol\gamma),t)/\partial\gamma_g$ already contains the factor $x_{c,u}$ in Equation~\eqref{eq:app-eeg-ig}; the integrated gate derivative is the group attribution itself. Raw $A_g$ retains its sign. Provided the endpoint scores are defined, the summed groups are checked numerically against
\begin{equation}
\sum_g A_g\simeq s(x,t)-s(0,t).
\label{eq:app-eeg-completeness}
\end{equation}
A positive value raises the specified score along this path; a negative value lowers it. These are input-to-score attributions for the trained model.

\paragraph{Input groups and physical coordinates.}
For MI, each one-second input window is divided into five 200-ms groups for each of 32 channels, giving $32\times5$ groups. A five-second trial has 21 overlapping windows at 0.2-s steps. Window-group attributions are returned to the trial's physical timeline using each recorded window start. Let $m_{i,b}$ count the windows covering physical 200-ms bin $b$ in trial $i$. For the signed raw attributions,
\begin{equation}
B_{i,c,b}=\frac{1}{m_{i,b}}\sum_{(w,p)\mapsto b}A_{i,w,c,p},
\qquad H_{c,b}=\frac1N\sum_{i=1}^{N}B_{i,c,b}.
\label{eq:app-mi-overlap}
\end{equation}
The $32\times25$ array $H$ is the signed trial-averaged channel-by-time map. MI's scalp values are its means across the 25 physical time bins.

For Speech, each three-second trial has 10 physical channels and 30 successive 100-ms groups per channel. Each resulting $10\times30$ attribution array is averaged with equal trial weights, without overlapping-window reconstruction. For SSVEP, one input contains 32 time-branch and 32 frequency-branch channel groups. Both branches follow the same $\alpha$ path; trial means are computed separately for the two branches and mapped to two scalp plots. Frequency-branch channel attributions are not assigned a physical time bin. Topomap interpolation changes the display between electrodes, not the channel values passed to it.

\paragraph{Conditional attribution of prompt words.}
For a displayed word $j$, set all its token embeddings to zero and restore only that word along a separate interpolation path. Other words, any text prefix, positional information and the attention mask remain fixed. For MI, the neural anchor is the mean of unit-normalized projected window features within each trial and then across the $N$ selected trials:
\begin{equation}
\bar z_E=\frac1N\sum_{i=1}^{N}\frac1{W_i}\sum_{w=1}^{W_i}\hat z_E(x_{i,w}),
\qquad
I_j=\int_0^1\frac{d}{d\alpha}
\left[\bar z_E\,\hat z_T(e_{\neg j}+\alpha e_j)^\top\right]d\alpha .
\label{eq:app-word-conditional-ig}
\end{equation}
The mean $\bar z_E$ is not normalized again. Speech and SSVEP use equal weights for their selected trials in the corresponding neural anchor. The text forward uses the checkpoint's text encoder and projection, including LoRA when that checkpoint contains it. Subject to numerical integration error, $I_j$ is the full-prompt score minus the score obtained by zeroing word $j$ alone. Each word has its own baseline. Word scores are signed and do not sum to a decomposition of the whole prompt.

\paragraph{From raw attributions to Figure~\ref{fig:main-6}.}
Raw signed arrays, endpoint scores and integration residuals are retained before display. In a raw-contribution view, signed values can be aggregated as in Equation~\eqref{eq:app-mi-overlap} and negative values clipped only when applying a positive-contribution color scale. Such a view retains score-change units. The relative positive-contribution view used in Figure~\ref{fig:main-6} applies the following transform to every complete neural model input $i$ \emph{before} temporal or trial aggregation:
\begin{equation}
R_i(g)=
\begin{cases}
\dfrac{\operatorname{ReLU}(A_i(g))}{\max_h\operatorname{ReLU}(A_i(h))},&\max_h\operatorname{ReLU}(A_i(h))>0,\\[6pt]
0,&\text{otherwise}.
\end{cases}
\label{eq:app-relative-positive}
\end{equation}
The denominator spans all 160 groups of one MI window, all 300 groups of one Speech trial, or the 64 groups across both SSVEP branches of one trial. An input with no positive attribution remains all zero and remains in the average. For MI, the transformed window values follow the overlap and trial averaging of Equation~\eqref{eq:app-mi-overlap}; its topomap averages the resulting 25 time bins. Speech averages transformed trial arrays. SSVEP averages the transformed time and frequency groups separately within each stimulus class. No mean map is normalized again. The resulting dimensionless values show relative positive contributions; Equation~\eqref{eq:app-eeg-completeness} applies to the signed raw attributions.

Word highlighting follows a different display rule. Signed conditional word scores are averaged at the relevant window and trial levels; the displayed intensity is $\max(I_j,0)/M_{\mathrm{word}}$ when the separate prompt-level color limit $M_{\mathrm{word}}$ is positive. If all averaged word scores are nonpositive, all words remain white. The neural maps use the \texttt{RdBu\_r} scale from zero to the indicated MAX: zero and clipped negative values are blue, intermediate positive values approach white, and the largest displayed positive values are red. Raw-contribution plots may share a limit within a specified comparison; relative maps and topomaps may have separate panel limits. Comparisons of color therefore use each panel's MAX, units and processing mode.

\clearpage
\setcounter{table}{0}
\setcounter{figure}{0}
\setcounter{equation}{0}
\renewcommand{\thetable}{D.\arabic{table}}
\renewcommand{\thefigure}{D.\arabic{figure}}
\section{Further Analysis of SSVEP}
\label{app:ssvep-further}

\subsection{SSVEP Responses and Short-Window Evaluation}
\label{app:ssvep-responses}

An SSVEP is an oscillatory visual response whose frequency follows a periodic stimulus~\citep{vialatte2010steady}. The benchmark used here presents 40 targets distinguished by frequency and phase, with 5~s of stimulation per trial~\citep{wang2017benchmark}. Decoding commonly uses a short segment from such a trial. The benchmark study reported a mean visual latency of approximately 137~ms and recommended a shift of approximately 140~ms when extracting analysis windows~\citep{wang2017benchmark}. We use $t_0=0.14$~s after stimulus onset as a common analysis origin; physiological latency may vary among subjects. An initial window (IW) of length $L$ spans $[t_0,t_0+L)$; an onset offset $\delta$ shifts it to $[t_0+\delta,t_0+\delta+L)$. Window length and onset offset are varied separately below.

\subsection{Window Augmentation and Closed-Set Results}
\label{app:ssvep-window-augmentation}

Let $x_{u,c,r}(t)$ denote all channels of trial $r$ for subject $u$ and class $c$. The valid stimulus interval ends at $t_1=5.00$~s. The following operations apply to training signals; validation and test windows are specified for each result.

\paragraph{Random-window augmentation (RW).}
RW samples a subject, class, and trial from the training set, then selects one start time within that trial. Every channel contributes the same contiguous segment of length $L$:
\begin{equation}
\widetilde{x}^{\mathrm{RW}}=x_{u,c,r}[a:a+L],
\qquad a\sim\mathcal{U}[t_0,t_1-L].
\label{eq:app-d-rw}
\end{equation}
The upper bound keeps the window inside the valid interval.

\paragraph{Inter-trial recombination (IR).}
For one subject and class, IR extracts $[t_0,t_0+L)$ from each available training trial. It applies common segment boundaries $0=\tau_0<\tau_1<\cdots<\tau_J=L$, chooses a trial $r_j$ for each segment, and concatenates the selected pieces:
\begin{equation}
\widetilde{x}^{\mathrm{IR}}
=\bigoplus_{j=1}^{J}x_{u,c,r_j}
[t_0+\tau_{j-1}:t_0+\tau_j].
\label{eq:app-d-ir}
\end{equation}
The segments retain their temporal positions while their source trials change. This follows the time-domain segmentation and recombination idea of Lotte~\citep{lotte2015signal}; the subject, class, and initial-window restrictions above specify our SSVEP implementation.

\paragraph{Temporal masking.}
Given a window $x\in\mathbb{R}^{H\times T}$ and mask ratio $\rho$, one contiguous interval of approximately $h=\rho T$ time samples is set to zero in every channel:
\begin{equation}
\widetilde{x}_{j,t}=x_{j,t}M_t,\qquad
M_t=\begin{cases}0,&q\leq t<q+h,\\1,&\text{otherwise},\end{cases}
\quad q\sim\mathcal{U}\{0,\ldots,T-h\}.
\label{eq:app-d-mask}
\end{equation}
The admissible starts ensure that the masked interval stays inside the window. Mask-based EEG augmentation has a precedent in \citet{ding2024maskencoding}; our mask is a shared contiguous time interval. In the RW-trained experiments of Table~\ref{tab:overall-cross-encoder}, training produces approximately 300,000 windows and uses $\rho=0.3$ for the neural CLS and LaS models. FBCCA directly classifies the test signal and has no training stage~\citep{chen2015filter}.

\paragraph{RW-trained closed-set decoding.}
Table~\ref{tab:app-d-table1-rw} reproduces the nine neural comparisons from Table~\ref{tab:overall-cross-encoder}. All models use RW augmentation and temporal masking during training. Evaluation uses one Physical IW per trial at $\delta=0$. The LaS results use EmbeddingGemma-300M. Superscripts reproduce the Holm-corrected paired-test marks in Table~\ref{tab:overall-cross-encoder}.

\begin{table}[!htbp]
\centering
\caption{Forty-class accuracy (\%) after RW training and IW evaluation. $\Delta$ is LaS minus CLS in percentage points.}
\label{tab:app-d-table1-rw}
\small
\begin{tabular}{lcccc}
\toprule
EEG encoder & $L$ & CLS & LaS & $\Delta$ \\
\midrule
EEGNet & 0.6~s & 26.76 & 26.98 & +0.21 \\
 & 0.8~s & 41.07 & 40.42 & $-0.65$ \\
 & 1.0~s & 51.62 & 49.86 & $-1.76$ \\
MTSNet & 0.6~s & 34.32 & 36.56 & +2.24$^{**}$ \\
 & 0.8~s & 53.27 & 55.17 & +1.89 \\
 & 1.0~s & 65.39 & 66.79 & +1.39 \\
SSVEPFormer & 0.6~s & 33.46 & 35.94 & +2.48$^{*}$ \\
 & 0.8~s & 52.30 & 54.51 & +2.21$^{*}$ \\
 & 1.0~s & 65.23 & 67.50 & +2.27$^{*}$ \\
\bottomrule
\end{tabular}
\end{table}

\paragraph{Random test windows and IW.}
Table~\ref{tab:app-d-rw-iw} reports RW-trained LaS (MTSNet) with EmbeddingGemma-300M. The random-window evaluation samples 10,000 test windows per fold; IW evaluation fixes $\delta=0$ on Physical trials. The two columns therefore sample different within-trial positions.

\begin{table}[H]
\centering
\caption{Closed-set accuracy (\%) for RW-trained LaS (MTSNet) under random-window and IW evaluation.}
\label{tab:app-d-rw-iw}
\small
\begin{tabular}{ccc}
\toprule
Window length $L$ & 10,000 random test windows per fold & Physical IW ($\delta=0$) \\
\midrule
0.6~s & 48.1171 & 36.5595 \\
0.8~s & 63.4886 & 55.1667 \\
1.0~s & 74.7371 & 66.7857 \\
\bottomrule
\end{tabular}
\end{table}

\paragraph{IW-trained and IR-augmented closed-set results.}
Tables~\ref{tab:app-d-iw-accuracy} and~\ref{tab:app-d-iw-itr} reproduce the five-duration results in Appendix~\ref{app:b-main-adaptation}. LaS (MTSNet) uses EmbeddingGemma-300M and approximately 300,000 IR training windows. Neural CLS baselines train on IW without IR and retain temporal masking at $\rho=0.3$; FBCCA uses neither training nor augmentation. Every method is evaluated on IW. Estimated ITR includes a 0.5-s gaze-shift allowance per selection.

\begin{table}[!htbp]
\centering
\caption{Forty-class IW accuracy (\%, subject mean) across five window lengths. Values and significance marks reproduce Appendix~\ref{app:b-main-adaptation}.}
\label{tab:app-d-iw-accuracy}
\small
\begin{tabular}{lccccc}
\toprule
Method & 0.6~s & 0.7~s & 0.8~s & 0.9~s & 1.0~s \\
\midrule
FBCCA & 28.25$^{***}$ & 38.20$^{***}$ & 48.54$^{***}$ & 56.90$^{***}$ & 63.75$^{***}$ \\
EEGNet (CLS) & 45.24$^{***}$ & 52.51$^{***}$ & 58.38$^{***}$ & 63.52$^{***}$ & 65.45$^{***}$ \\
SSVEPFormer (CLS) & 55.25 & 63.46 & 70.33 & 75.65 & 77.74 \\
MTSNet (CLS) & 52.06$^{*}$ & 62.61 & 68.61$^{**}$ & 73.98$^{*}$ & 78.01 \\
LaS (MTSNet) & 55.18 & 64.46 & 71.93 & 76.42 & 79.50 \\
\bottomrule
\end{tabular}
\end{table}

\begin{table}[!htbp]
\centering
\caption{Estimated ITR (bits/min, subject mean) for the same IW conditions. Values and significance marks reproduce Appendix~\ref{app:b-main-adaptation}.}
\label{tab:app-d-iw-itr}
\small
\begin{tabular}{lccccc}
\toprule
Method & 0.6~s & 0.7~s & 0.8~s & 0.9~s & 1.0~s \\
\midrule
FBCCA & 41.0$^{***}$ & 60.4$^{***}$ & 80.0$^{***}$ & 94.5$^{***}$ & 104.7$^{***}$ \\
EEGNet (CLS) & 84.6$^{***}$ & 96.8$^{***}$ & 105.1$^{***}$ & 111.2$^{***}$ & 108.6$^{***}$ \\
SSVEPFormer (CLS) & 114.4 & 129.6 & 139.5 & 145.5 & 141.2$^{*}$ \\
MTSNet (CLS) & 105.7$^{*}$ & 126.1 & 134.7$^{**}$ & 140.6$^{*}$ & 142.3 \\
LaS (MTSNet) & 114.8 & 132.4 & 144.8 & 147.6 & 146.4 \\
\bottomrule
\end{tabular}
\end{table}

The superscripts in these two tables reproduce the paired-test marks reported in Appendix~\ref{app:b-main-adaptation}, separately for accuracy and ITR.

\paragraph{IR volume and mask ratio.}
Tables~\ref{tab:app-d-ir-volume} and~\ref{tab:app-d-mask-ratio} use MTSNet with BERT-base-uncased at $L=1.0$~s under IW evaluation. The IR-volume scan holds $\rho=0.3$; the mask scan holds approximately 300,000 IR training windows. The 300-trial condition at $\rho=0.3$ is common to both tables.

\begin{table}[!htbp]
\centering
\caption{IR training volume and IW closed-set accuracy (\%). The no-IR condition uses approximately 6,000 training windows.}
\label{tab:app-d-ir-volume}
\small
\begin{tabular}{cc}
\toprule
IR trials / approximate training windows & IW accuracy \\
\midrule
No IR / 6,000 & 34.7738 \\
50 / 50,000 & 71.1667 \\
100 / 100,000 & 73.3810 \\
200 / 200,000 & 74.6905 \\
300 / 300,000 & 76.3571 \\
\bottomrule
\end{tabular}
\end{table}

\begin{table}[!htbp]
\centering
\caption{Temporal mask ratio and IW closed-set accuracy (\%) with approximately 300,000 IR windows. The epoch budgets differ among mask conditions.}
\label{tab:app-d-mask-ratio}
\small
\begin{tabular}{cc}
\toprule
Mask ratio $\rho$ & IW accuracy \\
\midrule
0.0 & 72.8452 \\
0.1 & 74.8571 \\
0.2 & 76.0238 \\
0.3 & 76.3571 \\
0.4 & 75.3095 \\
0.5 & 73.2500 \\
\bottomrule
\end{tabular}
\end{table}

In the mask scan, the highest recorded accuracy occurs at $\rho=0.3$ (Table~\ref{tab:app-d-mask-ratio}). For the same RW-trained MTSNet LaS model, accuracy is lower on the initial Physical window than on randomly sampled test windows at all three durations (Table~\ref{tab:app-d-rw-iw}). In the IR-volume experiment with BERT-base-uncased, initial-window accuracy rises from 34.7738\% without IR to 76.3571\% with approximately 300,000 recombined windows (Table~\ref{tab:app-d-ir-volume}). At 1.0~s, IR-trained MTSNet LaS with EmbeddingGemma reaches 79.50\%, in the range of the supervised MTSNet result of 78.01\% obtained without IR (Table~\ref{tab:app-d-iw-accuracy}). This dependence on IR augmentation motivates examining how the response changes over the first 0.5~s after the latency-corrected onset.

\clearpage
\subsection{Early-Window Behaviour and Conditional Unseen-Class Decoding}
\label{app:ssvep-early-window}

\paragraph{Window-onset offset.}
At $L=1.0$~s, we move the window start by $\delta\in\{0,0.1,0.2,0.3,0.4,0.5\}$~s relative to $t_0$. Each evaluation uses one window $[0.14+\delta,1.14+\delta)$ per Physical trial, without combining predictions across offsets. Figure~\ref{fig:app-d-offset} shows 40-class accuracy for FBCCA, RW-trained CLS (MTSNet), and RW-trained LaS (MTSNet with EmbeddingGemma-300M). All three curves use the same 8,400 Physical trials at each offset.

\begin{figure}[H]
\centering
\includegraphics[width=0.82\linewidth]{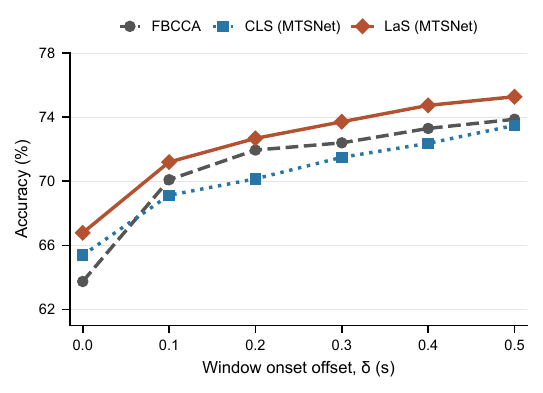}
\caption{Forty-class accuracy of a single 1.0-s Physical window at six offsets from the valid analysis onset $t_0=0.14$~s. The neural models are trained with RW; FBCCA directly classifies each window.}
\label{fig:app-d-offset}
\end{figure}

All three curves rise through the last measured offset, $\delta=0.5$~s.

\paragraph{Conditional unseen-class decoding after IR training.}
Figure~4c of the main paper reports five-class unseen-only results after RW training. Table~\ref{tab:app-d-ir-zsl} reports IR-trained MTSNet results under the same five-candidate evaluation format, organized by text encoder, adaptation state, and window length. IR produces approximately 300,000 training windows with $\rho=0.3$; each result covers eight class-held-out splits with five unseen candidates per split.

\begin{table}[!htbp]
\centering
\caption{Five-class unseen-only ZSL accuracy (\%) after IR training.}
\label{tab:app-d-ir-zsl}
\small
\begin{tabular}{llccc}
\toprule
Text encoder & State & 0.6~s & 0.8~s & 1.0~s \\
\midrule
BERT-base-uncased & Frozen & 27.0200 & 42.3800 & 40.0100 \\
BERT-base-uncased & LoRA & 34.1500 & 59.3900 & 55.8700 \\
EmbeddingGemma-300M & Frozen & 34.3333 & 65.7857 & 65.6071 \\
EmbeddingGemma-300M & LoRA & 32.0119 & 58.0952 & 60.2619 \\
\bottomrule
\end{tabular}
\end{table}

\clearpage
The shared early-window penalty is consistent with a response that remains nonstationary after latency correction. We hypothesize that supervised CLS can combine stimulus frequency with phase and class-specific onset transients, whereas LaS finds these additional early-window cues harder to learn. The gain from IR augmentation for LaS (Table~\ref{tab:app-d-ir-volume}) is consistent with this difference in access to early-window cues. Early nonstationarity may also impair the frequency representation needed for unseen-class transfer. At 0.6, 0.8, and 1.0~s, IR-trained MTSNet with EmbeddingGemma-300M obtains 34.3333\%, 65.7857\%, and 65.6071\% five-class ZSL accuracy (Table~\ref{tab:app-d-ir-zsl}), compared with 80.6786\%, 87.9762\%, and 90.8929\% for its RW-trained counterpart in Figure~4c.

\end{document}